\documentclass[letterpaper, 11pt]{article}

\usepackage[margin=1in]{geometry}
\usepackage[utf8]{inputenc}
\usepackage[T1]{fontenc}
\usepackage{amsmath}
\usepackage{graphicx}
\usepackage{smile}
\usepackage{xcolor}
\usepackage[colorlinks, linkcolor=blue, anchorcolor=blue, citecolor=blue]{hyperref}
\usepackage{enumitem}

\def\shownotes{1}  
\ifnum\shownotes=1
\newcommand{\authnote}[2]{$\ll$\textsf{\small #1 notes: #2}$\gg$}
\else
\newcommand{\authnote}[2]{}
\fi

\newcommand{\slope}{{\sf slope}}
\newcommand{\upperslope}{\overline{\sf slope}}
\newcommand{\lowerslope}{\underline{\sf slope}}
\newcommand{\dw}{W_1}

\DeclareMathOperator*{\upperlim}{limsup}
\DeclareMathOperator*{\lowerlim}{liminf}

\title{\LARGE \bf Distribution Approximation and Statistical Estimation Guarantees of Generative Adversarial Networks} 

\author{Minshuo Chen$^*$ \quad Wenjing Liao$^\dagger$ \quad Hongyuan Zha$^\diamond$ \quad Tuo Zhao$^*$ \\
$^*$ Industrial and systems engineering, Georgia Tech \quad $^\dagger$ Math, Georgia Tech \\
$^\diamond$ School of data science, Chinese University of Hong Kong, Shen Zhen \\
Email: {\tt \{mchen393, wliao60, tourzhao\}@gatech.edu}, {\tt zhahy@cuhk.edu.cn}}

\date{}

\begin{document}

\maketitle

\begin{abstract}
Generative Adversarial Networks (GANs) have achieved a great success in unsupervised learning. Despite its remarkable empirical performance, there are limited theoretical studies on the statistical properties of GANs. This paper provides approximation and statistical guarantees of GANs for the estimation of data distributions that have densities in a H\"{o}lder space. Our main result shows that, if the generator and discriminator network architectures are properly chosen, GANs are consistent estimators of data distributions under strong discrepancy metrics, such as the Wasserstein-1 distance. Furthermore, when the data distribution exhibits low-dimensional structures, we show that GANs are capable of capturing the unknown low-dimensional structures in data and enjoy a fast statistical convergence, which is free of curse of the ambient dimensionality.
Our analysis for low-dimensional data builds upon a universal approximation theory of neural networks with Lipschitz continuity guarantees, which may be of independent interest.
\\

\end{abstract}

\section{Introduction}

Generative Adversarial Networks (GANs, \cite{goodfellow2014generative}) utilize two neural networks competing with each other to generate new samples with the same distribution as the training data. They have been successful in many applications including producing photorealistic images, improving astronomical images, and modding video games \cite{reed2016generative, ledig2017photo,schawinski2017generative, brock2018large, volz2018evolving, radford2015unsupervised,salimans2016improved}. 

\begin{figure}[!htb]
\centering
\includegraphics[width = 0.8\textwidth]{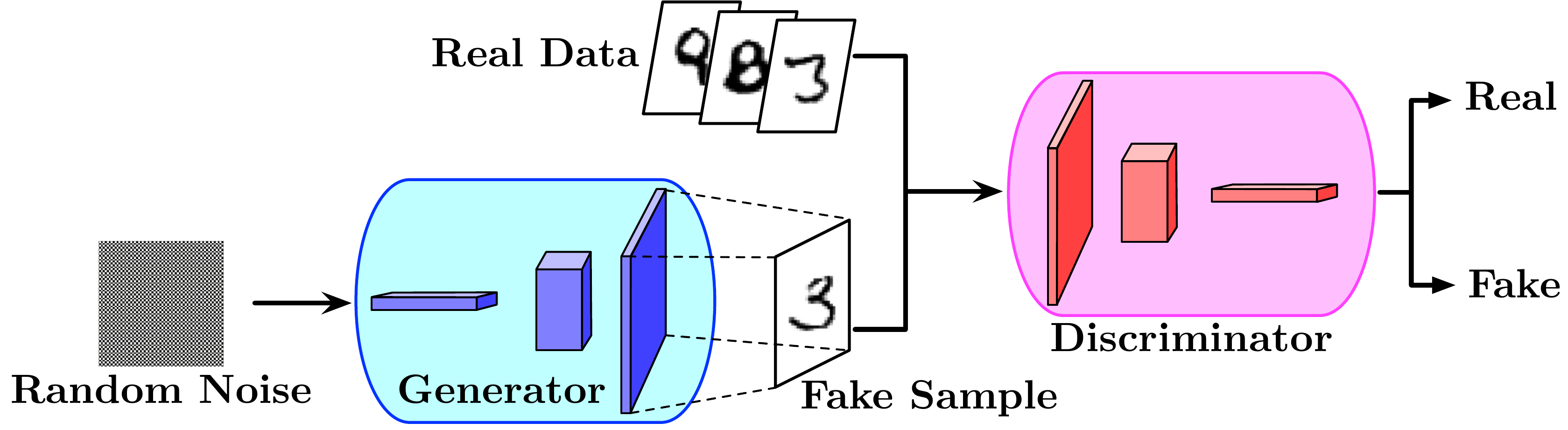}
\caption{The architecture of GANs}
\label{fig:gan}
\end{figure}

From the perspective of statistics, GANs have stood out as an important unsupervised method for learning target data distributions. Different from explicit distribution estimators, such as the kernel density estimator, GANs implicitly learn the data distribution and act as samplers to generate new fake samples mimicking the data distribution (see Figure \ref{fig:gan}).

To estimate a data distribution $\mu$, GANs solve the following minimax optimization problem 
\begin{equation}\label{eq:populationgan}
(g^*, f^*) \in \argmin_{g \in \cG} \max_{f \in \cF} ~\EE_{z \sim \rho} [f(g(z))] - \EE_{x \sim \mu} [f(x)],
\end{equation}
where $\cG$ denotes a class of generators, $\cF$ denotes a symmetric class (if $f\in \cF$, then $-f \in \cF$) of discriminators, and $z$ follows some easy-to-sample distribution $\rho$, e.g., a uniform distribution. The estimator of $\mu$ is given by a pushforward distribution of $\rho$ under $g^*$.

The inner maximization problem of \eqref{eq:populationgan} is an Integral Probability Metric (IPM, \cite{muller1997integral}), which quantifies the discrepancy between two distributions $\mu$ and $\nu$ w.r.t. the symmetric function class $\cF$:
\begin{equation*}
d_{\cF}(\mu, \nu) = \sup_{f \in \cF}~ \EE_{x \sim \mu} [f(x)] - \EE_{y \sim \nu} [f(y)].
\end{equation*}
Accordingly, GANs essentially minimize an IPM between the generated distribution and the data distribution. IPM unifies many standard discrepancy metrics. For example, when $\cF$ is taken to be all $1$-Lipschitz functions, $d_{\cF}(\cdot, \cdot)$ is the Wasserstein-1 distance $W_1(\cdot, \cdot)$; when $\cF$ is the class of all indicator functions, $d_{\cF}(\cdot, \cdot)$ is the total variation distance; when $\cF$ is taken as neural networks, $d_{\cF}(\cdot, \cdot)$ is the so-called ``neural net distance'' \cite{arora2017generalization}.

In practical GANs, the generator and discriminator classes $\cG$ and $\cF$ are parametrized by neural networks. We denote $\cG = \cG_{\textrm{NN}}$ and $\cF = \cF_{\textrm{NN}}$ to emphasize such a parameterization. In this paper, we focus on using feedforward ReLU networks, since it has wide applications \cite{nair2010rectified, glorot2011deep, maas2013rectifier} and can ease the notorious vanishing gradient issue during training, which commonly arises with sigmoid or hyperbolic tangent activations \cite{glorot2011deep, Goodfellow-et-al-2016}

When $n$ samples of the data distribution $\mu$ are given, denoted as $\{x_i\}_{i=1}^n$, one can replace $\mu$ in \eqref{eq:populationgan}  by its empirical counterpart $\hat{\mu}_n$, and \eqref{eq:populationgan} becomes
\begin{equation}\label{eq:empiricalgan}
(g^*_\theta, f^*_\omega) \in \argmin_{g_\theta \in \cG_{\textrm{NN}}} \max_{f_\omega \in \cF_{\textrm{NN}}} ~ \EE_{z \sim \rho} [f_\omega(g_\theta(z))]  - \frac{1}{n} \sum_{i=1}^n f_\omega(x_i),
\end{equation}
where $\theta$ and $\omega$ are parameters in the generator and discriminator networks, respectively. The empirical estimator of $\mu$ given by GANs is the pushforward distribution of $\rho$ under $g_\theta^*$, denoted by $(g_\theta^*)_\sharp \rho$.

In contrast to the prevalence of GANs in applications, there are very limited works on the theoretical properties of GANs \cite{arora2017generalization,bai2018approximability,liang2018well, schreuder2021statistical, Block2021ANE}. This paper focuses on the following fundamental questions from a theoretical point of view:

\begin{itemize}

\item {\it (Q1)}. What types of distributions can be approximated by a deep neural network generator?

\item {\it (Q2)}. If the distribution can be approximated, what is the statistical rate of estimation using GANs?

\item {\it (Q3)}. If further there are unknown low-dimensional structures in the data distribution, can GANs capture the low-dimensional data structure and enjoy a fast rate of estimation?
\end{itemize}

\subsection{Main results}

\noindent \textbf{Results in Euclidean space}. To address {\it (Q1)} and {\it (Q2)}, we show that, if the generator and discriminator network architectures are properly chosen, GANs can learn distributions with H\"{o}lder densities supported on a convex domain. Specifically, we consider a data distribution $\mu$ supported on a compact convex subset $\cX \subset \RR^d$, where $d$ is the data dimension. We assume $\mu$ has an $\alpha$-H\"{o}lder density with respect to Lebesgue measure in $\RR^d$ and the density is lower bounded away from $0$ on $\cX$.

Our generator and discriminator network architectures are explicitly chosen --- we specify the width and depth of the network, total number of neurons, and total number of weight parameters (details are provided in Section \ref{sec:theory}). Roughly speaking, the generator needs to be flexible enough to approximately transform an easy-to-sample distribution to the data distribution, and the discriminator is powerful enough to distinguish the generated distribution from the data distribution.

Let $g^*_\theta$ be the optimal solution of \eqref{eq:empiricalgan}, and then $(g^*_\theta)_\sharp \rho$ is the generated data distribution as an estimation of $\mu$. Our main result can be summarized as, for any $\beta \geq 1$, if the generator and discriminator network architectures are properly chosen, then
\begin{equation}\label{eq:highdrate}
\EE\left[d_{\cH^\beta}\left((g^*_\theta)_\sharp\rho, \mu \right)\right] = \tilde{O}\big(n^{-\frac{\beta}{2\beta + d}} \log^2 n\big),
\end{equation}
where the expectation is taken over the randomness of samples and $\tilde{O}$ hides polynomial factors in $\beta, d$. 
It shows that the $\beta$-H\"{o}lder IPM between the generated distribution and the data distribution converges at a rate depending on the H\"{o}lder index $\beta$ and dimension $d$. When $\beta = 1$, our theory implies that GANs can estimate any distribution with a H\"{o}lder density under the Wasserstein-$1$ distance. A comparison to closely related works is provided in Section \ref{sec:related}.

In our analysis, we decompose the distribution estimation error into a statistical error and an approximation error by an oracle inequality.
A key step is to properly choose the generator network architecture to control the approximation error. Specifically, the generator architecture allows an accurate approximation to a data transformation $T$ such that $T_\sharp \rho = \mu$. The existence of such a transformation $T$ is guaranteed by optimal transport theory \cite{villani2008optimal}, and holds universally for all the data distributions with H\"{o}lder densities.

\noindent \textbf{Results in low-dimensional linear subspace}. Moreover, we provide a positive answer to {\it (Q3)} by considering data distributions with low-dimensional linear structures. Specifically, we assume the data support $\cX \subset \RR^d$ is a compact subset of a $q$-dimensional linear subspace. Let columns of $A \in \RR^{d \times q}$ denote a set of orthonormal basis of the $q$-dimensional linear subspace.
We assume the pushforward $A^{\top}_\sharp \mu$ of data distribution has a density function $p_\mu$ defined in $\RR^q$, and $p_\mu$ is $\alpha$-H\"{o}lder continuous and lower bounded away from $0$ on its support. We leverage the data geometric structures and generate samples by transforming an easy-to-sample distribution $\rho$ in $\RR^q$. With a proper choice of the generator and discriminator network architectures, the statistical error of GANs converges at a fast rate
\begin{equation}\label{eq:lowdrate}
\EE\left[\dw\left((g^*_\theta)_\sharp\rho, \mu \right)\right] = \tilde{O}\left(n^{-\frac{1}{2 + q}} \log^2 n\right).
\end{equation}
By taking $\beta = 1$ in \eqref{eq:highdrate}, we note that \eqref{eq:lowdrate} enjoys a faster statistical convergence in the Wasserstein-1 distance, since the exponent only depends on the intrinsic dimension $q$.
Meanwhile, \eqref{eq:lowdrate} indicates that GANs can circumvent the curse of ambient dimensionality when data are supported on a low-dimensional subspace.

From a technical point of view, a key challenge in obtaining the fast rate in \eqref{eq:lowdrate} is to prove that the generator can capture the unknown linear structure in data. We achieve this by introducing a learnable linear projection layer in the generator, and pairing it with an ``anti-projection'' layer in the discriminator. We show (see Lemma \ref{lemma:subspace_match}) that by optimizing \eqref{eq:empiricalgan}, the linear projection layer in generator accurately recovers the linear subspace of data.

\noindent {\bf Results in low-dimensional mixture model}. We further consider learning low-dimensional mixture distributions. In particular, we assume the data distribution $\mu$ consists of $K$ components, i.e., $\mu = \sum_{k=1}^K p_k \mu_k$ with $p_k > 0, \sum_{k=1}^K p_k = 1.$
Each $\mu_k$ can be represented as low-dimensional pushforward distribution $\mu_k = (g_k)_\sharp {\rm Unif}([0, 1]^q)$ ($q \leq d$), where $g_k$ is a H\"{o}lder mapping. As can be seen, $\mu$ exhibits low-dimensional structures and include the linear subspace setting as a special case. Mixture data are widely seen in practice, such as in image classification problems \cite{lee2002unsupervised, chen2010predictive, fang2017robust, caron2018deep}.

To estimate $\mu$, we transform a $(q+1)$-dimensional uniform distribution $\rho$. We optimize \eqref{eq:empiricalgan} over properly chosen generator and discriminator networks. Then we prove that the statistical error of GANs also converges at a fast rate
\begin{align}\label{eq:mix_rate}
\EE \left[\dw\left((g_\theta^*)_\sharp \rho, \mu \right)\right] \leq C_\delta d n^{-\frac{1}{q + \delta}},
\end{align}
where $\delta > 0$ is any positive constant and $C_\delta$ is independent of $n$. \eqref{eq:mix_rate} further demonstrates that GANs are adaptive to data intrinsic structures and better explains the empirical success of GANs.

\noindent {\bf Roadmap}: The rest of the paper is organized as follows: Section \ref{sec:theory} presents the statistical guarantees of GANs for learning data distributions with a H\"{o}lder density. 
Section \ref{sec:lowd} extends the statistical theory to low-dimensional linear data, and shows that GANs can adapt to the intrinsic structures. Section \ref{sec:mixture} further shows that GANs are adaptive to low-dimensional nonlinear mixture models. Section \ref{sec:related} compares our results to existing literature. Section \ref{sec:proof} proves the theories in Section \ref{sec:theory} and Section \ref{sec:lowdproof} presents an outline for establishing results in Section \ref{sec:lowd}. Lastly, Section \ref{sec:discuss} concludes the paper and discusses related topics.

\noindent {\bf Notations}: Given a real number $\alpha$, we denote $\lfloor \alpha \rfloor$ as the largest integer smaller than $\alpha$ (in particular, if $\alpha$ is an integer, $\lfloor \alpha \rfloor = \alpha - 1$). Given a vector $v \in \RR^d$, we denote its $\ell_2$ norm by $\norm{v}_2$, the $\ell_\infty$ norm as $\norm{v}_\infty = \max_{i} \vert v_i\vert$, and the number of nonzero entries by $\norm{v}_0$. Given a matrix $A \in \RR^{d_1 \times d_2}$, we denote $\norm{A}_\infty = \max_{i, j} \vert A_{i, j}\vert$ as the maximal magnitude of entries and the number of nonzero entries by $\norm{A}_0$. We denote function $L_\infty$-norm as $\norm{f}_\infty = \sup_x \lvert f(x)\rvert$. For a multivariate transformation $T: \RR^{d_1} \mapsto \RR^{d_2}$, and a given distribution $\rho$ in $\RR^{d_1}$, we denote the pushforward distribution as $T_\sharp \rho$, i.e., for any measurable set $\Omega$, $T_\sharp \rho (\Omega) = \rho (T^{-1} (\Omega))$.

\section{Preliminary}\label{sec:pre}
In this section, we introduce distributions with H\"{o}lder densities, discrepancy metrics between distributions, optimal transport theory, and neural network architectures.

\subsection{H\"{o}lder density and IPM}
Throughout the paper, we focus on estimating a data distribution $\mu$ supported on domain $\cX$. In Section \ref{sec:theory}, we consider $\mu$ having a well-defined density function $p_\mu$ with respect to the Lebesgue measure in $\RR^d$. Moreover, we characterize the smoothness of $p_\mu$ by H\"{o}lder continuity.
\begin{definition}[$\alpha$-H\"{o}lder Function]
Given a H\"{o}lder index $\alpha > 0$, a function $f : \cX \mapsto \RR$ belongs to the H\"{o}lder class $\cH^\alpha(\cX)$, if and only if, for any multi-index $s \in \NN^d$ with $\lvert s \rvert = \sum_{i=1}^d s_i \leq \lfloor \alpha \rfloor$, the derivative $\partial^{s} f = \frac{\partial^{\lvert s \rvert} f}{\partial x_1^{s_1} \dots \partial x_d^{s_d}}$ exists, and for any $s$ satisfying $\lvert s \rvert = \lfloor \alpha \rfloor$, we have
\begin{align*}
\sup_{x \neq y}~ \frac{\big\vert \partial^{s} f(x) - \partial^{s} f(y)\big\vert}{\norm{x - y}_2^{\alpha - \lfloor \alpha \rfloor}} < \infty \quad \textrm{for any $x, y$ in the interior of $\cX$}.
\end{align*}
\end{definition}
When $f \in \cH^\alpha(\cX)$, we define its H\"{o}lder norm as
\begin{align*}
\norm{f}_{\cH^{\alpha}(\cX)} = \sum_{0 \leq s \leq \lfloor \alpha \rfloor} \norm{\partial^{s} f}_{\infty} + \sum_{\lvert s \rvert = \lfloor \alpha \rfloor} \sup_{x \neq y} \frac{\big\vert \partial^{s} f(x) - \partial^{s} f(y)\big\vert}{\norm{x - y}_2^{\alpha - \lfloor \alpha \rfloor}}.
\end{align*}
The H\"{o}lder continuity above can be generalized to multi-dimensional mappings. Specifically, for $g = [g_1, \dots, g_{d}]^\top : \cX \mapsto \RR^d$, we say it is $\alpha$-H\"{o}lder if and only if each coordinate mapping $g_i$ is $\alpha$-H\"{o}lder. In addition, the H\"{o}lder norm of $g$ is defined as $\norm{g}_{\cH^\alpha(\cX)} = \sum_{i=1}^d \norm{g_i}_{\cH^\alpha(\cX)}$.

In order to measure the performance of GANs in estimating target distribution $\mu$, we adopt the Integral Probability Metric (IPM) with respect to H\"{o}lder discriminative functions. In particular, suppose GAN generates a fake distribution $\nu$. For any $\beta \geq 1$, we denote
\begin{align*}
d_{\cH^\beta}(\mu, \nu) = \sup_{f \in \cH^\beta} \EE_{x \sim \mu} [f(x)] - \EE_{y \sim \nu} [f(y)].
\end{align*}
\begin{remark}
It is convenient to restrict $\cH^\beta$ in IPM $d_{\cH^\beta}$ to have a bounded radius. Specifically, for any $f \in \cH^\beta$, we assume $\norm{f}_{\cH^\beta} \leq C$ for some constant $C$. Otherwise, we can simply rescale $f$ while maintaining the discriminative power of the IPM. In addition, since IPMs are translation invariant, meaning that discriminative functions $f$ and $f+c$ for some constant $c$ are equivalent. Therefore, we also assume $f(0) = 0$ for simplicity.
\end{remark}
In the special case of $\beta = 1$, $d_{\cH^\beta}(\cdot, \cdot)$ shares the same discriminative power as Wasserstein-1 distance, which can be defined using the dual formulation,
\begin{align*}
\dw(\mu, \nu) = \sup_{\norm{f}_{\rm Lip} \leq 1} \EE_{x \sim \mu} [f(x)] - \EE_{y \sim \nu} [f(y)].
\end{align*}
In the right-hand side above, $\norm{f}_{\rm Lip}$ denotes the Lipschitz coefficient of $f$. It can be checked that Lipschitz functions are H\"{o}lder continuous with H\"{o}lder index $\alpha = 1$. Therefore, $\dw(\cdot, \cdot)$ is equivalent to $d_{\cH^1}(\cdot, \cdot)$.

\subsection{Optimal transport}
GANs are closely related to Optimal Transport (OT, \cite{santambrogio2010models, galichon2017survey, ganin2014unsupervised, courty2016optimal}), as the generator essentially learns a pushforward mapping of an easy-to-sample distribution. A typical problem in OT is the following: Let $\cX, \cZ$ be subsets of $\RR^d$. Given two probability spaces $(\cX, \mu)$ and $(\cZ, \rho)$, OT aims to find a transformation $T : \cZ \mapsto \cX$, such that $T(z) \sim \mu$ for $z \sim \rho$.
In general, the transformation $T$ may neither exist nor be unique. Fortunately, in the case that $\mu$ and $\rho$ have H\"{o}lder densities $p_\mu$ and $p_\rho$, respectively, the Monge map ensures the existence of a H\"{o}lder transformation $T^*$, when $\cX$ is convex. In particular, the Monge map $T^*$ is the solution to the following optimization problem:
\begin{equation}\label{eq:monge}
T^* \in \argmin_T~~ \EE_{z \sim \rho} \left[\ell(z, T(z))\right],\quad\textrm{subject to}\quad T_\sharp \rho = \mu,
\end{equation}
where $\ell$ is a cost function. \eqref{eq:monge} is known as the Monge problem. When $\cX$ is convex and the cost function is quadratic, the solution to \eqref{eq:monge} satisfies the Monge-Amp\`{e}re equation \cite{monge1784memoire}. The regularity of $T^*$ was proved in \cite{caffarelli1992regularity, caffarelli1992boundary, caffarelli1996boundary} and \cite{urbas1988regularity, urbas1997second} independently. Their main result is summarized in the following lemma.
\begin{lemma}[\cite{caffarelli1992regularity}]\label{lemma:monge}
Suppose $\mu$ and $\rho$ both have $\alpha$-H\"{o}lder densities, and the support $\cX$ is convex. Then there exists a transformation $T^* : \cZ \mapsto \cX$ such that $T^*_\sharp \rho = \mu$. Moreover, this transformation $T^*$ belongs to the H\"{o}lder class $\cH^{\alpha+1}(\cZ)$.
\end{lemma}
We will see in Theorem \ref{thm:generator}, Lemma \ref{lemma:monge} provides important guidelines for choosing proper generator networks in distribution estimation.

\subsection{Network architecture and universal approximation}
Recall that we parameterize the generator and discriminator in GANs as ReLU neural networks, which takes the following form
\begin{equation}\label{eq:fnn}
f(x) = W_{L} \cdot \textrm{ReLU}(W_{L-1} \cdots \textrm{ReLU}(W_{1} x + b_{1}) \cdots + b_{L-1}) + b_{L},
\end{equation}
with $W_{i}$'s and $b_i$'s being weight matrices and intercepts, respectively. The ReLU activation function computes ${\rm ReLU}(a) = \max\{a, 0\}$ and is applied entrywise. When optimize \eqref{eq:empiricalgan} during training, we take $\cF_{\rm NN}$ and $\cG_{\rm NN}$ as a class of neural networks, which we refer to as a network architecture. We define the following prototypical network architecture:
\begin{align}\label{eq:NN_architecture}
\begin{split}
{\rm NN}(R, \kappa, L, p, J, d_{\rm in}, d_{\rm out}) = \Big\{
&g : \RR^{d_{\rm in}} \mapsto \RR^{d_{\rm out}} ~ \big\vert~ g ~\textrm{in}  \textrm{ the form of \eqref{eq:fnn},} \\
&\textrm{with $L$ layers and max width $p$,} \\
&\norm{g_i}_\infty \leq R, ~\norm{W_{i}}_{\infty} \leq \kappa, ~\norm{b_{i}}_\infty \leq \kappa, \\
& \sum_{j=1}^L \norm{W_{i}}_0 + \norm{b_{i}}_0 \leq J, \textrm{ for}~ i = 1, \dots, L \Big\}.
\end{split}
\end{align}
In later sections, we will take generator and discriminator networks based on \eqref{eq:NN_architecture} with appropriate configuration parameters.

A key property of the network architecture \eqref{eq:NN_architecture} is its universal approximation ability \cite{cybenko1989approximation, hornik1991approximation, chui1992approximation, barron1993universal, mhaskar1996neural}.
Recently, \cite{yarotsky2017error} established a universal approximation theory for ReLU networks, where a network with optimal size is constructed to approximate any Sobolev functions. We extend to H\"{o}lder functions and summarize the result in the following lemma, whose proof is deferred to Appendix \ref{pf:approx}.
\begin{lemma}[Universal Approximation]\label{lemma:approx}
Let $\cX$ be a compact domain in $[0,1]^d$.
Given any $\delta \in (0, 1)$, there exists a ReLU network architecture such that, for any $f \in \cH^\alpha(\cX)$ for $\alpha \geq 1$, if the weight parameters are properly chosen, the network yields a function $\hat{f}$ for the approximation of $f$ with $\norm{\hat{f} - f}_\infty \leq \delta$. Such a network has (i) no more than $c(\log \frac{1}{\delta} + 1)$ layers, and (ii) at most $c' \delta^{-\frac{d}{\beta}} (\log \frac{1}{\delta} + 1)$ neurons and weight parameters, where the constants $c$ and $c'$ depend on $d$, $\alpha$, and H\"{o}lder norm $\norm{f}_{\cH^\alpha(\cX)}$.
\end{lemma}

\section{Distribution estimation in Euclidean space}\label{sec:theory}

We consider a data distribution $\mu$ supported on a convex subset $\cX \subset \RR^d$ and assume that $\mu$ has a density function $p_\mu$ with respect to the Lebesgue measure in $\RR^d$. GANs seek to estimate the data distribution $\mu$ by transforming some easy-to-sample distribution $\rho$ supported on domain $\cZ \subset \RR^d$, such as a uniform distribution. Our main results provide statistical guarantees of GANs for the estimation of $\mu$, based on the following assumptions.
\begin{assumption}\label{assump1}
The domains $\cX$ and $\cZ$ are compact, and $\cX$ is convex. There exists a constant $B>0$ such that for any $x \in \cX$ or $x \in \cZ$, $\norm{x}_\infty \leq B$.
\end{assumption}

\begin{assumption}\label{assump2}
Given a H\"{o}lder index $\alpha > 0$, the density function $p_\mu$ of $\mu$ (w.r.t. Lebesgue measure in $\RR^d$) belongs to the H\"{o}lder class $\cH^\alpha(\cX)$ with $\norm{p_\mu}_{\cH^{\alpha}(\cX)} \leq C$ for some constant $C > 0$. Meanwhile, $p_\mu$ is lower bounded, i.e., $$\inf_{x \in \cX}~ p_\mu(x) \geq \tau$$ for some constant $\tau>0.$
\end{assumption}

\begin{assumption}\label{assump3}
The easy-to-sample distribution $\rho$ has a $C^\infty$ (smooth) density function $p_\rho$.
\end{assumption}

H\"{o}lder regularity is commonly used in literature on smooth density estimation \cite{wasserman2006all, tsybakov2008introduction}. In the remaining of the paper, 
we occasionally omit the domain in H\"{o}lder spaces when it is clear from the context. The condition of $p_\mu$ being lower bounded is a common technical assumption in the optimal transport theory \cite{moser1965volume, caffarelli1996boundary}. This condition and the convexity of $\cX$ guarantee that, there exists a H\"{o}lder transformation $T$ such that $T_\sharp \rho = \mu$ (see Lemma \ref{lemma:monge}). Besides, Assumption \ref{assump3} is always satisfied, since $\rho$ is often taken as a uniform distribution.

Given Assumption \ref{assump1} - \ref{assump3}, we set the generator network architecture as
\begin{equation*}
\cG_{\textrm{NN}}(R, \kappa, L, p, J) = {\rm NN}(R, \kappa, L, p, J, d_{\rm in} = d, d_{\rm out} = d)
\end{equation*}
and the discriminator network architecture as
\begin{equation*}
\cF_{\textrm{NN}}(\bar{R}, \bar{\kappa}, \bar{L}, \bar{p}, \bar{J}) = {\rm NN}(\bar{R}, \bar{\kappa}, \bar{L}, \bar{p}, \bar{J}, d_{\rm in} = d, d_{\rm out} = 1).
\end{equation*}

We first show a properly chosen generator network can universally approximate data distributions with a H\"older density. 
\begin{theorem}[Distribution approximation theory]\label{thm:generator}
For any data distribution $(\cX, \mu)$ and easy-to-sample distribution $(\cZ, \rho)$ satisfying Assumption \ref{assump1} - \ref{assump3}, there exists an $(\alpha+1)$-H\"{o}lder continuous transformation $T: \RR^{d} \rightarrow \RR^d$ such that $T_\sharp \rho = \mu$. Moreover, 
given any $\epsilon \in (0, 1)$, there exists a generator network with configuration
\begin{equation}\label{eq:gnnsize}
\begin{split}
& L = O(\log(1/\epsilon)), \quad p = O(d \epsilon^{-\frac{d}{\alpha+1}}), \quad J = O(d \epsilon^{-\frac{d}{\alpha+1}} \log (1/\epsilon)), \\
& \hspace{1.2in} R = B, \quad \kappa = \max\{C, B\},
\end{split}
\end{equation}
such that, if the weight parameters of this network are properly chosen, then it yields a transformation $g_\theta$ satisfying $$\max_{z \in \cZ} \norm{g_\theta(z) - T(z)}_\infty \leq \epsilon \quad \text{and} \quad W_1((g_\theta)_\sharp \rho, \mu) \leq \sqrt{d} \epsilon.$$
\end{theorem}

In Theorem \ref{thm:generator}, the existence of a transformation $T$ is guaranteed by optimal transport theory (Lemma \ref{lemma:monge}). Furthermore, we explicitly choose a generator network architecture to approximately realize $T$, such that the easy-to-sample distribution is approximately transformed to the data distribution.

Our statistical result is the following finite-sample estimation error bound in terms of the H\"{o}lder IPM between $(g_\theta^*)_\sharp \rho$ and $\mu$, where $g^*_\theta$ is the optimal solution of GANs in \eqref{eq:empiricalgan}. We use $O(\cdot)$ to hide constant factors depending on $B$, $C$, $\alpha$, and $\beta$; $\tilde{O}(\cdot)$ further hides polynomial factors of $d$ and logarithmic factors of $n$.

\begin{theorem}[Statistical estimation theory]\label{thm:main}
Suppose Assumption \ref{assump1} -- \ref{assump3} hold. For any $\beta \geq 1$, choose $\epsilon = n^{-\frac{\beta}{2\beta + d}}$ in Theorem \ref{thm:generator} for the generator network
and
\begin{align*}
& \bar{L} = O\left(\frac{\beta}{2\beta + d} \log n \right), \quad \bar{p} = O\left(n^{\frac{d}{2\beta + d}}\right), \quad \bar{J} = O\left(\frac{\beta}{2\beta + d} n^{\frac{d}{2\beta + d}} \log n\right), \\
& \hspace{1.7in} \bar{R} = C, \quad \bar{\kappa} = C, 
\end{align*}
for the discriminator network. Then it holds
\begin{align}
\EE \left[d_{\cH^\beta}((g_\theta^*)_\sharp \rho, \mu)\right] = \tilde{O}\left(n^{-\frac{\beta}{2\beta + d}}\log^2 n \right).
\label{thm2eq}
\end{align}
\end{theorem}
Theorem \ref{thm:main} demonstrates that GANs can effectively learn data distributions, with a convergence rate depending on the smoothness of the function class in IPM and the dimension $d$.

We remark that, both networks have uniformly bounded outputs. Such a requirement can be achieved by adding an additional clipping layer to the end of the network, in order to truncate the output in the range $[-R, R]$. Specifically, we can use
$
g(a) =  \max\{-R, \min\{a, R\}\} = \textrm{ReLU}(a-R) - \textrm{ReLU}(a + R) - R.
$

In the case that only $m$ samples from the easy-to-sample distribution $\rho$ are collected, GANs solve the following empirical minimax problem
\begin{align}\label{eq:finitegan}
\min_{g_\theta \in \cG_{\textrm{NN}}} \max_{f_\omega \in \cF_{\textrm{NN}}} \frac{1}{m} \sum_{i=1}^m f_\omega(g_\theta(z_i)) - \frac{1}{n} \sum_{j=1}^n f_\omega(x_j).
\end{align}
We denote $(g_{\theta}^{*, m}, f_\omega^{*, m})$ as the optimal solution of \eqref{eq:finitegan}. We show in the following corollary that GANs retain similar statistical guarantees for distribution estimation with finite generated samples. 
\begin{corollary}\label{cor:finitesample}
Suppose Assumption \ref{assump1} -- \ref{assump3} hold and $m \geq n$. We choose
\begin{align*}
& \hspace{0.6in} L = O\left(\frac{\alpha+1}{2(\alpha+1) + d} \log m\right), \quad p = O\left(d m^{\frac{d}{2(\alpha+1)+d}} \right), \\
& \quad J = O\left(\frac{d(\alpha+1)}{2(\alpha+1) +d} m^{\frac{d}{2(\alpha+1)+d}} \log m \right),\quad R = B, \quad \kappa = \max\{C, B\},
\end{align*}
for the generator network and the same architecture as in Theorem \ref{thm:main} for the discriminator network. Then it holds
\begin{align*}
\EE \left[d_{\cH^\beta}((g_\theta^{*, m})_\sharp \rho, \mu)\right] = \tilde{O}\left(n^{-\frac{\beta}{2\beta + d}} + m^{-\frac{\alpha+1}{2(\alpha+1) + d}}\right).
\end{align*}
\end{corollary}

Here $\tilde{O}$ also hides a logarithmic factors $m$. As it is often cheap to obtain a large amount of samples from $\rho$, the convergence rate in Corollary \ref{cor:finitesample} is dominated by $n^{-\frac{\beta}{2\beta + d}}$ whenever $m \geq n^{\frac{\beta}{\alpha+1} \frac{2(\alpha+1) + d}{2\beta + d} \vee 1}$.

Theorem \ref{thm:main} and Corollary \ref{cor:finitesample} suggest that GANs suffer from the curse of data dimensionality. However, such an exponential dependence on the dimension $d$ is inevitable without further assumptions on the data, as indicated by the minimax optimal rate of distribution estimation: To estimate a distribution $\mu$ with a $\cH^\alpha(\cX)$ density, the minimax optimal rate under the $\cH^\beta$ IPM loss satisfies
\begin{align*}
\inf_{\tilde{\mu}_n} \sup_{\mu \in \cH^\alpha} \EE \left[d_{\cH^\beta}(\tilde{\mu}_n, \mu) \right] \gtrsim n^{-\frac{\alpha+\beta}{2\alpha + d}} + n^{-\frac{1}{2}},
\end{align*}
where $\tilde{\mu}_n$ is any estimator of $\mu$ based on $n$ data points \cite{liang2018well, tang2022minimax}.

\section{Distribution estimation in low-dimensional linear subspace}\label{sec:lowd}
In this section, we 
prove that GANs are adaptive to unknown low-dimensional linear structures in data. We consider the data domain $\cX \subset \RR^d$ being a compact subset of a $q$-dimensional linear subspace with $q \ll d$. Our analysis holds for general $q \leq d$, while $q \approx d$ is less of interest as practical data sets are often low-dimensional with intrinsic dimension much smaller than ambient dimension \cite{tenenbaum2000global, roweis2000nonlinear, pope2021intrinsic}.
\begin{assumption}\label{assumption:lowdmanifold}
The data domain $\cX$ is compact, i.e., there exists a constant $B > 0$ such that for any $x \in \cX$, $\norm{x}_\infty \leq B$. Moreover, $\cX$ is a convex subset of a $q$-dimensional linear subspace in $\RR^d$, and the span of $\cX$ is the $q$-dimensional subspace.
\end{assumption}
\begin{figure}[!htb]
\centering
\includegraphics[width = 0.85\textwidth]{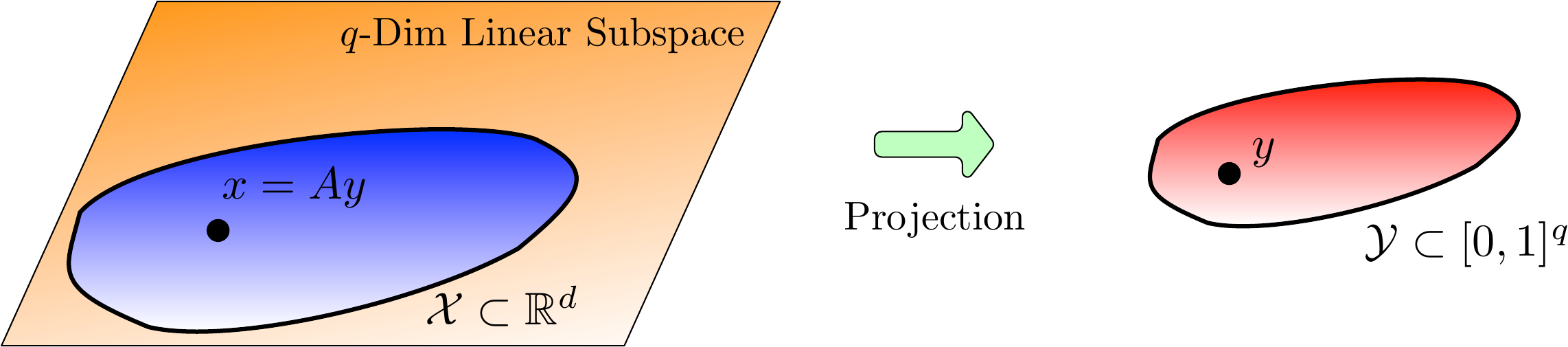}
\caption{Low-dimensional linear structures in $\cX$.}
\label{fig:linear_lowd}
\end{figure}
Under Assumption \ref{assumption:lowdmanifold}, a data point $x \in \cX$ can be represented as $A y$, where $y \in \RR^{q}$ and $A \in \RR^{d \times q}$ is a linear transformation (See graphical illustration in Figure \ref{fig:linear_lowd}). The following lemma formally justifies the existence of the linear transformation $A$.
\begin{lemma}\label{lemma:linearprojection}
Suppose Assumption \ref{assumption:lowdmanifold} holds. Consider a matrix $A \in \RR^{d \times q}$ with columns being an orthonormal basis of the $q$-dimensional linear subspace. Then it holds that $\cY = A^\top \cX = \{A^\top x : x \in \cX\}$ is a compact and convex subset of $\RR^{q}$, and $A \cY = \cX$.
\end{lemma}
The proof is deferred to Appendix \ref{pf:linearprojection}. The projected domain $\cY$ captures the intrinsic geometric structures in $\cX$. More importantly, using transformation $A$ allows us to define smoothness of the target data distribution. Specifically, we consider a data distribution $\mu$ supported on $\cX$. Since $\cX$ is a low-dimensional space, $\mu$ does not have a well defined density function with respect to the Lebesgue measure in $\RR^d$. Thanks to Lemma \ref{lemma:linearprojection}, the pushforward distribution $A^{\top}_\sharp \mu$ has a well-defined density function. Accordingly, we make the following data distribution assumption.
\begin{assumption}\label{assumption:lowddensity}
Without loss of generality, we assume $\cY \subset [0, 1]^q$. Given a H\"{o}lder index $\alpha > 0$, the density function $p_\mu$ of $A^{\top}_\sharp \mu$ belongs to $\cH^\alpha(\cY)$ with a bounded H\"{o}lder norm $\norm{p_\mu}_{\cH^{\alpha}(\cY)} \leq C$ for some constant $C > 0$, and $p_\mu \geq \tau > 0$ on $\cY$ for some constant $\tau$.
\end{assumption}
We assume $\cY \subset [0, 1]^q$ for convenience. Otherwise, we can rescale the input space $\cX$ by a constant $c$, so that the projected space $\cY \subset [0, 1]^q$. Since $\cX$ is compact, the constant $c$ is bounded and will not undermine the statistical rate of convergence.

To generate samples mimicking data distribution $\mu$, we consider transforming a $q$-dimensional easy-to-sample distribution $\rho$ supported on $[0, 1]^q$ to leverage the structural assumption in domain $\cX$. We define the generator network architecture $\cG_{\rm NN}^{\rm ld}(R, \kappa, L, p, J)$ as
\begin{equation}\label{eq:generator_lowd}
\begin{split}
\cG_{\rm NN}^{\rm ld}(R, \kappa, L, p, J) = & \big\{U \circ g_\theta : U \in \RR^{d \times q}~\text{with orthonormal columns and} \\
& \hspace{0.7in} g_\theta \in {\rm NN}(R, \kappa, L, p, J, d_{\rm in} = q, d_{\rm out} = d) \big\}.
\end{split}
\end{equation}
Note that $U \in \RR^{d \times q}$ lifts the transformed easy-to-sample distribution $(g_\theta)_\sharp \rho$ to $\RR^d$. We expect $U$ to extract the linear structures in data, while $g_\theta$ approximates an optimal transport plan for transforming $\rho$ to $A_\sharp^\top \mu$.

Pairing with the generator, we define the discriminator network architecture $\cF_{\rm NN}^{\rm ld}(\bar{R}, \bar{\kappa}, \bar{L}, \bar{p}, \bar{J}, \bar{\gamma})$ as
\begin{equation}\label{eq:discriminator_lowd}
\begin{split}
\cF_{\rm NN}^{\rm ld}(\bar{R}, \bar{\kappa}, \bar{L}, \bar{p}, \bar{J}, \bar{\gamma}) & = \big\{ f_\omega \circ V^\top : V \in \RR^{d \times q}~\text{with}~\norm{V}_2 \leq 1, \\
& \qquad f_\omega \in {\rm NN}(\bar{R}, \bar{\kappa}, \bar{L}, \bar{p}, \bar{J}, d_{\rm in} = q, d_{\rm out} = 1), \textrm{ and } \\
& \qquad \left\lvert f_\omega(x) - f_\omega(y) \right\rvert \leq \bar{\gamma} \norm{x - y}_\infty~\text{for}~x, y \in [0, 1]^q \big\}.
\end{split}
\end{equation}
The matrix $V$ is chosen to ``couple'' with the linear structures learned by the generator (``anti-projection'') and $f_\omega$ will approximate Lipschitz functions in $\RR^q$ for approximating Wasserstein distance. We remark that an appropriate choice of Lipschitz coefficient $\bar{\gamma}$ on $f_\omega$ will not undermine the approximation power of $\cF_{\rm NN}^{\rm ld}$ as confirmed in Lemma \ref{lemma:approx_lip}. Meanwhile, the Lipschitz constraint of discriminator ensures that the generator can accurately capture the linear structures in data. In practice, such a Lipschitz regularity is often enforced by computational heuristics \cite{virmaux2018lipschitz, pauli2021training, gouk2021regularisation}.
\begin{figure}[!htb]
\centering
\includegraphics[width = 0.65\textwidth]{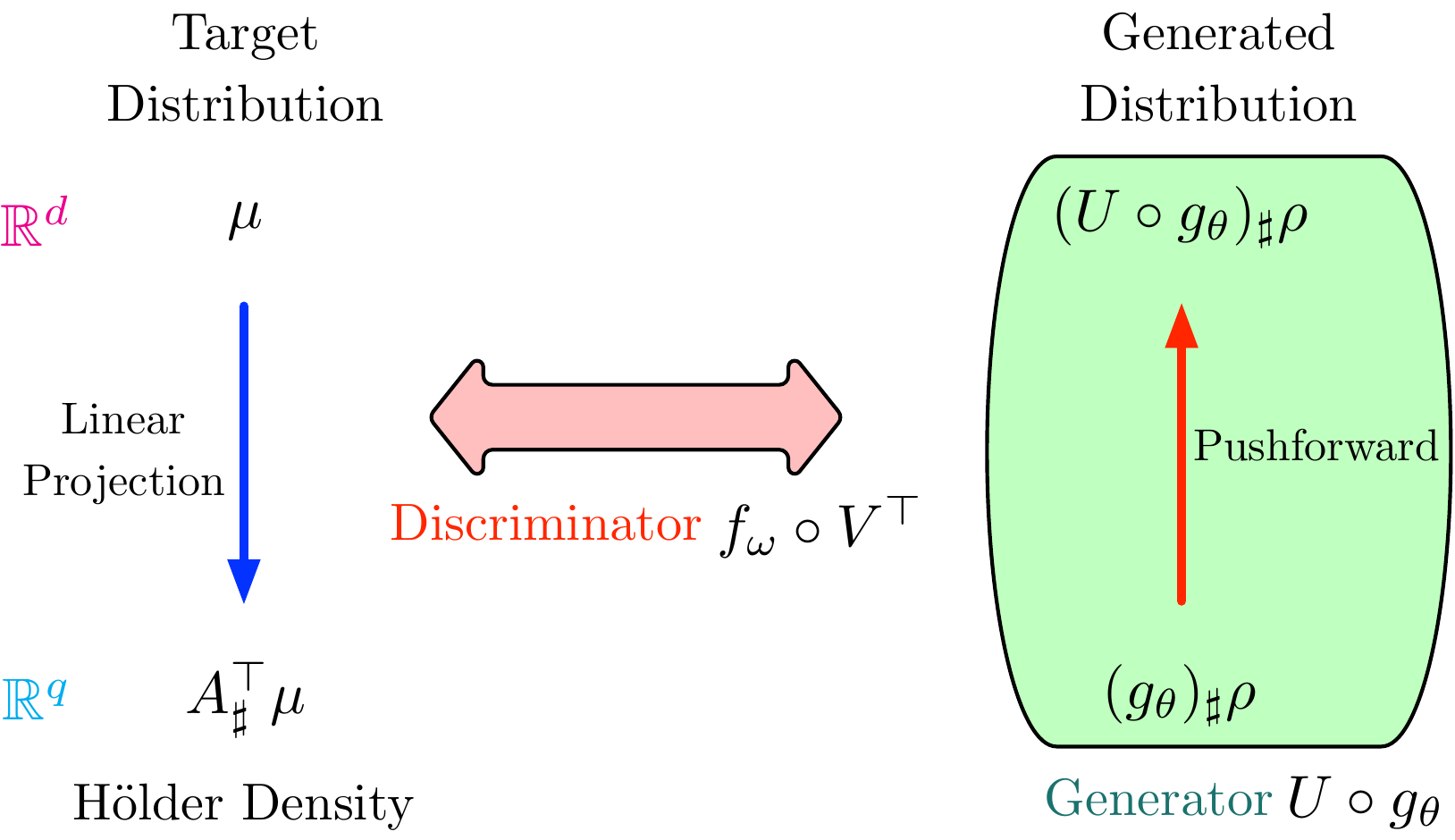}
\caption{Learning data distribution $\mu$ with unknown linear structures using generator in \eqref{eq:generator_lowd} and discriminator in \eqref{eq:discriminator_lowd}.}
\label{fig:lowd_gan}
\end{figure}

With proper configurations of network classes \eqref{eq:generator_lowd} and \eqref{eq:discriminator_lowd}, we train GANs using \eqref{eq:empiricalgan} (see Figure \ref{fig:lowd_gan} for an illustration) and denote the optimizer as $(U^*, g^*_\theta, V^*, f_\omega^*)$, i.e.,
\begin{align*}
(U^*, g^*_\theta, V^*, f_\omega^*) \in \argmin_{f_\omega \circ V^\top \in \cF_{\rm NN}^{\rm ld}} \max_{U \circ g_\theta \in \cG_{\rm NN}^{\rm ld}} & \EE_{z \sim \rho} \left[(f_\omega \circ V^\top) \circ (U \circ g_\theta)(z)\right] \\
& \quad - \frac{1}{n} \sum_{i=1}^n (f_\omega \circ V^\top)(x_i).
\end{align*}
The following theorem establishes a fast statistical rate of convergence of $(U^* \circ g_\theta^*)_\sharp \rho$ to data distribution $\mu$.
\begin{theorem}\label{thm:lowdbound}
Suppose Assumption \ref{assumption:lowdmanifold} and \ref{assumption:lowddensity} hold. We choose
\begin{equation*}
\begin{split}
& \quad\qquad R = B, \quad \kappa = \max\{B, C\}, \quad L = O\left(\frac{\alpha}{2\alpha + q}\log n\right), \\
& p = O\left(q n^{\frac{q\alpha}{(\alpha+1)(2\alpha + q)}} \vee d\right), \quad J = O\left(dq + \frac{\alpha}{2\alpha + q} n^{\frac{q\alpha}{(\alpha+1)(2\alpha + q)}} \log n \right).
\end{split}
\end{equation*}
for the generator $\cG_{\rm NN}^{\rm ld}(R, \kappa, L, p, J)$ in \eqref{eq:generator_lowd} and 
\begin{equation*}
\begin{split}
&\quad \bar{R} = C,\quad \bar{\kappa} = C, \quad \bar{\gamma} = 10q, \quad \bar{L} = O\left(\frac{1}{2 + q}\log n \right), \\
& \bar{p} = O\left(n^{q/(2+q)} \vee d\right),\quad \bar{J} = O\left(dq + \frac{1}{2 + q}n^{q/(2+q)}\log n \right).
\end{split}
\end{equation*}
for the discriminator $\cF_{\rm NN}^{\rm ld}(\bar{R}, \bar{\kappa}, \bar{L}, \bar{p}, \bar{J}, \bar{\gamma})$ in \eqref{eq:discriminator_lowd}. Then it holds
\begin{equation}
\EE\left[\dw ((U^*\circ g_\theta^*)_\sharp \rho, \mu)\right] = \tilde{O}\left(n^{-\frac{1}{2 + q}} \log^2 n\right). \nonumber
\end{equation}
\end{theorem}

Compared to Theorem \ref{thm:main}, we observe that the sizes of generator and discriminator in Theorem \ref{thm:lowdbound} crucially depend on $q$ and only weakly depend on $d$. Meanwhile, the rate of convergence is fast as the exponent only depends on $q$. This result provides important understandings of why GANs can circumvent the curse of dimensionality in real-world applications, since low-dimensional intrinsic structures are often seen in real-world data sets. Nonetheless, linear structures in Assumption \ref{assumption:lowdmanifold} is largely simplified, as it is rare the case that real-world data lie in a subset of a low-dimensional linear subspace (see a generalization to nonlinear mixture data in Section \ref{sec:mixture}). At the same time, real data are often contaminated with observational noise and concentrate only near a low-dimensional manifold.

Theorem \ref{thm:lowdbound} demonstrates that GANs with properly chosen generator and discriminator are adaptive to the unknown linear structures in data. Since data are concentrated on a linear subspace, one may advocate PCA-like methods for estimating the linear structure first and then learn the data distribution on a projected subspace. However, such a method requires two-step learning and is rarely used in practical GANs. In fact, GANs simultaneously capture the linear structure and learning the target data distribution via optimizing the empirical risk \eqref{eq:empiricalgan}.

A major difficulty in establishing Theorem \ref{thm:lowdbound} is proving GANs can capture the unknown linear structures in data. We exploit the optimality of $(U^*, g_\theta^*)$ to prove that $\norm{U^* - A}_{\rm F}$ is small, i.e., the column spaces of $U^*$ and the ground truth matrix $A$ match closely. In particular, the mismatch $\norm{U^* - A}_{\rm F}$ depends on the approximation power of the generator and discriminator (see Lemma \ref{lemma:subspace_match}). Built upon this crucial ingredient, the remaining analysis focuses on tackling the projected Wasserstein distance with respect to the data transformation $A$ (see \cite{wang2020two, wang2022manifold} for applications of projected Wasserstein distance in two-sample test). In this way, we circumvent the curse of ambient dimensionality.

\section{Distribution estimation in low-dimensional mixture model}\label{sec:mixture}
Section \ref{sec:lowd} provides a detailed study of GANs estimating target distributions with unknown linear structures. The obtained estimation guarantee enjoys a fast convergence rate, dependent on the linear subspace dimension. In this section, we generalize to data distributions with nonlinear intrinsic structures and show that GANs maintain the fast statistical estimation guarantee.

We consider target data distribution $\mu$ supported on $\cX \subset [0, 1]^d$ being a mixture of $K$ components. 
\begin{assumption}\label{assumption:mixture}
Data distribution $\mu$ takes the decomposition
\begin{align}\label{eq:mixture}
\mu = \sum_{k=1}^K p_k \mu_k \quad \text{with} \quad \sum_{k=1}^K p_k = 1,
\end{align}
where $p_k > 0$ is the prior of the corresponding component.

Moreover, each component $\mu_k$ is a pushforward distribution of the $q$-dimensional uniform distribution on $[0, 1]^q$ ($q \leq d$). In particular, $\mu_k = (g_k)_\sharp {\rm Unif}([0, 1]^q)$ and $g_k : [0, 1]^q \mapsto \RR^d$ is $\alpha_k$-H\"{o}lder continuous for some $\alpha_k \geq 1$. Moreover, there exists a constant $C_{\alpha} > 0$ such that $\norm{g_k}_{\cH^{\alpha_k}} \leq C_{\alpha}$ for any $k = 1, \dots, K$.
\end{assumption}
Mixture data are widely seen in practice. For example, MNIST data set is naturally clustered into $10$ groups corresponding to different handwritten digits. Images in CIFAR-10 and ImageNet can also be clustered according to labels. Moreover, the intrinsic dimensions of these data sets are all estimated to be much smaller than their ambient dimensions \cite{pope2021intrinsic}. In our mixture model, we can view the $q$-dimensional coordinates as the low-dimensional intrinsic parameters, and $g_k$ gives rise to a parametrization of data in the $k$-th component.

We can understand the mixture distribution $\mu$ as the marginal distribution of a random vector $x$, whose distribution further depends on a latent random variable $\xi$. Specifically, let $\xi$ be a categorical random variable with $\PP(\xi = k) = p_k$ for $k = 1, \dots, K$. Then we define $x ~|~ \xi = k \sim \mu_k$. It can be checked that the marginal distribution of $x$ is $\mu$. More importantly, introducing such a latent variable allows an easy sampling from the mixture distribution; we will choose the generator network based on this intuition.

Note that Assumption \ref{assumption:lowddensity} is a special case of Assumption \ref{assumption:mixture}, when taking $K = 1$. Assumption \ref{assumption:mixture} also suggests a natural partition of domain $\cX$. Let $\cX_k = g_k([0, 1]^q) = \{g_k(z) : z \in [0, 1]^q\}$. Then we have $\cX = \bigcup_{k=1}^K \cX_k$. This shares the same principle as a low-dimensional manifold embedded in $\RR^d$, as $\cX_k$ can be viewed as a local neighborhood on $\cX$ (although $g_k$ is not necessarily a homeomorphism between $\cX_k$ and $[0, 1]^q$).

We generate samples using a generator by transforming a uniform distribution $\rho$ on unit cube $[0, 1]^{q+1}$. In particular, we will use the first coordinate to mimic a latent variable and the remaining $q$ coordinates are used to generate $\mu_k$ in each component. We set the generator network architecture as
\begin{align*}
\cG_{\rm NN}^{\rm mix}(R, \kappa, L, p, J) = {\rm NN}(R, \kappa, L, p, J, d_{\rm in} = q+1, d_{\rm out} = d)  
\end{align*}
and establish a mixture distribution approximation theory.
\begin{proposition}\label{prop:mixture_approximation}
Suppose Assumption \ref{assumption:mixture} holds. Given any $\epsilon > 0$, we choose $\cG_{\rm NN}^{\rm mix}(R, \kappa, L, p, J)$ with
\begin{align*}
R = 1, \kappa = \max\{C_\alpha, 1\}, L = O\left(\log \frac{1}{\epsilon} \right), p = O\left(K d \epsilon^{-\frac{q}{\alpha}}\right), J = O\left(K d \epsilon^{-\frac{q}{\alpha}} \log \frac{1}{\epsilon} \right).
\end{align*}
Then there exists $g_\theta \in \cG_{\rm NN}^{\rm mix}$ such that
\begin{align*}
\dw\left((g_\theta)_\sharp \rho, \mu\right) \leq \sqrt{d} \epsilon.
\end{align*}
\end{proposition}
Proposition \ref{prop:mixture_approximation} is proved in Appendix \ref{pf:mixture}, which draws motivation from Theorem \ref{thm:generator}. At a coarse level, the generator network is to approximate $g_k$ in each component. Therefore, the network architecture consists of $K$ parallel transformations.

We set the discriminator network architecture as
\begin{align*}
\cF_{\rm NN}^{\rm mix}(\bar{R}, \bar{\kappa}, \bar{L}, \bar{p}, \bar{J}, \bar{\gamma}) = \Big\{f_\omega : & f_\omega \in {\rm NN}(\bar{R}, \bar{\kappa}, \bar{L}, \bar{p}, \bar{J}, d_{\rm in} = d, d_{\rm out} = 1), ~\text{and} \\
& \quad |f_\omega(x) - f_\omega(y)| \leq \bar{\gamma} \norm{x - y}_\infty ~\text{for}~x, y \in [0, 1]^d \Big\}.
\end{align*}
The discriminator network is to approximating $1$-Lipschitz discriminative functions on $[0, 1]^d$. We recall training GANs via optimizing \eqref{eq:empiricalgan} and the optimizer is denoted as $(g_\theta^*, f_\omega^*)$. The following theorem provides a fast finite-sample distribution estimation guarantee for learning low-dimensional mixture models.
\begin{theorem}\label{thm:mixture}
Suppose Assumption \ref{assumption:mixture} holds. Denote $\alpha = \min_k \alpha_k$. We choose
\begin{equation*}
\begin{split}
& R = 1, \quad \kappa = \max\{C_\alpha, 1\}, \quad L = O\left(\frac{1}{q}\log n\right), \\
& ~~~ p = O\left(K d n^{\frac{1}{\alpha}}\right), \quad J = O\left(K d n^{\frac{1}{\alpha}} \log n \right).
\end{split}
\end{equation*}
for the generator $\cG_{\rm NN}^{\rm mix}(R, \kappa, L, p, J)$ (corresponding to set $\epsilon_1 = n^{-1/q}$ in Proposition \ref{prop:mixture_approximation}) and 
\begin{equation*}
\begin{split}
& \bar{R} = \sqrt{d},\quad \bar{\kappa} = O(1), \quad \bar{\gamma} = 10d, \quad \bar{L} = O\left(\log n + d \right), \\
& ~\quad~ \bar{p} = O\left(n^{d/q}\right),\quad \bar{J} = O\left(n^{d/q}(\log n + d) \right).
\end{split}
\end{equation*}
for the discriminator $\cF_{\rm NN}^{\rm mix}(\bar{R}, \bar{\kappa}, \bar{L}, \bar{p}, \bar{J}, \bar{\gamma})$. Then for any positive constant $\delta > 0$, it holds
\begin{equation}
\EE\left[\dw \left((g_\theta^*)_\sharp \rho, \mu\right)\right] \leq C_\delta d n^{-\frac{1}{q + \delta}}, \nonumber
\end{equation}
where $C_\delta$ is independent of $n$.
\end{theorem}
The proof is deferred to Appendix \ref{pf:mixture}. We discuss implications of Theorem \ref{thm:mixture}.
\begin{enumerate}
\item {\it Fast rate}. We obtain a fast rate of convergence in estimating the low-dimensional mixture distribution $\mu$. This result provides a generalization of linear data in Theorem \ref{thm:lowdbound}, and better explains the empirical success of GANs in practice. We remark that the rate of convergence in Theorem \ref{thm:mixture} depends on a positive constant $\delta > 0$. When $n$ is sufficiently large, the convergence rate is arbitrarily close to $n^{-1/q}$, yet is always marginally slower. Theorem \ref{thm:main} and \ref{thm:lowdbound} share the same spirit by providing explicit logarithmic factors in $n$.

\item {\it Relation to Theorem \ref{thm:lowdbound}}. Theorem \ref{thm:lowdbound} considers a special case of Assumption \ref{assumption:mixture} and provides a fine-grained analysis. In particular, in addition to a distribution estimation guarantee, Theorem \ref{thm:lowdbound} shows that GANs are capable of accurately recover the unknown linear structures in data (Lemma \ref{lemma:subspace_match}). Theorem \ref{thm:mixture} generalizes Theorem \ref{thm:lowdbound} and shows that GANs are capable of capturing nonlinear data intrinsic structures encoded by $g_k$'s.

In addition, Theorem \ref{thm:lowdbound} indicates a pairing between generator and discriminator networks, in that, they have comparable sizes (exponentially dependent on $q$) and there is a coupling between the projection and ``anti-projection'' layers. Theorem \ref{thm:mixture} is slightly different: The generator is chosen to learn the low-dimensional nonlinear transformation $g_k$'s and the discriminator is capable of approximating any $1$-Lipschitz function in $\RR^d$. We observe that the size of discriminator in Theorem \ref{thm:mixture} is exponentially dependent on $d$. In this sense, we view the discriminator in Theorem \ref{thm:mixture} being overparameterized. This overparameterization, however, does not undermine the fast statistical convergence, largely due to the Lipschitz continuity property of the discriminator.
\end{enumerate}

\section{Related work}\label{sec:related}

\noindent $\bullet$ {\bf Distribution approximation using deep generative models}. Using generator to accurately approximate the data distribution is of essential importance in understanding the statistical properties of GANs. \cite{bai2018approximability} considered data distribution being exactly realized by an invertible generator, i.e., all the weight matrices and activation functions are invertible. Such an invertibility requires the width of the generator to be the same as input data dimension $d$. Existing literature has shown that such narrow networks lack approximation ability \cite{lu2017expressive, park2020minimum}. In fact, to ensure universal approximation for Lebesgue-integrable functions and $L^p$ functions in $\RR^d$, the weakest width requirement needs to be $d+4$ and $d+1$, respectively. Our work, in contrast, allows the generator to be wide and expressive for any data distribution with H\"{o}lder densities.

\noindent $\star$ {\it Approximating empirical distribution using neural networks}. After the release of an early version of the manuscript, the authors were aware of a concurrent work studying distribution approximation using generative networks. Specifically, \cite{lu2020universal} established universal approximation abilities of neural network generators for approximating sub-Gaussian data distributions. They proved the existence of a properly chosen generator architecture for achieving an $\epsilon$ approximation error of data distribution in Wasserstein-1 distance. Our Theorem \ref{thm:generator} shares a similar conclusion to \cite{lu2020universal} for data distributions with H\"{o}lder densities. However, the analysis in \cite{lu2020universal} is very different and relies on memorizing discretized data distribution using neural networks. More recently, \cite{huang2022error} showed that GANs can approximate any data distribution (in any dimension) by transforming an absolutely continuous distribution. The idea is to memorize the empirical data distribution using ReLU networks. Nonetheless, the designed generator is not able to generate new samples (different from the training data), which cannot explain the success of GANs in practice.

\noindent $\bullet$ {\bf Statistical properties of GANs}. Statistical guarantees of generative models for distribution estimation has been studied in several works. We compare with existing works in Table \ref{tab:literature}.
\begin{table}[h]
\caption{A comparison of this paper to closely related works in problem setups and statistical results. NN stands for neural networks and `---' indicates no specific choice is given. Weak metric refers to ``neural net distance'' in \cite{arora2017generalization} and strong metric refers to IPMs with nonparametric discriminative function classes, e.g., using $1$-Lipschitz discriminative functions corresponds to the Wasserstein-1 distance.}
\label{tab:literature}
\centering
\begin{tabular}{ p{5em} | p{6em} | p{6em} | p{10.25em} | p{4em}}
\hline
& Generator & Discriminator & Distribution & Metric  \\\hline
\multicolumn{5}{c}{{\it Generalization error bound}} \\\hline
\cite{arora2017generalization, zhang2017discrimination, jiang2018computation} & NN & NN & General \hspace{0.8in} (Euclidean) & Weak \\\hline
\cite{bai2018approximability, liang2017well, liang2018well} & Invertible NN & NN & Realizable by invertible NN generators \hspace{0.8in} (Eulidean and low-d) & Strong \\\hline 
\cite{luise2020generalization} & $C^s$ & --- & $C^s$ pushforward of sub-Gaussian distributions (Eulidean and low-d) & Sinkhorn \\\hline
\cite{Block2021ANE} & --- & H\"{o}lder & General \hspace{0.8in} (Euclidean and low-d) & Strong \\\hline
\multicolumn{5}{c}{\it Statistical estimation bound} \\\hline
\cite{schreuder2021statistical} & NN & Lipschitz & $C^s$ pushforward of uniform distributions \hspace{0.3in} (low-d) & Strong \\\hline
{\bf Ours} & NN & NN & Having H\"{o}lder densities \hspace{0.6in} (Euclidean and low-d) & Strong \\\hline
\end{tabular}
\label{tab:comparison}
\end{table}

\noindent $\star$ {\it Generalization bound of GANs}. \cite{arora2017generalization} studied the generalization error of GANs. Lemma 1 in \cite{arora2017generalization} shows that GANs cannot generalize under the Wasserstein distance and the Jensen-Shannon divergence unless the sample size is $\tilde{O}(\epsilon^{-\text{poly}(d)})$, where $\epsilon$ is the generalization gap. Alternatively, they defined a surrogate metric called ``neural net distance'' $d_{\cF_{\textrm{NN}}}(\cdot, \cdot)$, where $\cF_{\textrm{NN}}$ is the class of discriminator networks. They proved that GANs generalize under the neural net distance, with sample complexity of $\tilde{O}(\epsilon^{-2})$. This result has two limitations: 1). The sample complexity depends on some unknown parameters of the discriminator network class (e.g., the Lipschitz constant of discriminators with respect to parameters); 2). A small neural net distance does not necessarily imply that two distributions are close \cite[Corollary 3.2]{arora2017generalization}, which in turn can not answer {\it (Q1)} firmly. Our results are explicit in the network architectures, and provide a statistical convergence of GANs under the Wasserstein distance.

Some follow-up works attempted to address the first limitation in \cite{arora2017generalization}. \cite{zhang2017discrimination} explicitly quantified the Lipschitz constant and the covering number of the discriminator network. They improved the generalization bound in \cite{arora2017generalization} with the technique in \cite{bartlett2017spectrally}. Whereas the bound has an exponential dependence on the depth of the discriminator. \cite{jiang2018computation} proved a tighter generalization bound under spectral normalization applied to the discriminator, where the bound has a polynomial dependence on the size of the discriminator. These generalization theories rely on the assumption that the generator can approximate the data distribution well with respect to the neural net distance, nonetheless, the existence of such a generator is unknown.

\cite{bai2018approximability} tackled the second limitation in \cite{arora2017generalization}, and studied the estimation error of GANs under the Wasserstein distance for a special class of distributions implemented by a generator, while the discriminator is designed to guarantee zero bias (or approximation error). Specifically, \cite{bai2018approximability} showed that for certain generator classes, there exist corresponding discriminator classes with a strong discriminative power against the generator. Particular examples include two-layer ReLU network discriminators (half spaces) for distinguishing Gaussian distributions/mixture of Gaussians, and $(L+2)$-layer discriminators for $(L+1)$-layer invertible generators. In these examples, if the data distribution can be exactly implemented by some generator, then the neural net distance can provably approximate the Wasserstein distance. Consequently, GANs can generalize under the Wasserstein distance. As mentioned earlier, these results require an invertibility assumption on the generator.

Concurrent with \cite{bai2018approximability}, \cite{liang2018well} studied the estimation error of GANs under the Sobolev IPMs. \cite{liang2018well} considered both nonparametric and parametric settings. In the nonparametric setting, the generator and discriminator network architectures are not explicitly chosen, so the bias of the distribution estimation remains unknown. As a result, the bound cannot provide an explicit sample complexity for distribution estimation. Their parametric results are very similar to \cite{bai2018approximability}, which requires the same invertibility assumptions and the data distribution needs to be exactly implementable by the generator. 

\noindent $\star$ {\it Generative distribution estimation under IPMs}.
Recently, several works studied distribution estimation under certain discrepancy measures using generative models, when data exhibit low-dimensional structures \cite{luise2020generalization, schreuder2021statistical, Block2021ANE}. The distribution estimation framework is
\begin{align*}
g^* \in \argmin_{g \in \cG}~ {\tt discrepancy}(g_\sharp \rho, \mu)
\end{align*}
and the corresponding statistical rate of estimation is free of the curse of data ambient dimensionality. Specifically, in \cite{luise2020generalization}, the generative models are assumed to be continuously differentiable up to order $s$. By simultaneously optimize the choice of latent distribution $\rho$ and generative model $g$, they proved that the Sinkhorn divergence between the generated distribution and data distribution converges only depending on data intrinsic dimension. \cite{schreuder2021statistical} consider data being generated by a ground truth pushfowrad mapping applied to latent samples from a low-dimensional unit cube, which is a special case of $K = 1$ in our Assumption \ref{assumption:mixture}. Using Lipschitz generator, they proved that the generalization bound in terms of Wassesrstein-1 distance converges only depending on the dimension of the latent space. More recently, \cite{Block2021ANE} established a generalization bound in terms of H\"{o}lder IPMs for generative models and the bound converges depending on data intrinsic dimension. Nonetheless, how well the generator can represent the data distribution remains unclear. It is worth mentioning that \cite{chae2021likelihood} considered estimating low-dimensional singular distributions using deep generative models. They adopted a likelihood approach, which is different from GANs. All of the aforementioned results rely on training the generative model by minimizing certain discrepancy metric, e.g., Wasserstein-1 distance and Sinkhorn divergence. There is no explicit discriminator network involved, while our analysis considers neural network discriminators and still provides statistical guarantees of GANs in terms of Wasserstein-1 distance.

\noindent $\star$ {\it Density estimation under IPMs}. There is also a line of works considering nonparametric density estimation under IPMs \cite{singh2018nonparametric,uppal2019nonparametric}. \cite{singh2018nonparametric} studied the minimax error under Sobolev IPMs. Later, \cite{uppal2019nonparametric} generalized the minimax result to Besov IPMs for estimating distributions with Besov densities. Yet our work is different from these works. Specifically, the distribution estimation framework in \cite{singh2018nonparametric,uppal2019nonparametric} is
\begin{align}
\min_{\nu \in \cP} \max_{f_\omega \in \mathcal{F}_{\rm NN}} \int f_\omega(y) \nu(y) dy - \frac 1 n \sum_{i=1}^n f_{\omega}(x_i), \label{eq:notgan}
\end{align}
where $\nu: \cX \mapsto \RR$ is a density function and $\cP$ denotes a class of density estimators, such as the wavelet-thresholding estimator in \cite{uppal2019nonparametric}. Compared to our framework in \eqref{eq:empiricalgan}, we consider the push-forward structure in GANs, where the generator $g_\theta$ is a multidimensional mapping. In contrast, \eqref{eq:notgan} considers density estimators, where $\nu : \cX \mapsto \RR$ is some density function parameterized by a neural network --- involving NO generator architecture which transforms the easy-to-sample distribution to the data distribution. Moreover, to evaluate the integral in \eqref{eq:notgan}, one needs to exactly know the feature space $\cX$, and efficiently sample from $\cX$. Consequently, \cite{singh2018nonparametric,uppal2019nonparametric} only apply to $\cX = [0, 1]^d$. Besides, only estimating the density function requires extensive extra efforts to sample from it, e.g., using Monte Carlo simulation, due to the lack of the push-forward structure. However, our theories are applicable to push-forward GANs, and allow an efficient sampling of generated (fake) data.

\section{Proof of statistical theory in Euclidean space}\label{sec:proof}

We provide proofs of Theorem \ref{thm:generator} and \ref{thm:main}. The developed analytical framework will also be adopted for proving Theorem \ref{thm:lowdbound} with additional treatments on low-dimensional structures in Section \ref{sec:lowdproof}.

\subsection{Proof of Theorem \ref{thm:generator}}\label{sec:proofthmgenerator}

Theorem \ref{thm:generator} is obtained by combining Lemma \ref{lemma:monge} and \ref{lemma:approx}. Under Assumption \ref{assump1} - \ref{assump3}, Lemma \ref{lemma:monge} ensures the existence of a $\cH^{\alpha+1}(\cZ)$ data transformation $T$ such that $T_\sharp \rho = \mu$. The remaining step is to choose a proper generator network for approximating $T$.

If the latent space $\cZ \subset [0, 1]^d$, we can directly apply Lemma \ref{lemma:approx} for constructing the generator. Otherwise, if $\cZ \subset [-B, B]^d$, we define a linear scaling function $\phi(z) = (z + B \mathbf{1}) / (2B) \in [0, 1]^d$ for any $z \in \cZ$, where $\mathbf{1}$ denotes a vector of $1$'s. For the data transformation $T$, we rewrite it as $T \circ \phi^{-1} (\phi(\cdot))$ so that it suffices to approximate $T \circ \phi^{-1}$ supported on $[0, 1]^d$. $T\circ \phi^{-1}$ retains the same H\"{o}lder smoothness as $T$, since $\phi$ is invertible and linear. To this end, without loss of generality, we focus on $\cZ \subset [0, 1]^d$.

Our generator network architecture is constructed in the following way. By denoting $T = [T_1, \dots, T_d]^\top$ with $T_i: \cZ \rightarrow \RR$ for $i=1,\ldots,d$, we approximate each coordinate mapping $T_i$ using Lemma \ref{lemma:approx}. For a given error $\epsilon \in (0, 1)$, $T_i$ can be approximated by a ReLU network with $O\left(\log \frac{1}{\epsilon}\right)$ layers and $O\big(\delta^{-\frac{d}{\alpha+1}} \log \frac{1}{\epsilon}\big)$ neurons and weight parameters. Thus, mapping $T$ can be approximated by $d$ such networks and we denote as $g_\theta$. Further, the distribution approximation error is
\begin{align*}
\dw((g_\theta)_\sharp \rho, \mu) & = \sup_{\norm{f}_{\rm Lip}\leq 1} \EE_{z \sim \rho} [f(g_\theta(z))] - \EE_{x \sim \mu} [f(x)] \\
& \leq \EE_{z \sim \rho} \norm{g_\theta(z) - T(z)}_2 \\
& \leq \sqrt{d} \epsilon.
\end{align*}

\subsection{Proof of Theorem \ref{thm:main}}\label{sec:proofthmmain}
We prove an oracle inequality for establishing Theorem \ref{thm:main}, which decomposes the distribution estimation error into the generator approximation error $\cE_1$, the discriminator approximation error $\cE_2$, and the statistical error $\cE_3$.
\begin{lemma}\label{lemma:oracle}
Let $\cH^\beta(\cX)$ be the H\"{o}lder function class defined on $\cX$ with H\"{o}lder index $\beta \geq 1$. Define $\cH^{\beta}_\infty(\cX) = \left\{f \in \cH^\beta(\cX) : \lvert f(x) -  f(y)\rvert \leq \|x - y\|_\infty\right\}$. Then it holds
\begin{align*}
d_{\cH^\beta} ((g^*_\theta)_\sharp \rho, \mu) & \leq \cE_1+4\cE_2+\cE_3,
\end{align*}
where $\cE_1=\inf_{g_\theta \in \cG_{\rm NN}} d_{\cH^\beta_\infty}\left((g_\theta)_{\sharp} \rho, \mu\right),~\cE_2=\sup_{f \in \cH^{\beta}} \inf_{f_{\omega} \in \cF_{\rm NN}} \norm{f - f_{\omega}}_\infty$, and $\cE_3=d_{\cH^{\beta}}\left(\mu, \hat{\mu}_n\right) + d_{\cF_{\rm NN}} \left(\mu, \hat{\mu}_n\right)$.
\end{lemma}

\begin{proof}[Proof of Lemma \ref{lemma:oracle}]
We introduce the empirical data distribution as an intermediate term for bounding $d_{\cH^\beta}((g_\theta^*)_\sharp \rho, \mu)$. Using the triangle inequality, we derive 
\begin{align}\label{eq:recoverbound}
& \quad~ d_{\cH^\beta}((g^*_\theta)_\sharp \rho, \mu) \nonumber \\
& \leq d_{\cH^\beta}((g^*_\theta)_\sharp \rho, \hat{\mu}_n) + d_{\cH^\beta}(\hat{\mu}_n, \mu) \notag\\
& = d_{\cF_\textrm{NN}}((g^*_\theta)_\sharp \rho, \hat{\mu}_n) + d_{\cH^\beta}((g^*_\theta)_\sharp \rho, \hat{\mu}_n) - d_{\cF_\textrm{NN}}((g^*_\theta)_\sharp \rho, \hat{\mu}_n) \nonumber \\
& \quad + d_{\cH^\beta}(\hat{\mu}_n, \mu) \notag \\
& \overset{(i)}{\leq} d_{\cF_\textrm{NN}}((g^*_\theta)_\sharp \rho, \hat{\mu}_n) + 2 \sup_{f \in \cH^\beta} \inf_{f_{\omega} \in \cF_{\textrm{NN}}} \norm{f - f_{\omega}}_\infty + d_{\cH^\beta}(\hat{\mu}_n, \mu),
\end{align}
where step $(i)$ is obtained by rewriting $d_{\cH^\beta}((g^*_\theta)_\sharp \rho, \hat{\mu}_n) - d_{\cF_\textrm{NN}}((g^*_\theta)_\sharp \rho, \hat{\mu}_n)$ as
\begin{align*}
&~~\quad d_{\cH^\beta}((g^*_\theta)_\sharp \rho, \hat{\mu}_n) - d_{\cF_\textrm{NN}}((g^*_\theta)_\sharp \rho, \hat{\mu}_n) \\
& = \sup_{f \in \cH^\beta} \left[\EE_{x \sim (g_\theta^*)_\sharp \rho} [f(x)] - \EE_{x \sim \hat{\mu}_n} [f(x)]\right] \\
& \qquad - \sup_{f_\omega \in \cF_\textrm{NN}} \left[ \EE_{x \sim (g_\theta^*)_\sharp \rho} [f_\omega(x)] - \EE_{x\sim \hat{\mu}_n} [f_\omega(x)]\right] \\
& =  \sup_{f \in \cH^\beta} \inf_{f_\omega \in \cF_\textrm{NN}} \left[\EE_{x \sim (g_\theta^*)_\sharp \rho} [f(x)] - \EE_{x \sim \hat{\mu}_n} [f(x)]\right] \\
& \qquad - \left[ \EE_{x \sim (g_\theta^*)_\sharp \rho} [f_\omega(x)] - \EE_{x\sim \hat{\mu}_n} [f_\omega(x)]\right] \\
& = \sup_{f \in \cH^\beta} \inf_{f_\omega \in \cF_\textrm{NN}} \EE_{x \sim (g_\theta^*)_\sharp \rho} [f(x) - f_{\omega}(x)] - \EE_{x\sim \hat{\mu}_n} [f(x) - f_{\omega}(x)] \\
& \leq \sup_{f \in \cH^\beta} \inf_{f_\omega \in \cF_{\rm NN}} \EE_{x \sim (g_\theta^*)_\sharp \rho} [\lvert f(x) - f_{\omega}(x)\rvert] + \EE_{x\sim \hat{\mu}_n} [\lvert f(x) - f_{\omega}(x) \rvert] \\
& \leq 2 \sup_{f \in \cH^\beta} \inf_{f_\omega \in \cF_\textrm{NN}} \norm{f - f_{\omega}}_\infty.
\end{align*}
Now we bound $d_{\cF_\textrm{NN}}((g^*_\theta)_\sharp \rho, \hat{\mu}_n)$ using a similar triangle inequality trick:
\begin{align*}
d_{\cF_\textrm{NN}}((g^*_\theta)_\sharp \rho, \hat{\mu}_n) & = \inf_{g_\theta \in \cG_{\textrm{NN}}} d_{\cF_\textrm{NN}}((g_\theta)_{\sharp}\rho, \hat{\mu}_n) \\
& \leq \inf_{g_\theta \in \cG_{\textrm{NN}}} d_{\cF_\textrm{NN}}((g_\theta)_{\sharp}\rho, \mu) + d_{\cF_\textrm{NN}}(\mu, \hat{\mu}_n) \\
& = \inf_{g_\theta \in \cG_{\textrm{NN}}} d_{\cF_\textrm{NN}}((g_\theta)_{\sharp}\rho, \mu) - d_{\cH^\beta_\infty} ((g_\theta)_{\sharp}\rho, \mu) \\
& \quad + d_{\cH^\beta_\infty} ((g_\theta)_{\sharp}\rho, \mu) + d_{\cF_\textrm{NN}}(\mu, \hat{\mu}_n) \\
& \leq 2 \sup_{f \in \cH^\beta} \inf_{f_\omega \in \cF_\textrm{NN}} \norm{f - f_{\omega}}_\infty + \inf_{g_\theta \in \cG_{\textrm{NN}}} d_{\cH^\beta_\infty} ((g_\theta)_{\sharp} \rho, \mu) \\
& \quad + d_{\cF_\textrm{NN}}(\mu, \hat{\mu}_n),
\end{align*}
where the last inequality holds by the identity $\cH^\beta_\infty \subset \cH^\beta$. Substituting the above ingredients into \eqref{eq:recoverbound}, we have
\begin{align*}
d_{\cH^\beta}((g_\theta^*)_{\sharp} \rho, \mu) & \leq \underbrace{\inf_{g_\theta \in \cG_{\textrm{NN}}} d_{\cH^\beta_\infty}((g_\theta)_{\sharp}\rho, \mu)}_{\cE_1:~ \textrm{generator approximation error}} \\
& \quad + 4 \underbrace{\sup_{f \in \cH^\beta} \inf_{f_\omega \in \cF_\textrm{NN}} \norm{f - f_{\omega}}_\infty}_{\cE_2:~ \textrm{discriminator approximation error}} \\
& \quad + \underbrace{d_{\cH^\beta}(\hat{\mu}_n, \mu) + d_{\cF_\textrm{NN}}(\mu, \hat{\mu}_n)}_{\cE_3:~ \textrm{statistical error}}.
\end{align*}
The proof is complete.
\end{proof}

We next bound each error term separately. $\cE_1$ and $\cE_2$ can be controlled by proper choices of the generator and discriminator architectures. $\cE_3$ can be controlled based on empirical process \cite{van1996weak, gyorfi2006distribution}.

\noindent $\bullet$ \textbf{Bounding generator approximation error $\cE_1$}.
We answer this question: Given $\epsilon_1 \in (0,1)$, how can we properly choose $\cG_{\textrm{NN}}$ to guarantee $\cE_1 \leq \epsilon_1$? Later, we will pick $\epsilon_1$ based on the sample size $n$, and H\"{o}lder indexes $\beta$ and $\alpha$.
\begin{lemma}\label{lemma:generator}
Given any $\epsilon_1 \in (0, 1)$, there exists a ReLU network architecture $\cG_\textrm{NN}(R, \kappa, L, p, K)$ with parameters given by \eqref{eq:gnnsize} with $\epsilon =\epsilon_1$
such that, for any data distribution $(\cX, \mu)$ and easy-to-sample distribution $(\cZ, \rho)$ satisfying Assumptions \ref{assump1} -- \ref{assump3}, if the weight parameters of this network are properly chosen, then it yields a transformation $g_\theta$ satisfying $d_{\cH^\beta_\infty}((g_\theta)_\sharp \rho, \mu) \leq \epsilon_1$.
\end{lemma}

\begin{proof}[Proof of Lemma \ref{lemma:generator}]
Without loss of generality, we assume $\cZ = \cX = [0, 1]^d$. Otherwise, we can rescale the domain to be a subset of $[0, 1]^d$. By Monge map (Lemma \ref{lemma:monge}), there exists a mapping $T = [T_1, \dots, T_d]: \cZ \mapsto \cX$ such that $T_\sharp \nu = \mu$. Such a mapping is H\"{o}lder continuous, i.e., each coordinate mapping $T_i$ for $i = 1, \dots, d$ belongs to $\cH^{\alpha+1}$. We approximate each function $T_i$ using the network architecture identified in Lemma \ref{lemma:approx}. Specifically, given approximation error $\delta \in (0, 1)$. There exists a network architecture with no more than $c (\log \frac{1}{\delta} + 1)$ layers and $c' \delta^{-\frac{d}{\alpha+1}}(\log \frac{1}{\delta} + 1)$ neurons and weight parameters, such that with properly chosen weight parameters, yields an approximation $\hat{T}_i$ of $T_i$ satisfying $\| \hat{T}_i - T_i \|_\infty \leq \delta$. Applying this argument $d$ times, we form an approximation $g_\theta = [\hat{T}_1, \dots, \hat{T}_d]$ of $T$. We show $(g_\theta)_\sharp \rho$ satisfies the following IPM bound
\begin{align*}
d_{\cH^\beta_\infty}((g_\theta)_\sharp \rho, \mu) & = d_{\cH^\beta} ((g_\theta)_\sharp \rho, T_\sharp \rho) \\
& = \sup_{f \in\cH^\beta} \EE_{x \sim (g_\theta)_\sharp \rho}[f(x)] - \EE_{y \sim T_\sharp \rho} [f(y)] \\
& = \sup_{f \in \cH^\beta} \EE_{z \sim \rho}[f(g_\theta(z))] - \EE_{z \sim \rho} [f(T(z))] \\
& \leq \EE_{z \sim \rho} \left[\norm{g_\theta(z) - T(z)}_\infty \right] \\
& = \EE_{z \sim \rho} \left[\norm{[\hat{T}_1(z) - T_1(z), \dots, \hat{T}_d(z) - T_d(z)]^\top}_\infty \right] \\
& \leq \delta.
\end{align*}
Therefore, choosing $\delta = \epsilon_1$ gives rise to $d_{\cH^\beta_\infty}((g_\theta)_\sharp \rho, \mu) \leq \epsilon_1$.
\end{proof}

\noindent $\bullet$ \textbf{Bounding discriminator approximation error $\cE_2$}.
Analogous to the generator, we pre-define an error $\epsilon_2 \in (0, 1)$, and determine the discriminator architecture. 

The discriminator is expected to approximate any function $f \in \cH^\beta(\cX)$. We have the following result. 
\begin{lemma}\label{lemma:discriminator}
Given any $\epsilon_2 \in (0, 1)$, there exists a ReLU network architecture $\cF_{\textrm{NN}}(\bar{R}, \bar{\kappa}, \bar{L}, \bar{p}, \bar{J})$ with
\begin{align*}
& \bar{L} = O\big(\log (1/\epsilon_2)\big), \quad \bar{p} = O\big(\epsilon_2^{-d/\beta} \big), \quad \bar{J} = O\big(\epsilon_2^{-d/\beta} \log (1/\epsilon_2) \big), \\
& \hspace{1.3in} \bar{R} = C, \quad \bar{\kappa} = C,
\end{align*}
such that, for any discriminative function $f \in \cH^\beta(\cX)$, if the weight parameters are properly chosen, this network architecture yields a function $f_{\omega}$ satisfying $\| f_{\omega} - f \|_\infty \leq \epsilon_2$.
\end{lemma}

\begin{proof}[Proof of Lemma \ref{lemma:discriminator}]
Using Lemma \ref{lemma:approx} immediately yields a network architecture for uniformly approximating functions in $\cH^\beta(\cX)$. Specifically, let the approximation error be $\epsilon_2 > 0$. We choose the network architecture $\cF_{\rm NN}$ consisting of $\bar{L} = O\big(\log (1/\epsilon_2)\big)$ layers and $\bar{K} = O\big(\epsilon_2^{-d/\beta} \log(1/\epsilon_2)\big)$ total number of neurons and weight parameters. The maximum width is $\bar{p} = O\big(\epsilon_2^{-d/\beta}\big)$. Meanwhile, for any function $f \in \cH^\beta(\cX)$, we have $\norm{f}_{\cH^\beta} \leq C$. Threfore, it is enough to choose $\bar{R} = C$ and $\bar{\kappa} = C$. Accordingly, for any $f \in \cH^\beta(\cX)$, there exists a function $\hat{f}_\omega$ given by the network architecture $\cF_{\rm NN}(\bar{R}, \bar{\kappa}, \bar{L}, \bar{p}, \bar{K})$, such that $\|f - \hat{f}_\omega\|_\infty \leq \epsilon_2$. To this end, we can establish that for any $f \in \cH^\beta(\cX)$, inequality $\inf_{f_{\omega} \in \cF_{\textrm{NN}}} \|f - f_{\omega}\|_\infty \leq \epsilon_2$ holds.
\end{proof}

\noindent $\bullet$ \textbf{Bounding statistical error $\cE_3$}.
The statistical error term is essentially the concentration of empirical data distribution $\hat{\mu}_n$ to its population counterpart. Given a symmetric function class $\cF$, we show $\EE \left[d_{\cF}(\hat{\mu}_n, \mu)\right]$ scales with the complexity of the function class $\cF$.
\begin{lemma}\label{lemma:stat}
For a symmetric function class $\cF$ with $\sup_{f \in \cF} \norm{f}_\infty \leq M$ for a constant $M$, we have
\begin{align*}
\EE \left[d_{\cF}(\hat{\mu}_n, \mu)\right]  \leq 2 \inf_{0 < \delta < M} \Big(2\delta + \frac{12}{\sqrt{n}} \int_{\delta}^{M} \sqrt{\log \cN(\epsilon, \cF, \norm{\cdot}_\infty)} d\epsilon \Big),
\end{align*}
where $\cN(\epsilon, \cF, \norm{\cdot}_\infty)$ denotes the $\epsilon$-covering number of $\cF$ with respect to the $L_\infty$ norm.
\end{lemma}

\begin{proof}[Proof of Lemma \ref{lemma:stat}]
The proof utilizes the symmetrization technique and Dudley's entropy integral, which can be found in empirical process theory \cite{dudley1967sizes, van1996weak}. We prove here for completeness. Let $y_1, \dots, y_n$ be i.i.d. samples from $\mu$, independent of $x_i$'s. By symmetrization, we derive
\begin{align*}
\EE [d_{\cF}(\hat{\mu}_n, \mu)] & = \EE \left[\sup_{f \in \cF} \frac{1}{n} \sum_{i=1}^n f(x_i) - \EE_{y \sim \mu} [f(y)]\right] \\
& = \EE \left[\sup_{f \in \cF} \frac{1}{n} \sum_{i=1}^n f(x_i) - \EE_{\substack{y_i \sim \mu, \\ i = 1, \dots, n}} \frac{1}{n} \sum_{i=1}^n f(y_i)\right] \\
& \leq \EE_{x} \EE_y \left[\sup_{f \in \cF} \frac{1}{n} \sum_{i=1}^n (f(x_i) - f(y_i)) \right] \\
& = \EE_x \EE_y \EE_{\xi} \left[\sup_{f \in \cF} \frac{1}{n} \sum_{i=1}^n \xi_i (f(x_i) - f(y_i)) \right] \\
& = 2 \EE_{x, \xi} \left[\sup_{f \in \cF} \frac{1}{n} \sum_{i=1}^n \xi_i f(x_i) \right],
\end{align*}
where $\xi_i$'s are i.i.d. Rademacher random variables, i.e., $\PP(\xi_i = 1) = \PP(\xi_i = -1) = \frac{1}{2}$. The next step is to discretize the function space $\cF$. Let $\{\delta_i\}_{i=1}^k$ be a decreasing series of real numbers with $\delta_{i+1} < \delta_i$. We construct a collection of coverings on $\cF$ under the function $\ell_\infty$ norm with accuracy $\delta_i$. Denote the $\delta_i$-covering number as $\cN(\delta_i, \cF, \norm{\cdot}_\infty)$. For a given $f$, denote the closest element (in the $\ell_\infty$ sense) to $f$ in the $\delta_i$ covering as $f^{(i)}$ for $i = 1, \dots, k$. We expand $\EE_{x, \xi} \left[\sup_{f \in \cF} \frac{1}{n} \sum_{i=1}^n \xi_i f(x_i) \right]$ as a telescoping sum as
\begin{align*}
\EE_{x, \xi} \left[\sup_{f \in \cF} \frac{1}{n} \sum_{i=1}^n \xi_i f(x_i) \right] & \leq \EE \left[\sup_{f \in \cF} \frac{1}{n} \sum_{i=1}^n \xi_i (f(x_i) - f^{k}(x_i))\right] \\
& \quad + \sum_{j=1}^{k-1} \EE \left[\sup_{f \in \cF} \frac{1}{n} \sum_{i=1}^n \xi_i(f^{(j+1)}(x_i) - f^{(j)}(x_i))\right] \\
& \quad + \EE \left[\sup_{f \in \cF} \frac{1}{n} \sum_{i=1}^n \xi_i f^{(1)}(x_i) \right].
\end{align*}
We choose $\delta_1 = \textrm{diam}(\cF)$, i.e., the diameter of the class $\cF$. Then $f^{(1)}$ can be arbitrarily picked from $\cF$. Therefore, the last term $\EE \left[\sup_{f \in \cF} \frac{1}{n} \sum_{i=1}^n \xi_i f^{(1)}(x_i) \right] = 0$ since $\xi_i$'s are symmetric. The first term $\EE \left[\sup_{f \in \cF} \frac{1}{n} \sum_{i=1}^n \xi_i (f(x_i) - f^{k}(x_i))\right]$ can be bounded by Cauchy-Schwarz inequality:
\begin{align*}
& \quad~ \EE \left[\sup_{f \in \cF} \frac{1}{n} \sum_{i=1}^n \xi_i (f(x_i) - f^{k}(x_i))\right] \\
& \leq \EE \left[\sup_{f \in \cF} \frac{1}{n} \sqrt{(\sum_{i=1}^n \xi_i^2)(\sum_{i=1}^n (f(x_i) - f^{(k)}(x_i))^2)}\right] \\
& \leq \delta_k.
\end{align*}
We now bound each term in the telescoping sum $$\sum_{j=1}^{k-1} \EE \left[\sup_{f \in \cF} \frac{1}{n} \sum_{i=1}^n \xi_i(f^{(j+1)}(x_i) - f^{(j)}(x_i))\right].$$ Observe 
\begin{align*}
\norm{f^{(j+1)} - f^{(j)}}_\infty & = \norm{f^{(j+1)} - f + f - f^{(j)}}_\infty \\
& \leq \norm{f^{(j+1)} - f}_\infty + \norm{f - f^{(j)}}_\infty \\
& \leq \delta_{j+1} + \delta_j.
\end{align*}
By Massart's lemma \cite[Theorem 3.7]{mohri2018foundations}, we have
\begin{align*}
&~ \quad \EE \left[\sup_{f \in \cF} \frac{1}{n} \sum_{i=1}^n \xi_i(f^{(j+1)}(x_i) - f^{(j)}(x_i))\right] \\
& \leq \frac{(\delta_{j+1} + \delta_j)\sqrt{2\log (\cN(\delta_j, \cF, \norm{\cdot}_\infty) \cN(\delta_{j+1}, \cF, \norm{\cdot}_\infty))}}{\sqrt{n}} \\
& \leq \frac{2(\delta_{j+1} + \delta_j) \sqrt{\log \cN(\delta_{j+1}, \cF, \norm{\cdot}_\infty)}}{\sqrt{n}}.
\end{align*}
Summing up all the terms indexed by $j$, we establish
\begin{align*}
\EE_{x, \xi} \left[\sup_{f \in \cF} \frac{1}{n} \sum_{i=1}^n \xi_i f(x_i) \right] \leq \delta_k + 2 \sum_{j=1}^{k-1} \frac{(\delta_{j+1} + \delta_j) \sqrt{\log \cN(\delta_{j+1}, \cF, \norm{\cdot})}}{\sqrt{n}}.
\end{align*}
It suffices to set $\delta_{j+1} = \frac{1}{2}\delta_j$. Invoking the identity $\delta_{j+1} + \delta_j = 6 (\delta_{j+1} - \delta_{j+2})$, we derive
\begin{align*}
\EE_{x, \xi} \left[\sup_{f \in \cF} \frac{1}{n} \sum_{i=1}^n \xi_i f(x_i) \right] & \leq \delta_k + 12 \sum_{j=1}^{k-1} \frac{(\delta_{j+1} - \delta_{j+2}) \sqrt{\log \cN(\delta_{j+1}, \cF, \norm{\cdot}_\infty)}}{\sqrt{n}} \\
& \leq \delta_k + \frac{12}{\sqrt{n}} \int_{\delta_{k+1}}^{\delta_2} \sqrt{\log \cN(\epsilon, \cF, \norm{\cdot}_\infty)} d \epsilon \\
& \leq \inf_{\delta}~ 2\delta + \frac{12}{\sqrt{n}} \int_{\delta}^{\delta_1} \sqrt{\log \cN(\epsilon, \cF, \norm{\cdot}_\infty)} d \epsilon.
\end{align*}
By the assumption, we pick $\delta_1 = M$ and set the $\delta_1$-covering with only one element $f = 0$. This yields the desired result
\begin{align*}
\EE \left[d_{\cF}(\hat{\mu}_n, \mu)\right] \leq 2 \inf_{0 < \delta < M} \left(2\delta + \frac{12}{\sqrt{n}} \int_{\delta}^{M} \sqrt{\log \cN(\epsilon, \cF, \norm{\cdot}_\infty)} d\epsilon \right).
\end{align*}
\end{proof}

Now we need to find the covering number of H\"{o}lder class and that of the discriminator network. Classical result shows that the $\delta$-covering number of $\cH^\beta$ satisfies $\log \cN(\delta, \cH^\beta, \|\cdot\|_\infty) \leq C(1/\delta)^{\frac{d}{\beta} \vee 2}$ \cite{nickl2007bracketing}.

On the other hand, the following lemma quantifies the covering number of $\cF_{\textrm{NN}}$.
\begin{lemma}\label{lemma:NNcovering}
The $\delta$-covering number of $\cF_{\rm NN}(\bar{R}, \bar{\kappa}, \bar{L}, \bar{p}, \bar{J})$ satisfies the bound
\begin{align*}
\textstyle\cN\left(\delta, \cF_{\rm NN}(\bar{R}, \bar{\kappa}, \bar{L}, \bar{p}, \bar{K}), \norm{\cdot}_\infty\right) \leq \Big(\frac{2\bar{L}^2 (\bar{p}B + 2) (\bar{\kappa} \bar{p})^{\bar{L}+1}}{\delta} \Big)^{\bar{J}}.
\end{align*}
\end{lemma}

\begin{proof}[Proof of Lemma \ref{lemma:NNcovering}]
To construct a covering for $\cF_\textrm{NN}(\bar{R}, \bar{\kappa}, \bar{L}, \bar{p}, \bar{J})$, we discretize each weight parameter by a uniform grid with grid size $h$. To simplify the presentation, we omit the bar notation in this proof. Recall we write $f_\omega \in \cF_{\textrm{NN}}(R, \kappa, L, p, J)$ as $f_\omega = W_L \cdot \textrm{ReLU} (W_{L-1} \cdots \textrm{ReLU} (W_1 x + b_1) \dots + b_{L-1}) + b_L$. Let $f_\omega, f'_\omega \in \cF_{\rm NN}$ with all the weight parameters at most $h$ from each other. Denoting the weight matrices in $f_\omega, f'_\omega$ as $W_L, \dots, W_1, b_L, \dots, b_1$ and $W'_L, \dots, W'_1, b'_L, \dots, b'_1$, respectively, we bound the $L_\infty$ difference $\norm{f_\omega - f'_\omega}_\infty$ as
\begin{align*}
& \quad~ \norm{f_\omega - f'_\omega}_\infty \\
& = \big\|W_L \cdot \textrm{ReLU} (W_{L-1} \cdots \textrm{ReLU} (W_1 x + b_1) \cdots + b_{L-1}) + b_L \\
& \quad - (W'_L \cdot \textrm{ReLU} (W'_{L-1} \cdots \textrm{ReLU} (W'_1 x + b'_1) \cdots + b'_{L-1}) - b'_L)\big \|_\infty \\
& \leq \norm{b_L - b'_L}_\infty + \norm{W_L - W'_L}_1 \norm{W_{L-1} \cdots \textrm{ReLU} (W_1 x + b_1) \cdots + b_{L-1}}_\infty \\
&\quad + \norm{W_L}_1 \|W_{L-1} \cdots \textrm{ReLU} (W_1 x + b_1) \cdots + b_{L-1} \\
& \hspace{1.2in} - (W'_{L-1} \cdots \textrm{ReLU} (W'_1 x + b'_1) \cdots + b'_{L-1})\|_\infty \\
& \leq h + h p \norm{W_{L-1} \cdots \textrm{ReLU} (W_1 x + b_1) \cdots + b_{L-1}}_\infty \\
& \quad + \kappa p \|W_{L-1} \cdots \textrm{ReLU} (W_1 x + b_1) \cdots + b_{L-1} \\
& \hspace{1.2in} - (W'_{L-1} \cdots \textrm{ReLU} (W'_1 x + b'_1) \cdots + b'_{L-1})\|_\infty.
\end{align*}
We derive the following bound on $\norm{W_{L-1} \cdots \textrm{ReLU} (W_1 x + b_1) \dots + b_{L-1}}_\infty$:
\begin{align*}
& ~\quad \norm{W_{L-1} \cdots \textrm{ReLU} (W_1 x + b_1) \cdots + b_{L-1}}_\infty \\
& \leq \norm{W_{L-1} (\cdots \textrm{ReLU} (W_1 x + b_1) \cdots)}_\infty + \norm{b_{L-1}}_\infty \\
& \leq \norm{W_{L-1}}_1 \norm{W_{L-2} (\cdots \textrm{ReLU} (W_1 x + b_1) \cdots) + b_{L-2}}_\infty + \kappa \\
& \leq \kappa p \norm{W_{L-2} (\cdots \textrm{ReLU} (W_1 x + b_1) \cdots) + b_{L-2}}_\infty + \kappa \\
& \overset{(i)}{\leq} (\kappa p)^{L-1} B + \kappa \sum_{i=0}^{L-3}(\kappa p)^i \\
& \leq (\kappa p)^{L-1} B + \kappa (\kappa p)^{L-2},
\end{align*}
where $(i)$ is obtained by induction and $\norm{x}_\infty \leq B$. The last inequality holds, since $\kappa p > 1$. Substituting back into the bound for $\norm{f_\omega - f'_\omega}_\infty$, we have
\begin{align*}
\norm{f_\omega - f'_\omega}_\infty & \leq \kappa p \|W_{L-1} \cdots \textrm{ReLU} (W_1 x + b_1) \cdots + b_{L-1} \\
& \hspace{1.2in} - (W'_{L-1} \cdots \textrm{ReLU} (W'_1 x + b'_1) \cdots + b'_{L-1})\|_\infty \\
& \quad + h + h p \left[(\kappa p)^{L-1} B + \kappa (\kappa p)^{L-2} \right] \\
& \leq \kappa p \|W_{L-1} \cdots \textrm{ReLU} (W_1 x + b_1) \cdots + b_{L-1} \\
& \hspace{1.2in} - (W'_{L-1} \cdots \textrm{ReLU} (W'_1 x + b'_1) \cdots + b'_{L-1})\|_\infty \\
& \quad + h (p B + 2) (\kappa p)^{L-1} \\
& \overset{(i)}{\leq} (\kappa p)^{L-1} \norm{W_1 x + b_1 - W'_1 x - b'_1}_\infty + h (L-1) (pB + 2) (\kappa p)^{L-1} \\
& \leq h L (pB + 2) (\kappa p)^{L-1},
\end{align*}
where $(i)$ is obtained by induction. We choose $h$ satisfying $h L (pB + 2) (\kappa p)^{L-1} = \delta$. Then discretizing each parameter uniformly into $\kappa / h$ grids yields a $\delta$-covering on $\cF_\textrm{NN}$. There are totally $\binom{Lp^2}{J} \leq (Lp^2)^J$ choices of $K$ nonzero entries out of $Lp^2$ weight parameters. Therefore, the covering number is upper bounded by
\begin{align*}
\cN(\delta, \cF_\textrm{NN}(R, \kappa, L, p, J), \norm{\cdot}_\infty) & \leq (L p^2)^J \left(\frac{2\kappa}{h}\right)^{J} \nonumber \\
& \leq \left(\frac{2L^2 (pB + 2) (\kappa p)^{L+1}}{\delta} \right)^{J}.
\end{align*}
The proof is complete. 
\end{proof}

Combining Lemma \ref{lemma:stat} and the covering numbers, the statistical error can be bounded by
\begin{align*}
& \quad~ \EE \left[d_{\cH^\beta} (\hat{\mu}_n, \mu) + d_{\cF_{\textrm{NN}}}(\mu, \hat{\mu}_n)\right] \\
& \leq 4 \inf_{\delta_1 \in (0, C)} \left(\delta_1 + \frac{6}{\sqrt{n}} \int_{\delta_1}^{C} \sqrt{\log \cN(\epsilon, \cH^\beta, \norm{\cdot}_\infty)} d\epsilon \right) \\
& \quad + 4 \inf_{\delta_2 \in (0, C)} \left(\delta_2 + \frac{6}{\sqrt{n}} \int_{\delta_2}^{C} \sqrt{\log \cN(\epsilon, \cF_{\textrm{NN}}, \norm{\cdot}_\infty)} d\epsilon \right) \\
& \overset{(i)}{\leq} 4 \inf_{\delta_1 \in (0, C)} \left(\delta_1 + \frac{6}{\sqrt{n}} \int_{\delta_1}^{C} \sqrt{c\left(\frac{1}{\epsilon}\right)^{\left(\frac{d}{\beta} \vee 2\right)}} d\epsilon \right) \\
& \quad + 4 \inf_{\delta_2 \in (0, C)} \left(\delta_2 + \frac{6}{\sqrt{n}} \int_{\delta_2}^{C} \sqrt{\bar{J} \log \frac{\bar{L} (\bar{p}B + 2) (\bar{\kappa} \bar{p})^{\bar{L}}}{\epsilon}} d\epsilon \right).
\end{align*}
We find that the first infimum in step $(i)$ is attained at $\delta_1 = n^{-\frac{\beta}{d}}$. It suffices to take $\delta_2 = \frac{1}{n}$ in the second infimum. By omitting constants and polynomial dependence on $\beta$, we derive
\begin{align*}
\EE \left[d_{\cH^\beta} (\hat{\mu}_n, \mu) + d_{\cF_{\textrm{NN}}}(\mu, \hat{\mu}_n)\right] = \tilde{O}\left(\frac{1}{n} + n^{-\frac{\beta}{d}} + \frac{1}{\sqrt{n}} \sqrt{\bar{J} \bar{L} \log \left(n \bar{L} \bar{p}\right)}\right).
\end{align*}

\noindent $\bullet$ \textbf{Balancing approximation error and statistical error}.
Combining the previous three ingredients, by invoking the oracle inequality (Lemma \ref{lemma:oracle}), we can establish
\begin{align*}
\EE \left[d_{\cH^\beta} ((g^*_\theta)_\sharp \rho, \mu)\right] &= \tilde{O}\left(\epsilon_1 + \epsilon_2 + \frac{1}{n} + n^{-\frac{\beta}{d}} + \sqrt{\frac{\bar{J} \bar{L} \log \left(n \bar{L} \bar{p}\right)}{n}}\right) \\
&= \tilde{O}\left(\epsilon_1 + \epsilon_2 + \frac{1}{n} + n^{-\frac{\beta}{d}} + \sqrt{\frac{\epsilon_2^{-\frac{d}{\beta}} \log \frac{1}{\epsilon_2} \log \big(n \epsilon_2^{-\frac{d}{\beta}}\big)}{n}} \right).
\end{align*}
We choose $\epsilon_1 = n^{-\frac{\beta}{2\beta + d}}$, and $\epsilon_2$ satisfying $\epsilon_2 = n^{-\frac{1}{2}} \epsilon_2^{-\frac{d}{2\beta}}$, i.e., $\epsilon_2 = n^{-\frac{\beta}{2\beta + d}}$. This yields \eqref{thm2eq}.

\subsection{Proof of Corollary \ref{cor:finitesample}}\label{sec:prooffinitesample}
We need an extra concentration argument on the $m$ generated fake samples. This is tackled by an alternative oracle inequality \eqref{eq:alteroracle} shown in below. The rest of the proof utilizes the same argument in Theorem \ref{thm:main}.
\begin{proof}[Proof of Corollary \ref{cor:finitesample}]
We show an alternative oracle inequality for finite generated samples as follows. Inequality \eqref{eq:recoverbound} in the proof of Lemma \ref{lemma:oracle} yields
\begin{align*}
d_{\cH^\beta} ((g_\theta^{*, m})_\sharp \rho, \mu) & \leq d_{\cF_\textrm{NN}}((g^{*, m}_\theta)_\sharp \rho, \hat{\mu}_n) + 2 \sup_{f \in \cH^\beta} \inf_{f_{\omega} \in \cF_{\textrm{NN}}} \norm{f - f_{\omega}}_\infty \\
& \quad + d_{\cH^\beta}(\hat{\mu}_n, \mu).
\end{align*}
We further expand the first term on the right-hand side above as
\begin{align*}
d_{\cF_\textrm{NN}}((g^{*, m}_\theta)_\sharp \rho, \hat{\mu}_n) & \leq d_{\cF_\textrm{NN}}((g^{*, m}_\theta)_\sharp \rho, (g^{*, m}_\theta)_\sharp \hat{\rho}_m) + d_{\cF_\textrm{NN}}((g^{*, m}_\theta)_\sharp \hat{\rho}_m, \hat{\mu}_n).
\end{align*}
By the optimality of $g^{*, m}_\theta$, for any $g_\theta \in \cG_{\textrm{NN}}$, we have
\begin{align*}
& \quad~ d_{\cF_\textrm{NN}}((g^{*, m}_\theta)_\sharp \hat{\rho}_m, \hat{\mu}_n) \\
& \leq d_{\cF_\textrm{NN}}((g_\theta)_\sharp \hat{\rho}_m, \hat{\mu}_n) \\
& \leq d_{\cF_\textrm{NN}}((g_\theta)_\sharp \hat{\rho}_m, (g_\theta)_\sharp \rho) + d_{\cF_\textrm{NN}}((g_\theta)_\sharp \rho, \mu) + d_{\cF_\textrm{NN}}(\mu, \hat{\mu}_n) \\
& \leq d_{\cF_\textrm{NN}}((g_\theta)_\sharp \rho, \mu) + \sup_{g_\theta \in \cG_{\textrm{NN}}} d_{\cF_\textrm{NN}}((g_\theta)_\sharp \hat{\rho}_m, (g_\theta)_\sharp \rho) + d_{\cF_\textrm{NN}}(\mu, \hat{\mu}_n) \\
& \leq d_{\cH^\beta_\infty} ((g_\theta)_{\sharp} \rho, \mu) + 2 \sup_{f \in \cH^\beta} \inf_{f_\omega \in \cF_\textrm{NN}} \norm{f - f_{\omega}}_\infty + 2 d_{\cF_\textrm{NN}}(\mu, \hat{\mu}_n) \\
& \quad + \sup_{g_\theta \in \cG_{\textrm{NN}}} d_{\cF_\textrm{NN}}((g_\theta)_\sharp \hat{\rho}_m, (g_\theta)_\sharp \rho),
\end{align*}
where the last inequality follows the same argument in the proof of Lemma \ref{lemma:stat}. Combining all the inequalities together, we have
\begin{equation}\label{eq:alteroracle}
\begin{split}
&~\quad d_{\cH^\beta} ((g_\theta^{*, m})_\sharp \rho, \mu) \\
& \leq \inf_{g_\theta \in \cG_{\textrm{NN}}} d_{\cH^\beta_\infty} ((g_\theta)_{\sharp} \rho, \mu) + 4 \sup_{f \in \cH^\beta} \inf_{f_\omega \in \cF_\textrm{NN}} \norm{f - f_{\omega}}_\infty + 2 d_{\cF_\textrm{NN}}(\mu, \hat{\mu}_n) \\
& \quad + d_{\cH^\beta}(\hat{\mu}_n, \mu) + \sup_{g_\theta \in \cG_{\textrm{NN}}} d_{\cF_\textrm{NN}}((g_\theta)_\sharp \hat{\rho}_m, (g_\theta)_\sharp \rho) + d_{\cF_\textrm{NN}}((g^{*, m}_\theta)_\sharp \rho, (g^{*, m}_\theta)_\sharp \hat{\rho}_m).
\end{split}
\end{equation}
Given the proof of Theorem \ref{thm:main}, we only need to bound the extra statistical error terms $$\sup_{g_\theta \in \cG_{\textrm{NN}}} d_{\cF_\textrm{NN}}((g_\theta)_\sharp \hat{\rho}_m, (g_\theta)_\sharp \rho) \quad \textrm{and}\quad d_{\cF_\textrm{NN}}((g^{*, m}_\theta)_\sharp \rho, (g^{*, m}_\theta)_\sharp \hat{\rho}_m).$$ In fact, Lemma \ref{lemma:stat} and Lemma \ref{lemma:NNcovering} together imply
\begin{align*}
\sup_{g_\theta \in \cG_{\textrm{NN}}} d_{\cF_\textrm{NN}}((g_\theta)_\sharp \hat{\rho}_m, (g_\theta)_\sharp \rho) & = \tilde{O}\left(\frac{1}{\sqrt{m}} \sqrt{\bar{J} \bar{L} \log \left(m \bar{L} \bar{p}\right) + J L \log \left(m L p\right)}\right), \\
d_{\cF_\textrm{NN}}((g^{*, m}_\theta)_\sharp \rho, (g^{*, m}_\theta)_\sharp \hat{\rho}_m) & = \tilde{O}\left(\frac{1}{\sqrt{m}} \sqrt{\bar{J} \bar{L} \log \left(m \bar{L} \bar{p}\right)}\right),
\end{align*}
where the first inequality is obtained by taking $\cF = \cF_\textrm{NN} \circ \cG_{\textrm{NN}}$ in Lemma \ref{lemma:stat}, and its covering number is upper bounded by the product of the covering numbers of $\cF_{\textrm{NN}}$ and $\cG_{\textrm{NN}}$. Putting together, the estimation error $d_{\cH^\beta} ((g_\theta^{*, m})_\sharp \rho, \mu)$ can be bounded analogously to Theorem \ref{thm:main} as
\begin{align*}
& \quad~ \EE \left[d_{\cH^\beta} ((g_\theta^{*, m})_\sharp \rho, \mu)\right] \\
& = \tilde{O}\left(\epsilon_1 + \epsilon_2 + \frac{1}{n} + \frac{1}{m} + n^{-\frac{\beta}{d}} + \sqrt{\frac{\epsilon_2^{-\frac{d}{\beta}}}{n}} + \sqrt{\frac{\epsilon_1^{-\frac{d}{\alpha+1}} + \epsilon_2^{-\frac{d}{\beta}}}{m}} \right).
\end{align*}
It suffices to choose $\epsilon_2 = n^{-\frac{\beta}{2\beta + d}}$ and $\epsilon_1 = m^{-\frac{\alpha+1}{2(\alpha+1) + d}}$, which yields
\begin{align*}
\EE \left[d_{\cH^\beta} ((g_\theta^{*, m})_\sharp \rho, \mu)\right] = \tilde{O}\left(n^{-\frac{\beta}{2\beta + d}} + m^{-\frac{\alpha+1}{2(\alpha+1) + d}} + \sqrt{\frac{n^{\frac{d}{2\beta + d}}}{m}}\right).
\end{align*}
In the case of $m \geq n$, we have $\sqrt{\frac{n^{\frac{d}{2\beta + d}}}{m}} \leq n^{-\frac{\beta}{2\beta + d}}$. The proof is complete.
\end{proof}

\section{Proof of statistical theory in low-dimensional space}\label{sec:lowdproof}
The proof idea follows that of Theorem \ref{thm:main}, with extra attentions to the exploitation of low-dimensional structures in data. We first slightly modify the oracle inequality in Lemma \ref{lemma:oracle} to decompose the distribution estimation error.
\begin{lemma}\label{lemma:oracle_lowd}
Let $(U^*, g_\theta^*, V^*, f_\omega^*)$ be the global optimizer of \eqref{eq:empiricalgan}. The following error decomposition holds,
\begin{equation}\label{eq:oracle_lowd}
\begin{split}
& \quad~ \dw ((U^* \circ g^*_\theta)_\sharp \rho, \mu) \\
& \leq \underbrace{\inf_{g_\theta : A \circ g_\theta \in \cG_{\rm NN}^{\rm ld}} \norm{A \circ g_\theta - A \circ T^{\rm ld}}_\infty}_{\rm generator~approximation~error} + \underbrace{\dw(\hat{\mu}_n, \mu) + d_{\cF_{\rm NN}^{\rm ld}}(\hat{\mu}_n, \mu)}_{\rm statistical~error} \\
& \quad + \underbrace{\sup_{f \in {\rm Lip}_1(\RR^d)} \inf_{f_\omega \circ V^\top \in \cF_{\rm NN}^{\rm ld}} \norm{f \circ U^* - f_\omega \circ V^\top U^*}_\infty + \norm{f \circ A - f_\omega \circ V^\top A}_\infty}_{\rm discriminator~approximation~error~(HARD)} \\
& \quad + \underbrace{2 \sup_{f \in {\rm Lip}_1(\RR^d)} \inf_{f_\omega \circ V^\top \in \cF_{\rm NN}^{\rm ld}} \norm{f \circ A - f_\omega \circ V^\top A}_\infty}_{\rm discriminator~approximation~error~(EASY)}.
\end{split}
\end{equation}
\end{lemma}
\begin{proof}[Proof of Lemma \ref{lemma:oracle_lowd}]
We replicate the error decomposition in \eqref{eq:recoverbound} by taking $\beta = 1$,
\begin{equation}\label{eq:lowd_error_decomp}
\begin{split}
\dw ((U^* \circ g^*_\theta)_\sharp \rho, \mu) & \leq \dw ((U^* \circ g^*_\theta)_\sharp \rho, \hat{\mu}_n) + \dw (\hat{\mu}_n, \mu) \\
& = d_{\cF_\textrm{NN}^{\rm ld}}((U^* \circ g^*_\theta)_\sharp \rho, \hat{\mu}_n) \\
& \quad + \dw ((U^* \circ g^*_\theta)_\sharp \rho, \hat{\mu}_n) - d_{\cF_\textrm{NN}^{\rm ld}}((U^* \circ g^*_\theta)_\sharp \rho, \hat{\mu}_n) \\
& \quad + \dw (\hat{\mu}_n, \mu).
\end{split}
\end{equation}
Using the optimality of $(U^*, g_\theta^*)$, we further bound $d_{\cF_{\rm NN}}((U^* \circ g_\theta^*)_\sharp \rho, \hat{\mu}_n)$ in the last display as
\begin{align}\label{eq:lowd_error_generator}
& \quad~ d_{\cF_{\rm NN}}((U^* \circ g_\theta^*)_\sharp \rho, \hat{\mu}_n) \nonumber \\
& \leq d_{\cF_{\rm NN}^{\rm ld}} ((U^* \circ g_\theta^*)_\sharp \rho, \mu) + d_{\cF_{\rm NN}^{\rm ld}} (\mu, \hat{\mu}_n) \nonumber \\
& = \inf_{U \circ g_\theta \in \cG_{\rm NN}^{\rm ld}} d_{\cF_{\rm NN}^{\rm ld}} ((U \circ g_\theta)_\sharp \rho, \mu) + d_{\cF_{\rm NN}^{\rm ld}} (\mu, \hat{\mu}_n) \nonumber \\
& \overset{(i)}{=} \inf_{U \circ g_\theta \in \cG_{\rm NN}^{\rm ld}} d_{\cF_{\rm NN}^{\rm ld}} ((U \circ g_\theta)_\sharp \rho, \mu) - d_{\cH^1_\infty} ((U \circ g_\theta)_\sharp \rho, \mu) \nonumber \\
& \hspace{1in} + d_{\cH^1_\infty} ((U \circ g_\theta)_\sharp \rho, \mu) + d_{\cF_{\rm NN}^{\rm ld}} (\mu, \hat{\mu}_n),
\end{align}
where in $(i)$, discriminative class $\cH^1_\infty$ follows the same definition in Lemma \ref{lemma:oracle} with $\beta = 1$.

By Assumption \ref{assumption:lowddensity} and the optimal transport theory in Lemma \ref{lemma:monge}, we rewrite the data distribution $\mu$ as a pushforward distribution $\mu = (A \circ T^{\rm ld})_\sharp \rho$, where $T^{\rm ld} : \RR^q \mapsto \RR^q$ is an $(\alpha+1)$-H\"{o}lder continuous transport plan. Accordingly, we rewrite the empirical data distribution $\hat{\mu}_n$ as $\hat{\mu}_n = (A \circ T^{\rm ld})_\sharp \hat{\rho}_n$, with $\hat{\rho}_n$ an empirical version of $\rho$. Applying Lemma \ref{lemma:approx} and using the same argument in Theorem \ref{thm:generator} for approximating $A \circ T^{\rm ld}$, we obtain $A \circ \tilde{g}_\theta \in \cG_{\rm NN}^{\rm ld}$ as a proper approximation. Note that we have chosen $U = A$ in representing $A \circ T^{\rm ld}$ for simplicity. Substituting these notations into \eqref{eq:lowd_error_generator} gives rise to
\begin{align}\label{eq:lowd_error_firstterm}
& \quad~ d_{\cF_{\rm NN}}((U^* \circ g_\theta^*)_\sharp \rho, \hat{\mu}_n) \nonumber \\
& \overset{(i)}{\leq} d_{\cF_{\rm NN}^{\rm ld}} \left((A \circ \tilde{g}_\theta)_\sharp \rho, (A \circ T^{\rm ld})_\sharp \rho\right) - d_{\cH^1_\infty} \left((A \circ \tilde{g}_\theta)_\sharp \rho, (A \circ T^{\rm ld})_\sharp \rho\right) \nonumber \\
& \quad~ + d_{\cH^1_\infty} \left((A \circ \tilde{g}_\theta)_\sharp \rho, (A \circ T^{\rm ld})_\sharp \rho\right) + d_{\cF_{\rm NN}^{\rm ld}} (\mu, \hat{\mu}_n) \nonumber \\
& \overset{(ii)}{\leq} d_{\cF_{\rm NN}^{\rm ld}} \left((A \circ \tilde{g}_\theta)_\sharp \rho, (A \circ T^{\rm ld})_\sharp \rho\right) - d_{\cH^1_\infty} \left((A \circ \tilde{g}_\theta)_\sharp \rho, (A \circ T^{\rm ld})_\sharp \rho\right) \nonumber \\
& \quad~ + \norm{A \circ \tilde{g}_\theta - A \circ T^{\rm ld}}_\infty + d_{\cF_{\rm NN}^{\rm ld}} (\mu, \hat{\mu}_n),
\end{align}
where inequality $(i)$ holds by instantiating the infimum in \eqref{eq:lowd_error_generator} to $A \circ \tilde{g}_\theta$, and inequality $(ii)$ follows by the definition of IPM over $\cH^1_\infty$ class. We substitute \eqref{eq:lowd_error_firstterm} into \eqref{eq:lowd_error_decomp}, which leads to
\begin{align}\label{eq:lowd_3error}
& \quad~ \dw ((U^* \circ g^*_\theta)_\sharp \rho, \mu) \nonumber \\
& \leq \underbrace{\norm{A \circ \tilde{g}_\theta - A \circ T^{\rm ld}}_\infty}_{\rm generator~approximation~error} + \underbrace{\dw(\hat{\mu}_n, \mu) + d_{\cF_{\rm NN}^{\rm ld}}(\hat{\mu}_n, \mu)}_{\rm statistical~error} \nonumber \\
& \quad + \underbrace{\dw \left((U^* \circ g^*_\theta)_\sharp \rho, (A \circ T^{\rm ld})_\sharp \hat{\rho}_n)\right) - d_{\cF_\textrm{NN}^{\rm ld}}\left((U^* \circ g^*_\theta)_\sharp \rho, (A \circ T^{\rm ld})_\sharp \hat{\rho}_n\right)}_{\rm discriminator~approximation~error~(HARD)} \nonumber \\
& \quad + \underbrace{d_{\cF_{\rm NN}^{\rm ld}} \left((A \circ \tilde{g}_\theta)_\sharp \rho, (A \circ T^{\rm ld})_\sharp \rho\right) - d_{\cH^1_\infty} \left((A \circ \tilde{g}_\theta)_\sharp \rho, (A \circ T^{\rm ld})_\sharp \rho\right)}_{\rm discriminator~approximation~error~(EASY)}.
\end{align}
Two disciminator approximation error terms share a similar formulation, and we can further provide a simplified upper bound on them. Denote $\norm{f}_{\rm Lip}$ as the lipschitz constant of function $f$, and consider the (HARD) term for example.
\begin{align}\label{eq:lowd_discriminator_approx_error_hard}
& \quad~ \dw \left((U^* \circ g^*_\theta)_\sharp \rho, (A \circ T^{\rm ld})_\sharp \hat{\rho}_n)\right) - d_{\cF_\textrm{NN}^{\rm ld}}\left((U^* \circ g^*_\theta)_\sharp \rho, (A \circ T^{\rm ld})_\sharp \hat{\rho}_n\right) \nonumber \\
& = \sup_{\norm{f}_{\rm Lip} \leq 1} \EE_{z \sim \rho} \left[f \circ U^* \circ g_\theta^*(z)\right] - \EE_{z \sim \hat{\rho}_n} \left[f \circ A \circ T^{\rm ld}(z)\right] \nonumber \\
& \quad - \sup_{f_\omega \circ V^\top \in \cF_{\rm NN}^{\rm ld}} \EE_{z \sim \rho} \left[f_\omega \circ V^\top \circ U^* \circ g_\theta^*(z)\right] - \EE_{z \sim \hat{\rho}_n} \left[f_\omega \circ V^\top \circ A \circ T^{\rm ld}(z)\right] \nonumber \\
& \leq \sup_{\norm{f}_{\rm Lip} \leq 1} \inf_{f_\omega \circ V^\top \in \cF_{\rm NN}^{\rm ld}} \Big\{ \big\vert \EE_{z \sim \rho} \left[(f \circ U^* - f_\omega \circ V^\top \circ U^*) \circ g^*_\theta(z) \right] \big\vert \nonumber \\
& \hspace{1.4in} + \big\vert \EE_{z \sim \hat{\rho}_n} \left[(f \circ A - f_\omega \circ V^\top \circ A) \circ T^{\rm ld}(z)\right] \big\vert \Big\} \nonumber \\
& \leq \sup_{\norm{f}_{\rm Lip} \leq 1} \inf_{f_\omega \circ V^\top \in \cF_{\rm NN}^{\rm ld}} \norm{f \circ U^* - f_\omega \circ V^\top U^*}_\infty + \norm{f \circ A - f_\omega \circ V^\top A}_\infty.
\end{align}
Applying the same argement to the (EASY) error term yields
\begin{align}\label{eq:lowd_discriminator_approx_error_easy}
& \quad~ d_{\cF_{\rm NN}^{\rm ld}} \left((A \circ \tilde{g}_\theta)_\sharp \rho, (A \circ T^{\rm ld})_\sharp \rho\right) - d_{\cH^1_\infty} \left((A \circ \tilde{g}_\theta)_\sharp \rho, (A \circ T^{\rm ld})_\sharp \rho\right) \nonumber \\
& \leq 2 \sup_{\norm{f}_{\rm Lip} \leq 1} \inf_{f_\omega \circ V^\top \in \cF_{\rm NN}^{\rm ld}} \norm{f \circ A - f_\omega \circ V^\top A}_\infty.
\end{align}
Note that we already used the fact that $\cH^1_\infty$ is a subset of $\cH^1$. Plugging \eqref{eq:lowd_discriminator_approx_error_hard} and \eqref{eq:lowd_discriminator_approx_error_easy} into \eqref{eq:lowd_3error} and taking infimum over $\tilde{g}_\theta$, we obtain the desired oracle inequality,
\begin{align*}
& \quad~ \dw ((U^* \circ g^*_\theta)_\sharp \rho, \mu) \\
& \leq \underbrace{\inf_{g : A \circ g \in \cG_{\rm NN}^{\rm ld}} \norm{A \circ g - A \circ T^{\rm ld}}_\infty}_{\rm generator~approximation~error} + \underbrace{\dw(\hat{\mu}_n, \mu) + d_{\cF_{\rm NN}^{\rm ld}}(\hat{\mu}_n, \mu)}_{\rm statistical~error} \\
& \quad + \underbrace{\sup_{\norm{f}_{\rm Lip} \leq 1} \inf_{f_\omega \circ V^\top \in \cF_{\rm NN}^{\rm ld}} \norm{f \circ U^* - f_\omega \circ V^\top U^*}_\infty + \norm{f \circ A - f_\omega \circ V^\top A}_\infty}_{\rm discriminator~approximation~error~(HARD)} \\
& \quad + \underbrace{2 \sup_{\norm{f}_{\rm Lip} \leq 1} \inf_{f_\omega \circ V^\top \in \cF_{\rm NN}^{\rm ld}} \norm{f \circ A - f_\omega \circ V^\top A}_\infty}_{\rm discriminator~approximation~error~(EASY)}.
\end{align*}
The proof is complete. 
\end{proof}
In the sequel, we bound error terms in \eqref{eq:oracle_lowd} respectively. The generator approximation error can be reduced to approximating $T^{\rm ld}$. By some manipulation on the intrinsic structures of data distribution, we expect that the statistical error scales with the subspace dimension $q$. The main difficulty stems from bounding the discriminator approximation error. A quick comparison to Lemma \ref{lemma:oracle} indicates that the (EASY) error term may be bounded similarly as in Theorem \ref{thm:main}. In contrast, the (HARD) error term involves simultaneously approximating the discriminative function projected into the column space of $U^*$ and $A$. In general, such an approximation error is hardly small unless $U^*, A$ share approximately the same column space. Fortunately, this is indeed the case as shown in Lemma \ref{lemma:subspace_match} so that the (HARD) error term can be controlled.

\noindent $\bullet$ {\bf Bounding generator approximation error}.
Suppose that we require the generator approximation error to be bounded by $\epsilon_1 > 0$, i.e., $$\inf_{g : U \circ g_\omega \in \cG_{\rm NN}^{\rm ld}} \norm{U \circ g_\omega - A \circ T^{\rm ld}}_\infty \leq \epsilon_1.$$
It suffices to choose a proper generator architecture $\cG_{\rm NN}^{\rm ld}$ such that there exists $\tilde{g}_\omega$ satisfying $\norm{\tilde{g}_\omega - T^{\rm ld}}_\infty \leq \epsilon_1 / q$. To see this, we take $U = A$ and substitute $\tilde{g}_\omega$ into the generator approximation error,
\begin{align*}
\norm{A \circ \tilde{g}_\omega - A \circ T^{\rm ld}}_\infty & = \norm{\sum_{j=1}^q A_{:, j} [\tilde{g}_\omega - T^{\rm ld}]_j}_\infty \\
& \leq \sum_{j=1}^q \norm{A_{:, j}}_\infty \norm{\tilde{g}_\omega - T^{\rm ld}}_\infty \\
& \leq \epsilon_1,
\end{align*}
where the last inequality holds since $A$ has orthonormal columns. We can apply Lemma \ref{lemma:approx} and \ref{lemma:generator} for choosing proper network configuration of $\cG_{\rm NN}^{\rm ld}$ to ensure the existence of $\tilde{g}_\omega$. We recall that $T^{\rm ld}$ is a $\cH^{\alpha+1}$ continuous mapping in $\RR^q$ by Lemma \ref{lemma:monge}. Therefore, the resulting network architecture has the following configuration
\begin{equation}\label{eq:generator_ld_config}
\begin{split}
& \qquad R = B, \quad \kappa = O(1), \quad L = O\left(\log \frac{1}{\epsilon_1}\right), \\
& p = O\left(q \epsilon_1^{-\frac{q}{\alpha+1}}\right), \quad K = O\left(dq + \epsilon_1^{-\frac{q}{\alpha+1}} \log \frac{1}{\epsilon_1} \right).
\end{split}
\end{equation}
We will choose $\epsilon_1$ later in the last step of the proof to balance all the error terms.

\noindent $\bullet$ {\bf Bounding discriminator approximation error}. We first consider the (EASY) error term. Suppose that we require the (EASY) discriminator approximation error to be bounded by $\epsilon_2 > 0$. We check that once $f : \RR^d \mapsto \RR$ is $1$-Lipschitz and $A$ has orthonormal columns, then $f \circ A : \RR^q \mapsto \RR$ is also $1$-Lipschitz. To see this, for any $x, y \in \RR^d$, we have
\begin{align*}
\left\lvert f(A x) - f(A y) \right\rvert \leq \norm{Ax - A y}_2 \leq \norm{A}_2 \norm{x - y}_2 = \norm{x - y}_2.
\end{align*}
By taking $V = A$ in the (EASY) term, it suffices to ensure that $f_\omega$ can approximate any $1$-Lipschitz function in a compact subset of $[0, 1]^q$. Due to the additional $\bar{\gamma}$-Lipschitz continuity constraint in \eqref{eq:discriminator_lowd}, we need a stronger universal approximation theory of the discriminator. The following lemma shows that ReLU neural networks can accurately approximating $1$-Lipschitz functions in $L_\infty$-norm, while the Lipschitz continuity of the network remains independent of the approximation error.
\begin{lemma}\label{lemma:approx_lip}
For any $\epsilon_2 \in (0, 1)$, there exists a ReLU network architecture ${\rm NN}(\bar{R}, \bar{\kappa}, \bar{L}, \bar{p}, \bar{J})$, such that for any target $1$-Lipschitz function $f$ defined on $[0, 1]^q$ with $f(0) = 0$, the architecture yields an approximation $\hat{f}$ satisfying $\lVert f - \hat{f} \rVert_\infty \leq \epsilon_2$. Moreover, the Lipschitz continuity of $\hat{f}$ is bounded by
\begin{align*}
\left\lvert \hat{f}(x) - \hat{f}(y) \right\rvert \leq 10q \norm{x - y}_\infty \quad \text{for any} \quad x, y \in [0, 1]^q.
\end{align*}
The configuration of network architecture is
\begin{align*}
& \bar{R} = \sqrt{q}, \quad \bar{\kappa} = O(1), \quad \bar{L} = O\left(\log 1/\epsilon_2 + q \right), \\
& \hspace{0.12in} \bar{p} = O\left(\epsilon_2^{-q} \right), \quad \bar{J} = O\left(\epsilon_2^{-q} (\log 1/\epsilon_2 + q)\right).
\end{align*}
\end{lemma}
The proof is defered to Appendix \ref{pf:approx_lip}. Lemma \ref{lemma:approx_lip} improves the approximation guarantee in Lemma \ref{lemma:approx} with the additional Lipschitz continuity characterization, while the newtork size shares the same order of magnitude when specializing Lemma \ref{lemma:approx} to $d = q$ and $\beta = 1$. We take $\cF_{\rm NN}^{\rm ld}(\bar{R}, \bar{\kappa}, \bar{L}, \bar{p}, \bar{J}, \bar{\gamma})$ with $\bar{\gamma} = 10q$ and all the other parameters the same as in Lemma \ref{lemma:approx_lip}. Since the (EASY) error term is invariant with respect to translations on $f$, we can always assume $f(0) = 0$ without loss of generality. It then holds
\begin{align*}
{\rm (EASY)~Error~Term} \leq 2\epsilon_2.
\end{align*}

We next bound the (HARD) term. Recall that we need the column spaces of $U^*$ and $A$ to be approximately identical for controlling this error. Thanks to the choice of both the generator and discriminator class, we can show that the column spans of $U^*$ and $A$ match up to some error.
\begin{lemma}\label{lemma:subspace_match}
Given $\epsilon_1, \epsilon_2 \in (0, 1)$. Suppose Assumption \ref{assumption:lowddensity} and \ref{assumption:lowdmanifold} hold. Let the generator $\cG_{\rm NN}^{\rm ld}$ be chosen as \eqref{eq:generator_ld_config} and discriminator $\cF_{\rm NN}^{\rm ld}$ be chosen as in Lemma \ref{lemma:approx_lip} with $\bar{\gamma} = 10q$. For the global optimizer $(U^*, g_\theta^*)$, it holds
\begin{align}
& \norm{U^* - A}_{\rm F}^2 \nonumber \\
& \qquad \leq 4q \left(1 + 4\sqrt{q} \left(\min_i \EE_{z \sim \rho}\left[T^{\rm ld}_i(z)\right]\right)^{-1} \EE_{z \sim \rho} \left[\norm{g_\theta^*(z)}_2\right]\right)^2 \nonumber \\
& \hspace{0.6in} \cdot \left(\min_i \EE_{z \sim \rho} \left[T^{\rm ld}_i(z)\right]\right)^{-2} \epsilon^2, \nonumber
\end{align}
where $\epsilon = 10q \epsilon_1 + 3\epsilon_2$.
\end{lemma}
The full proof is deferred to Appendix \ref{pf:subspace_match}. We remark that $\EE_{z \sim \rho}[T_i^{\rm ld}(z)]$ is always lower bounded by a positive constant $\tau$ for any $i = 1, \dots, q$, since its density is positive on the support by Assumption \ref{assumption:lowddensity}. To establish Lemma \ref{lemma:subspace_match}, we leverage the optimality of $U^*,g_\theta^*$ and the corresponding discriminator network. We show by contraction that if the column spaces of $U^*$ and $A$ do not match closely, there exists a discriminator network capable of distinguishing the generated distribution and data distribution. 

Given Lemma \ref{lemma:subspace_match}, we are ready to derive an upper bound for the (HARD) discriminator approximation error term.
\begin{align}\label{eq:discriminator_hard}
& \quad~ \sup_{f \in {\rm Lip}_1(\RR^d)} \inf_{f_\omega \circ V^\top \in \cF_{\rm NN}^{\rm ld}} \norm{f \circ U^* - f_\omega \circ V^\top U^*}_\infty + \norm{f \circ A - f_\omega \circ V^\top A}_\infty \nonumber \\
& \overset{(i)}{\leq} \sup_{f \in {\rm Lip}_1(\RR^d)} \inf_{f_\omega \circ A^\top \in \cF_{\rm NN}^{\rm ld}} \norm{f \circ U^* - f_\omega \circ A^\top U^*}_\infty + \norm{f \circ A - f_\omega}_\infty \nonumber \\
& \overset{(ii)}{\leq} \sup_{f \in {\rm Lip}_1(\RR^d)} \inf_{f_\omega \circ A^\top \in \cF_{\rm NN}^{\rm ld}} \norm{f \circ U^* - f \circ A}_\infty + \norm{f_\omega - f_\omega \circ A^\top U^*)}_\infty \nonumber \\
& \hspace{1.5in} + 2 \norm{f \circ A - f_\omega}_\infty,
\end{align}
where $(i)$ is obtained by taking $V = A$, and inequality $(ii)$ is obtained by the triangle inequality
\begin{align*}
\norm{f \circ U^* - f_\omega \circ A^\top U^*}_\infty & \leq \norm{f \circ U^* - f \circ A}_\infty + \norm{f \circ A - f_\omega}_\infty \\
& \quad + \norm{f_\omega - f_\omega \circ A^\top U^*}_\infty.
\end{align*}
The first term on the right-hand side of \eqref{eq:discriminator_hard} can be bounded using the Lipschitz continuity of $f$, i.e.,
\begin{align*}
& \quad \norm{f \circ U^* - f \circ A}_\infty \\
& \leq \sup_{x \in [0, 1]^q} \norm{U^* - A^*}_{2} \norm{x}_2 \\
& \leq 2 q \big(1 + 4\sqrt{q} \left(\min_i \EE_{z \sim \rho}\left[T^{\rm ld}_i(z)\right]\right)^{-1} \EE_{z \sim \rho} \left[\norm{g_\theta^*(z)}_2\right]\big)\max_i \EE_{z \sim \rho}^{-1} \left[T^{\rm ld}_i(z)\right] \epsilon.
\end{align*}
A similar argument applies to
\begin{align*}
& \quad \norm{f_\omega - f_\omega \circ A^\top U^*}_\infty \\
& \leq \sup_{x \in [0, 1]^q} 10q \norm{I - A^\top U^*}_{2} \norm{x}_2 \\
& \leq \sup_{x \in [0, 1]^q} 10q^{3/2} \norm{A - U^*}_2 \\
& \leq 20 q^2 \big(1 + 4\sqrt{q} \left(\min_i \EE_{z \sim \rho}\left[T^{\rm ld}_i(z)\right]\right)^{-1} \EE_{z \sim \rho} \left[\norm{g_\theta^*(z)}_2\right]\big)\max_i \EE_{z \sim \rho}^{-1} \left[T^{\rm ld}_i(z)\right] \epsilon.
\end{align*}
The last term in the right-hand side of \eqref{eq:discriminator_hard} is the discriminator approximation error, which is bounded by $\epsilon_2$. As a result, the (HARD) error term is upper bounded by
\begin{align*}
& {\rm (HARD)~Error~Term}  \\
& ~ \leq 2 q \big(1 + 4\sqrt{q} \left(\min_i \EE_{z \sim \rho}\left[T^{\rm ld}_i(z)\right]\right)^{-1} \EE_{z \sim \rho} \left[\norm{g_\theta^*(z)}_2\right]\big)\max_i \EE_{z \sim \rho}^{-1} \left[T^{\rm ld}_i(z)\right] \epsilon \\
& ~ + 20 q^2 \big(1 + 4\sqrt{q} \left(\min_i \EE_{z \sim \rho}\left[T^{\rm ld}_i(z)\right]\right)^{-1} \EE_{z \sim \rho} \left[\norm{g_\theta^*(z)}_2\right]\big)\max_i \EE_{z \sim \rho}^{-1} \left[T^{\rm ld}_i(z)\right] \epsilon \\
& ~ + 2 \epsilon_2 \\
& = O(\epsilon_2 + q^3 \epsilon),
\end{align*}
where the last step is obtained by $\norm{g_\theta^*(z)}_2 \leq \sqrt{q}$ due to $g_\theta^*(z) \in [0, 1]^q$.

\noindent $\bullet$ {\bf Bounding statistical error}.
Similar to the statistical error in Lemma \ref{lemma:oracle}, we can bound it via finite-sample concentration. Yet we can pursue a faster convergence rate here by rewriting the data distribution as a pushforward of a low-dimensional distribution.
\begin{lemma}\label{lemma:staterror_lowd}
Suppose Assumption \ref{assumption:lowdmanifold} and \ref{assumption:lowddensity} hold. Statistical error terms in Lemma \ref{lemma:oracle_lowd} are bounded by
\begin{align}
\dw(\hat{\mu}_n, \mu) & = O\left(n^{-1/q} \log n\right), \nonumber \\
d_{\cF_{\rm NN}^{\rm ld}}(\hat{\mu}_n, \mu) & = O\left(\frac{1}{n} + \frac{1}{\sqrt{n}} \sqrt{\bar{J}\bar{L} \log (\bar{L} \bar{p} \bar{\kappa} n)} \right). \nonumber
\end{align}
\end{lemma}
From Lemma \ref{lemma:staterror_lowd}, we observe that the statistical error $\dw(\hat{\mu}_n, \mu)$ only depends on dimension $q$. To make sense the result, we rewrite the data distribution $\mu = A_\sharp (A^\top_\sharp \mu)$. In this way, we can translate the concentration of $\hat{\mu}_n$ to $\mu$ in $\RR^D$ into a counterpart in $\RR^q$. Recall that $A^\top_\sharp \mu$ is a distribution with a $\cH^\alpha(\RR^q)$ density by Assumption \ref{assumption:lowddensity}. Threfore, we can apply Lemma \ref{lemma:stat} to complete the proof. See detailed arguments in Appendix \ref{pf:staterror_lowd}.

\noindent $\bullet$ {\bf Balancing approximation error and statistical error}.
We collect all the error terms in the oracle inequality of Lemma \ref{lemma:oracle_lowd} and choose optimal scalings on $\epsilon_1$ and $\epsilon_2$. We list all the error upper bounds in the following for a quick reference.
\begin{enumerate}
\item Generator approximation error $O(\epsilon_1)$.
\item Statistical error $O\left(n^{-1/q}\log n + \frac{1}{\sqrt{n}} \sqrt{\bar{J}\bar{L} \log (\bar{L} \bar{p} \bar{\kappa} n)}\right)$.
\item (EASY) discriminator approximation error $O(\epsilon_2)$.
\item (HARD) discriminator approximation error $O(\epsilon_2 + q^3 \epsilon)$.
\end{enumerate}
Summing up four error bounds above yields
\begin{align*}
\dw((U^* \circ g_\theta^*)_\sharp \rho, \mu) = O \left(\epsilon_1 + \epsilon_2 + n^{-1/q}\log n + q^3 \epsilon + \frac{1}{\sqrt{n}} \sqrt{\bar{J}\bar{L} \log (\bar{L} \bar{p} \bar{\kappa} n)}\right).
\end{align*}
Substituting the configuration of $\cF_{\rm NN}^{\rm ld}$ in Lemma \ref{lemma:approx_lip} into the last display, we set $\epsilon_1 = \epsilon_2 = n^{-\frac{1}{2+q}}$. By collecting terms, we derive
\begin{align*}
\dw\left((U^* \circ g_\theta^*)_\sharp \rho, \mu\right) = \tilde{O} \left(n^{-\frac{1}{2+q}} \log^2 n \right).
\end{align*}
The corresponding configurations of generator $\cG_{\rm NN}^{\rm ld}$ and discriminator $\cF_{\rm NN}^{\rm ld}$ is obtained by substituting $\epsilon_1$ and $\epsilon_2$ in \eqref{eq:generator_ld_config} and Lemma \ref{lemma:approx_lip}, respectively.

\section{Conclusion and discussion}\label{sec:discuss}
We establish statistical convergence of distribution estimation using GANs. Specifically, with proper generator and discriminator network architecture, we show GANs are consistent estimator of data distribution in terms of the Wasserstein distance. Moreover, when data have intrinsic low-dimensional linear structures, we show GANs can capture the unknown linear structure and enjoy a faster statistical rate of estimation, which is free of the curse of dimensionality. Compared to existing works, our theory exploits the pushforward structure of GANs and network architectures are explicitly given without invertibility constraints.

In the sequel, we discuss several related topics and future directions.

\noindent {\bf Distribution estimation on manifold}. Low-dimensional manifolds are sensible tools to model data geometric structures. A manifold can be characterized by local neighborhoods (charts), which generalizes the mixture model in Section \ref{sec:mixture}. Our analysis does not cover manifold data due to the difficulty to accurately approximate target distributions using generator networks. Specifically, different from the mixture model with a pre-fixed partition of components, we often don't have specific information on how to properly construct charts on the manifold. It is rare the case that an artificially constructed collection of charts can preserve the regularity of target data distributions. This makes choosing a proper generator network difficult.

\noindent {\bf Convolutional filters and residual connections}. Convolutional filters \cite{krizhevsky2012imagenet} are widely used in GANs for image generating and processing. Empirical results show that convolutional filters can learn hidden representations aligned with various patterns in images \cite{zeiler2014visualizing, zhou2018interpreting}, e.g., textures and skeletons. An interesting question is to understand how convolutional filters capture the aforementioned low-dimensional structures in data sets.

\noindent {\bf Smoothness of data distributions and regularized distribution Estimation}. Theorem \ref{thm:main} indicates a convergence rate independent of the smoothness of the data distribution. The reason behind is that the empirical data distribution $\hat{\mu}_n$ cannot inherit the same smoothness as the underlying data distribution. This limitation exists in all previous works \cite{liang2017well, singh2018nonparametric,uppal2019nonparametric}.
It is interesting to investigate whether GANs can achieve a faster convergence rate (e.g., attain the minimax optimal rate).

From a theoretical perspective, \cite{liang2018well} suggested first obtaining a smooth kernel estimator from $\tilde{\mu}_n$, and then replacing $\hat{\mu}_n$ by $\tilde{\mu}_n$ to train GANs. In practice, kernel smoothing is hardly used in GANs. Instead, regularization (e.g., entropy regularization) and normalization (e.g., spectral normalization and batch-normalization) are widely applied as implicit regularizers to promote the smoothness of the learned distribution. Several empirical studies of GANs suggest that divergence-based and mutual information-based regularization can stabilize the training and improve the performance \cite{che2016mode, cao2018improving} of GANs. We leave the studies on statistical properties of regularized GANs for future investigation.

\noindent {\bf Computational concerns}. Our statistical guarantees hold for the global optimizer of \eqref{eq:empiricalgan}, whereas solving \eqref{eq:empiricalgan} is often difficult. In practice, it is observed that larger neural networks are easier to train and yield better statistical performance \cite{zhang2016understanding, jacot2018neural, du2018gradient_b, allen2018convergence, du2018gradient_a, li2018learning, arora2019fine, allen2019learning, du2019width}. This is referred to as overparameterization. Establishing a connection between computation and statistical properties of GANs is an important direction.

\bibliographystyle{ieeetr}
\bibliography{sn-bibliography}

\begin{thebibliography}{10}

\bibitem{goodfellow2014generative}
I.~Goodfellow, J.~Pouget-Abadie, M.~Mirza, B.~Xu, D.~Warde-Farley, S.~Ozair,
  A.~Courville, and Y.~Bengio, ``Generative adversarial nets,'' in {\em
  Advances in neural information processing systems}, pp.~2672--2680, 2014.

\bibitem{reed2016generative}
S.~Reed, Z.~Akata, X.~Yan, L.~Logeswaran, B.~Schiele, and H.~Lee, ``Generative
  adversarial text to image synthesis,'' {\em arXiv preprint arXiv:1605.05396},
  2016.

\bibitem{ledig2017photo}
C.~Ledig, L.~Theis, F.~Husz{\'a}r, J.~Caballero, A.~Cunningham, A.~Acosta,
  A.~Aitken, A.~Tejani, J.~Totz, Z.~Wang, {\em et~al.}, ``Photo-realistic
  single image super-resolution using a generative adversarial network,'' in
  {\em Proceedings of the IEEE conference on computer vision and pattern
  recognition}, pp.~4681--4690, 2017.

\bibitem{schawinski2017generative}
K.~Schawinski, C.~Zhang, H.~Zhang, L.~Fowler, and G.~K. Santhanam, ``Generative
  adversarial networks recover features in astrophysical images of galaxies
  beyond the deconvolution limit,'' {\em Monthly Notices of the Royal
  Astronomical Society: Letters}, vol.~467, no.~1, pp.~L110--L114, 2017.

\bibitem{brock2018large}
A.~Brock, J.~Donahue, and K.~Simonyan, ``Large scale gan training for high
  fidelity natural image synthesis,'' {\em arXiv preprint arXiv:1809.11096},
  2018.

\bibitem{volz2018evolving}
V.~Volz, J.~Schrum, J.~Liu, S.~M. Lucas, A.~Smith, and S.~Risi, ``Evolving
  mario levels in the latent space of a deep convolutional generative
  adversarial network,'' in {\em Proceedings of the Genetic and Evolutionary
  Computation Conference}, pp.~221--228, 2018.

\bibitem{radford2015unsupervised}
A.~Radford, L.~Metz, and S.~Chintala, ``Unsupervised representation learning
  with deep convolutional generative adversarial networks,'' {\em arXiv
  preprint arXiv:1511.06434}, 2015.

\bibitem{salimans2016improved}
T.~Salimans, I.~Goodfellow, W.~Zaremba, V.~Cheung, A.~Radford, and X.~Chen,
  ``Improved techniques for training gans,'' in {\em Advances in neural
  information processing systems}, pp.~2234--2242, 2016.

\bibitem{muller1997integral}
A.~M{\"u}ller, ``Integral probability metrics and their generating classes of
  functions,'' {\em Advances in Applied Probability}, vol.~29, no.~2,
  pp.~429--443, 1997.

\bibitem{arora2017generalization}
S.~Arora, R.~Ge, Y.~Liang, T.~Ma, and Y.~Zhang, ``Generalization and
  equilibrium in generative adversarial nets (gans),'' {\em arXiv preprint
  arXiv:1703.00573}, 2017.

\bibitem{nair2010rectified}
V.~Nair and G.~E. Hinton, ``Rectified linear units improve restricted boltzmann
  machines,'' in {\em Proceedings of the 27th international conference on
  machine learning (ICML-10)}, pp.~807--814, 2010.

\bibitem{glorot2011deep}
X.~Glorot, A.~Bordes, and Y.~Bengio, ``Deep sparse rectifier neural networks,''
  in {\em Proceedings of the fourteenth international conference on artificial
  intelligence and statistics}, pp.~315--323, 2011.

\bibitem{maas2013rectifier}
A.~L. Maas, A.~Y. Hannun, and A.~Y. Ng, ``Rectifier nonlinearities improve
  neural network acoustic models,'' in {\em ICML Workshop on Deep Learning for
  Audio, Speech, and Language Processing}, 2013.

\bibitem{Goodfellow-et-al-2016}
I.~Goodfellow, Y.~Bengio, and A.~Courville, {\em Deep Learning}.
\newblock Cambridge, MA, USA: MIT Press, 2016.

\bibitem{bai2018approximability}
Y.~Bai, T.~Ma, and A.~Risteski, ``Approximability of discriminators implies
  diversity in gans,'' {\em arXiv preprint arXiv:1806.10586}, 2018.

\bibitem{liang2018well}
T.~Liang, ``On how well generative adversarial networks learn densities:
  Nonparametric and parametric results,'' {\em arXiv preprint
  arXiv:1811.03179}, 2018.

\bibitem{schreuder2021statistical}
N.~Schreuder, V.-E. Brunel, and A.~Dalalyan, ``Statistical guarantees for
  generative models without domination,'' in {\em Algorithmic Learning Theory},
  pp.~1051--1071, PMLR, 2021.

\bibitem{Block2021ANE}
A.~Block, Z.~Jia, Y.~Polyanskiy, and A.~Rakhlin, ``Intrinsic dimension
  estimation,'' {\em arXiv preprint arXiv:2106.04018}, 2021.

\bibitem{villani2008optimal}
C.~Villani, {\em Optimal transport: old and new}, vol.~338.
\newblock New York, NY, USA: Springer Science \& Business Media, 2008.

\bibitem{lee2002unsupervised}
T.-W. Lee and M.~S. Lewicki, ``Unsupervised image classification, segmentation,
  and enhancement using ica mixture models,'' {\em IEEE Transactions on Image
  Processing}, vol.~11, no.~3, pp.~270--279, 2002.

\bibitem{chen2010predictive}
N.~Chen, J.~Zhu, and E.~Xing, ``Predictive subspace learning for multi-view
  data: a large margin approach,'' {\em Advances in neural information
  processing systems}, vol.~23, 2010.

\bibitem{fang2017robust}
X.~Fang, S.~Teng, Z.~Lai, Z.~He, S.~Xie, and W.~K. Wong, ``Robust latent
  subspace learning for image classification,'' {\em IEEE transactions on
  neural networks and learning systems}, vol.~29, no.~6, pp.~2502--2515, 2017.

\bibitem{caron2018deep}
M.~Caron, P.~Bojanowski, A.~Joulin, and M.~Douze, ``Deep clustering for
  unsupervised learning of visual features,'' in {\em Proceedings of the
  European conference on computer vision (ECCV)}, pp.~132--149, 2018.

\bibitem{santambrogio2010models}
F.~Santambrogio, ``Models and applications of optimal transport in economics,
  traffic and urban planning,'' {\em arXiv preprint arXiv:1009.3857}, 2010.

\bibitem{galichon2017survey}
A.~Galichon, ``A survey of some recent applications of optimal transport
  methods to econometrics,'' 2017.

\bibitem{ganin2014unsupervised}
Y.~Ganin and V.~Lempitsky, ``Unsupervised domain adaptation by
  backpropagation,'' {\em arXiv preprint arXiv:1409.7495}, 2014.

\bibitem{courty2016optimal}
N.~Courty, R.~Flamary, D.~Tuia, and A.~Rakotomamonjy, ``Optimal transport for
  domain adaptation,'' {\em IEEE transactions on pattern analysis and machine
  intelligence}, vol.~39, no.~9, pp.~1853--1865, 2016.

\bibitem{monge1784memoire}
G.~Monge, {\em M{\'e}moire sur le calcul int{\'e}gral des {\'e}quations aux
  diff{\'e}rences partielles}.
\newblock Paris, France: Imprimerie royale, 1784.

\bibitem{caffarelli1992regularity}
L.~A. Caffarelli, ``The regularity of mappings with a convex potential,'' {\em
  Journal of the American Mathematical Society}, vol.~5, no.~1, pp.~99--104,
  1992.

\bibitem{caffarelli1992boundary}
L.~A. Caffarelli, ``Boundary regularity of maps with convex potentials,'' {\em
  Communications on pure and applied mathematics}, vol.~45, no.~9,
  pp.~1141--1151, 1992.

\bibitem{caffarelli1996boundary}
L.~A. Caffarelli, ``Boundary regularity of maps with convex potentials--ii,''
  {\em Annals of mathematics}, pp.~453--496, 1996.

\bibitem{urbas1988regularity}
J.~I. Urbas, ``Regularity of generalized solutions of monge-ampere equations,''
  {\em Mathematische Zeitschrift}, vol.~197, no.~3, pp.~365--393, 1988.

\bibitem{urbas1997second}
J.~Urbas, ``On the second boundary value problem for equations of monge-ampere
  type,'' {\em Journal fur die Reine und Angewandte Mathematik}, vol.~487,
  pp.~115--124, 1997.

\bibitem{cybenko1989approximation}
G.~Cybenko, ``Approximation by superpositions of a sigmoidal function,'' {\em
  Mathematics of control, signals and systems}, vol.~2, no.~4, pp.~303--314,
  1989.

\bibitem{hornik1991approximation}
K.~Hornik, ``Approximation capabilities of multilayer feedforward networks,''
  {\em Neural networks}, vol.~4, no.~2, pp.~251--257, 1991.

\bibitem{chui1992approximation}
C.~K. Chui and X.~Li, ``Approximation by ridge functions and neural networks
  with one hidden layer,'' {\em Journal of Approximation Theory}, vol.~70,
  no.~2, pp.~131--141, 1992.

\bibitem{barron1993universal}
A.~R. Barron, ``Universal approximation bounds for superpositions of a
  sigmoidal function,'' {\em IEEE Transactions on Information theory}, vol.~39,
  no.~3, pp.~930--945, 1993.

\bibitem{mhaskar1996neural}
H.~N. Mhaskar, ``Neural networks for optimal approximation of smooth and
  analytic functions,'' {\em Neural computation}, vol.~8, no.~1, pp.~164--177,
  1996.

\bibitem{yarotsky2017error}
D.~Yarotsky, ``Error bounds for approximations with deep relu networks,'' {\em
  Neural Networks}, vol.~94, pp.~103--114, 2017.

\bibitem{wasserman2006all}
L.~Wasserman, {\em All of nonparametric statistics}.
\newblock New York, NY, USA: Springer Science \& Business Media, 2006.

\bibitem{tsybakov2008introduction}
A.~B. Tsybakov, {\em Introduction to nonparametric estimation}.
\newblock New York, NY, USA: Springer Science \& Business Media, 2008.

\bibitem{moser1965volume}
J.~Moser, ``On the volume elements on a manifold,'' {\em Transactions of the
  American Mathematical Society}, vol.~120, no.~2, pp.~286--294, 1965.

\bibitem{tang2022minimax}
R.~Tang and Y.~Yang, ``Minimax rate of distribution estimation on unknown
  submanifold under adversarial losses,'' {\em arXiv preprint
  arXiv:2202.09030}, 2022.

\bibitem{tenenbaum2000global}
J.~B. Tenenbaum, V.~De~Silva, and J.~C. Langford, ``A global geometric
  framework for nonlinear dimensionality reduction,'' {\em Science}, vol.~290,
  no.~5500, pp.~2319--2323, 2000.

\bibitem{roweis2000nonlinear}
S.~T. Roweis and L.~K. Saul, ``Nonlinear dimensionality reduction by locally
  linear embedding,'' {\em science}, vol.~290, no.~5500, pp.~2323--2326, 2000.

\bibitem{pope2021intrinsic}
P.~Pope, C.~Zhu, A.~Abdelkader, M.~Goldblum, and T.~Goldstein, ``The intrinsic
  dimension of images and its impact on learning,'' {\em arXiv preprint
  arXiv:2104.08894}, 2021.

\bibitem{virmaux2018lipschitz}
A.~Virmaux and K.~Scaman, ``Lipschitz regularity of deep neural networks:
  analysis and efficient estimation,'' {\em Advances in Neural Information
  Processing Systems}, vol.~31, 2018.

\bibitem{pauli2021training}
P.~Pauli, A.~Koch, J.~Berberich, P.~Kohler, and F.~Allg{\"o}wer, ``Training
  robust neural networks using lipschitz bounds,'' {\em IEEE Control Systems
  Letters}, vol.~6, pp.~121--126, 2021.

\bibitem{gouk2021regularisation}
H.~Gouk, E.~Frank, B.~Pfahringer, and M.~J. Cree, ``Regularisation of neural
  networks by enforcing lipschitz continuity,'' {\em Machine Learning},
  vol.~110, no.~2, pp.~393--416, 2021.

\bibitem{wang2020two}
J.~Wang, R.~Gao, and Y.~Xie, ``Two-sample test using projected wasserstein
  distance: Breaking the curse of dimensionality,'' {\em arXiv preprint
  arXiv:2010.11970}, 2020.

\bibitem{wang2022manifold}
J.~Wang, M.~Chen, T.~Zhao, W.~Liao, and Y.~Xie, ``A manifold two-sample test
  study: Integral probability metric with neural networks,'' {\em arXiv
  preprint arXiv:2205.02043}, 2022.

\bibitem{lu2017expressive}
Z.~Lu, H.~Pu, F.~Wang, Z.~Hu, and L.~Wang, ``The expressive power of neural
  networks: A view from the width,'' in {\em Advances in neural information
  processing systems}, pp.~6231--6239, 2017.

\bibitem{park2020minimum}
S.~Park, C.~Yun, J.~Lee, and J.~Shin, ``Minimum width for universal
  approximation,'' {\em arXiv preprint arXiv:2006.08859}, 2020.

\bibitem{lu2020universal}
Y.~Lu and J.~Lu, ``A universal approximation theorem of deep neural networks
  for expressing probability distributions,'' {\em Advances in neural
  information processing systems}, vol.~33, pp.~3094--3105, 2020.

\bibitem{huang2022error}
J.~Huang, Y.~Jiao, Z.~Li, S.~Liu, Y.~Wang, and Y.~Yang, ``An error analysis of
  generative adversarial networks for learning distributions,'' {\em Journal of
  Machine Learning Research}, vol.~23, no.~116, pp.~1--43, 2022.

\bibitem{zhang2017discrimination}
P.~Zhang, Q.~Liu, D.~Zhou, T.~Xu, and X.~He, ``On the
  discrimination-generalization tradeoff in gans,'' {\em arXiv preprint
  arXiv:1711.02771}, 2017.

\bibitem{jiang2018computation}
H.~Jiang, Z.~Chen, M.~Chen, F.~Liu, D.~Wang, and T.~Zhao, ``On computation and
  generalization of gans with spectrum control,'' {\em arXiv preprint
  arXiv:1812.10912}, 2018.

\bibitem{liang2017well}
T.~Liang, ``How well can generative adversarial networks learn densities: A
  nonparametric view,'' {\em arXiv preprint arXiv:1712.08244}, 2017.

\bibitem{luise2020generalization}
G.~Luise, M.~Pontil, and C.~Ciliberto, ``Generalization properties of optimal
  transport gans with latent distribution learning,'' {\em arXiv preprint
  arXiv:2007.14641}, 2020.

\bibitem{bartlett2017spectrally}
P.~L. Bartlett, D.~J. Foster, and M.~J. Telgarsky, ``Spectrally-normalized
  margin bounds for neural networks,'' in {\em Advances in Neural Information
  Processing Systems}, pp.~6240--6249, 2017.

\bibitem{chae2021likelihood}
M.~Chae, D.~Kim, Y.~Kim, and L.~Lin, ``A likelihood approach to nonparametric
  estimation of a singular distribution using deep generative models,'' {\em
  arXiv preprint arXiv:2105.04046}, 2021.

\bibitem{singh2018nonparametric}
S.~Singh, A.~Uppal, B.~Li, C.-L. Li, M.~Zaheer, and B.~P{\'o}czos,
  ``Nonparametric density estimation under adversarial losses,'' in {\em
  Advances in Neural Information Processing Systems}, pp.~10225--10236, 2018.

\bibitem{uppal2019nonparametric}
A.~Uppal, S.~Singh, and B.~P{\'o}czos, ``Nonparametric density estimation \&
  convergence of gans under besov ipm losses,'' {\em arXiv preprint
  arXiv:1902.03511}, 2019.

\bibitem{van1996weak}
A.~W. Van Der~Vaart and J.~A. Wellner, ``Weak convergence,'' in {\em Weak
  convergence and empirical processes}, pp.~16--28, New York, NY, USA:
  Springer, 1996.

\bibitem{gyorfi2006distribution}
L.~Gy{\"o}rfi, M.~Kohler, A.~Krzyzak, and H.~Walk, {\em A distribution-free
  theory of nonparametric regression}.
\newblock New York, NY, USA: Springer Science \& Business Media, 2006.

\bibitem{dudley1967sizes}
R.~M. Dudley, ``The sizes of compact subsets of hilbert space and continuity of
  gaussian processes,'' {\em Journal of Functional Analysis}, vol.~1, no.~3,
  pp.~290--330, 1967.

\bibitem{mohri2018foundations}
M.~Mohri, A.~Rostamizadeh, and A.~Talwalkar, {\em Foundations of machine
  learning}.
\newblock Cambridge, MA, USA: MIT press, 2018.

\bibitem{nickl2007bracketing}
R.~Nickl and B.~M. P{\"o}tscher, ``Bracketing metric entropy rates and
  empirical central limit theorems for function classes of besov-and
  sobolev-type,'' {\em Journal of Theoretical Probability}, vol.~20, no.~2,
  pp.~177--199, 2007.

\bibitem{krizhevsky2012imagenet}
A.~Krizhevsky, I.~Sutskever, and G.~E. Hinton, ``Imagenet classification with
  deep convolutional neural networks,'' in {\em Advances in neural information
  processing systems}, pp.~1097--1105, 2012.

\bibitem{zeiler2014visualizing}
M.~D. Zeiler and R.~Fergus, ``Visualizing and understanding convolutional
  networks,'' in {\em European conference on computer vision}, pp.~818--833,
  Springer, 2014.

\bibitem{zhou2018interpreting}
B.~Zhou, D.~Bau, A.~Oliva, and A.~Torralba, ``Interpreting deep visual
  representations via network dissection,'' {\em IEEE transactions on pattern
  analysis and machine intelligence}, vol.~41, no.~9, pp.~2131--2145, 2018.

\bibitem{che2016mode}
T.~Che, Y.~Li, A.~P. Jacob, Y.~Bengio, and W.~Li, ``Mode regularized generative
  adversarial networks,'' {\em arXiv preprint arXiv:1612.02136}, 2016.

\bibitem{cao2018improving}
Y.~Cao, G.~W. Ding, K.~Y.-C. Lui, and R.~Huang, ``Improving gan training via
  binarized representation entropy (bre) regularization,'' {\em arXiv preprint
  arXiv:1805.03644}, 2018.

\bibitem{zhang2016understanding}
C.~Zhang, S.~Bengio, M.~Hardt, B.~Recht, and O.~Vinyals, ``Understanding deep
  learning requires rethinking generalization,'' {\em arXiv preprint
  arXiv:1611.03530}, 2016.

\bibitem{jacot2018neural}
A.~Jacot, F.~Gabriel, and C.~Hongler, ``Neural tangent kernel: Convergence and
  generalization in neural networks,'' in {\em Advances in neural information
  processing systems}, pp.~8571--8580, 2018.

\bibitem{du2018gradient_b}
S.~S. Du, J.~D. Lee, H.~Li, L.~Wang, and X.~Zhai, ``Gradient descent finds
  global minima of deep neural networks,'' {\em arXiv preprint
  arXiv:1811.03804}, 2018.

\bibitem{allen2018convergence}
Z.~Allen-Zhu, Y.~Li, and Z.~Song, ``A convergence theory for deep learning via
  over-parameterization,'' {\em arXiv preprint arXiv:1811.03962}, 2018.

\bibitem{du2018gradient_a}
S.~S. Du, X.~Zhai, B.~Poczos, and A.~Singh, ``Gradient descent provably
  optimizes over-parameterized neural networks,'' {\em arXiv preprint
  arXiv:1810.02054}, 2018.

\bibitem{li2018learning}
Y.~Li and Y.~Liang, ``Learning overparameterized neural networks via stochastic
  gradient descent on structured data,'' in {\em Advances in Neural Information
  Processing Systems}, pp.~8157--8166, 2018.

\bibitem{arora2019fine}
S.~Arora, S.~S. Du, W.~Hu, Z.~Li, and R.~Wang, ``Fine-grained analysis of
  optimization and generalization for overparameterized two-layer neural
  networks,'' {\em arXiv preprint arXiv:1901.08584}, 2019.

\bibitem{allen2019learning}
Z.~Allen-Zhu, Y.~Li, and Y.~Liang, ``Learning and generalization in
  overparameterized neural networks, going beyond two layers,'' in {\em
  Advances in neural information processing systems}, pp.~6155--6166, 2019.

\bibitem{du2019width}
S.~S. Du and W.~Hu, ``Width provably matters in optimization for deep linear
  neural networks,'' {\em arXiv preprint arXiv:1901.08572}, 2019.

\bibitem{chen2019efficient}
M.~Chen, H.~Jiang, W.~Liao, and T.~Zhao, ``Efficient approximation of deep relu
  networks for functions on low dimensional manifolds,'' in {\em Advances in
  Neural Information Processing Systems}, pp.~8172--8182, 2019.

\bibitem{weed2019sharp}
J.~Weed and F.~Bach, ``Sharp asymptotic and finite-sample rates of convergence
  of empirical measures in wasserstein distance,'' {\em Bernoulli}, vol.~25,
  no.~4A, pp.~2620--2648, 2019.

\bibitem{wainwright2019high}
M.~J. Wainwright, {\em High-dimensional statistics: A non-asymptotic
  viewpoint}, vol.~48.
\newblock Cambridge University Press, 2019.

\end{thebibliography}

\appendix
\section{Proof of Theorem \ref{thm:mixture}}\label{pf:mixture}
The proof is built upon the framework for establishing Theorem \ref{thm:main}. We restate the oracle inequality in Lemma \ref{lemma:oracle} (with $\beta = 1$):
\begin{align*}
\dw\left((g_\theta^*)_\sharp \rho, \mu \right) \leq \inf_{g_\theta \in \cG_{\rm NN}^{\rm mix}} d_{\cH^1_\infty}((g_\theta)_{\sharp}\rho, \mu) + 4 \sup_{\norm{f}_{\rm Lip} \leq 1} \inf_{f_\omega \in \cF_{\rm NN}^{\rm mix}} \norm{f - f_{\omega}}_\infty + \dw(\hat{\mu}_n, \mu) + d_{\cF_{\rm NN}^{\rm mix}}(\mu, \hat{\mu}_n).
\end{align*}
The remaining proof consists of two major parts: 1) bounding generator approximation error $\inf_{g_\theta \in \cG_{\rm NN}^{\rm mix}} d_{\cH^1_\infty}((g_\theta)_{\sharp}\rho, \mu)$; 2) fast convergence of $\EE[\dw(\hat{\mu}_n, \mu) + d_{\cF_{\rm NN}^{\rm mix}}(\mu, \hat{\mu}_n)]$. Note that the discriminator approximation error $\sup_{\norm{f}_{\rm Lip} \leq 1} \inf_{f_\omega \in \cF_{\rm NN}^{\rm mix}} \norm{f - f_{\omega}}_\infty$ is a direct consequence of Lemma \ref{lemma:approx_lip} (replacing $q$ by $d$).

\noindent $\bullet$ {\bf Bounding generator approximation error}. Given $\epsilon_1 > 0$, we constructively show the existence of a generator network architecture giving rise to an $\epsilon_1$ approximation of $\mu$.
\begin{lemma}[Restatement of Proposition \ref{prop:mixture_approximation}]\label{lemma:mixture_generator}
Suppose Assumption \ref{assumption:mixture} holds. Given any $\epsilon_1 > 0$, we choose network architecture ${\rm NN}(R, \kappa, L, p, J, d_{\rm in} = q+1, d_{\rm out} = d)$ with
\begin{align*}
R = 1, \kappa = \max\{C_\alpha, 1\}, L = O\left(\log \frac{1}{\epsilon_1} \right), p = O\left(K d \epsilon_1^{-\frac{q}{\alpha}}\right), J = O\left(K d \epsilon_1^{-\frac{q}{\alpha}} \log \frac{1}{\epsilon_1} \right).
\end{align*}
Then there exists $g_\theta \in \cG_{\rm NN}$ such that
\begin{align*}
d_{\cH^1_\infty}\left((g_\theta)_\sharp \rho, \mu\right) \leq \epsilon_1.
\end{align*}
\end{lemma}
\begin{proof}[Proof of Lemma \ref{lemma:mixture_generator}]
We adopt a two-step construction: 1) we use the first coordinate $z_1$ of $z \sim \rho$ to generate a latent variable approximately distributed like $\PP(z_1 = k) = p_k$ for $k = 1, \dots, K$; 2) we use the remaining $q$ coordinates of $\rho$ to generate data approximating each component in the mixture.

\noindent $\star$ {\it Generating latent variable}. Let $m_k = \sum_{i=1}^k p_k$ for $k =1, \dots, K$ and $m_0 = 0$. We define a trapezoid function on $[m_k, m_{k+1}]$ as
\begin{align*}
T_k(a) = \begin{cases}
1 & \text{for}~ a \in [m_{k-1} + \epsilon_1/(12K), m_{k} - \epsilon_1/(12K)], ~\text{if}~ p_{k} > \epsilon_1/(6K) \\
\frac{12K}{\epsilon_1} (a - m_{k-1}) & \text{for}~ a \in [m_{k-1}, m_{k-1} + \epsilon_1/(12K)], ~\text{if}~ p_{k} > \epsilon_1/(6K) \\
1 - \frac{12K}{\epsilon_1} (m_{k} - a) & \text{for}~ a \in [m_{k} - \epsilon_1/(12K), m_{k}], ~\text{if}~ p_{k} > \epsilon_1/(6K) \\
0 & \text{otherwise}
\end{cases}.
\end{align*}
See Figure \ref{fig:latent} for an illustration.
\begin{figure}[!htb]
\centering
\includegraphics[width = 0.45\textwidth]{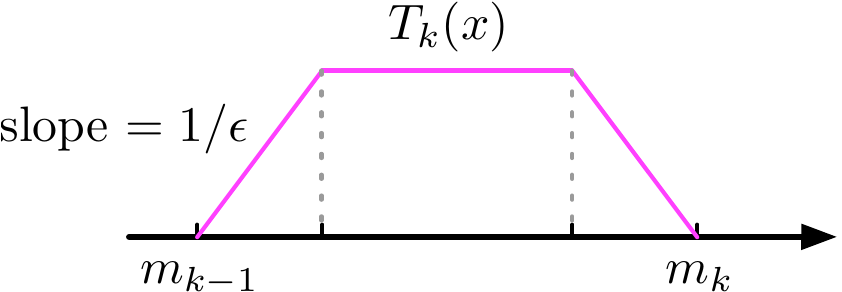}
\caption{Trapezoid function $T_k(x)$ on interval $[m_{k-1}, m_{k}]$ if $p_{k} > \epsilon_1/(6K)$.}
\label{fig:latent}
\end{figure}

Passing the first coordinate $z_1$ through $T_m$ yields an approximate binary random variable, with $\PP(T_k(z_1) = 0) = 1 - p_{k}$ and $\PP(T_k(z_1) = 1) = p_k - \epsilon_1/(6K)$.

\noindent $\star$ {\it Generating each component}. Since each component can be represented by $\mu_k = (g_k)_\sharp {\rm Unif}([0, 1]^q)$, it suffices to approximate $g_k$ using neural networks. Given the regularity of $g_k$ in Assumption \ref{assumption:mixture}, invoking Theorem \ref{thm:generator} (built upon Lemma \ref{lemma:approx}), we obtain that there exists a neural network ${\rm NN}_{g_k}(R, \kappa, L, p, J, d_{\rm in} = q, d_{\rm out} = d)$ that yields an approximation $\hat{g}_k$ of $g_k$ with $$\sup_{z \in [0, 1]^q} \norm{\hat{g}_k(z) - g_k(z)}_\infty \leq \epsilon_1/(3K).$$ Such a network has
\begin{align}\label{eq:gk_NN_size}
R = 1, \ \kappa = \max\{C_\alpha, 1\}, \ L = O\left(\log \frac{1}{\epsilon_1}\right), \ p = O\left(d \epsilon_1^{-q/\alpha_k}\right), \ J = O\left(d \epsilon_1^{-q/\alpha_k} \log \frac{1}{\epsilon_1}\right).
\end{align}

\noindent $\star$ {\it Putting together}. For any $z \sim {\rm Unif}([0, 1]^{q+1})$, we obtain a generated sample by
\begin{align}\label{eq:mixture_generate}
\tilde{x} = \sum_{k=1}^K T_k(z_1) \hat{g}_k(z_{\backslash 1}),
\end{align}
where $z_{\backslash 1} = [z_2, \dots, z_{q+1}]^\top$. However, \eqref{eq:mixture_generate} cannot be exactly implemented by a neural network due to the multiplication operation. We adopt Proposition 3 in \cite{yarotsky2017error} (see also Corollary 1 in \cite{chen2019efficient}) for approximating multiplication. Specifically, we denote $\hat{\times}(\cdot, \cdot)$ as an $\frac{\epsilon_1}{3K}$-approximation of multiplication and applying it entrywise yields
\begin{align*}
\hat{x} = \sum_{k=1}^K \hat{\times}\left(T_k(z_1), \hat{g}_k(z_{\backslash 1})\right).
\end{align*}
We use an abstract notation $g_\theta$ to denote the mapping from $z$ to $\hat{x}$. Now we bound the distribution approximation error
\begin{align*}
d_{\cH^1_\infty} \left((g_\theta)_\sharp \rho, \mu\right) & = \sup_{f \in \cH^1_\infty} \EE_{z \sim \rho} [f(g_\theta(z))] - \EE_{x \sim \mu} [f(x)] \\
& = \sup_{f \in \cH^1_\infty} \EE_{z \sim \rho} [f(g_\theta(z))] - \EE_{z \sim \rho} \left[f\left(\sum_{k=1}^K g_k(z_{\backslash 1}) \mathds{1}\left\{z_1 \in [m_{k-1}, m_k]\right\}\right)\right] \\
& = \sup_{f \in \cH^1_\infty} \EE_{z \sim \rho} \left[f\left(\sum_{k=1}^K \hat{\times}\left(\hat{g}_k(z_{\backslash 1}), T_k(z_1)\right)\right) - f\left(\sum_{k=1}^K g_k(z_{\backslash 1}) \mathds{1}\left\{z_1 \in [m_{k-1}, m_k]\right\}\right) \right] \\
& \overset{(i)}{\leq} \EE_{z \sim \rho} \left[\norm{\sum_{k=1}^K \hat{\times}\left(\hat{g}_k(z_{\backslash 1}), T_k(z_1)\right) - g_k(z_{\backslash 1}) \mathds{1}\left\{z_1 \in [m_{k-1}, m_k]\right\}}_\infty\right] \\
& \leq \sum_{k=1}^K \EE_{z \sim \rho} \left[\norm{\hat{\times}\left(\hat{g}_k(z_{\backslash 1}), T_k(z_1)\right) - g_k(z_{\backslash 1}) \mathds{1}\left\{z_1 \in [m_{k-1}, m_k]\right\}}_\infty\right] \\
& \leq \underbrace{\sum_{k=1}^K \EE_{z \sim \rho} \left[\norm{\hat{\times}\left(\hat{g}_k(z_{\backslash 1}), T_k(z_1)\right) - \hat{g}_k(z_{\backslash 1}) T_k(z_1)}_\infty\right]}_{(A)} \\
& \quad + \underbrace{\sum_{k=1}^K \EE_{z \sim \rho} \left[\norm{\hat{g}_k(z_{\backslash 1}) T_k(z_1) - g_k(z_{\backslash 1}) \mathds{1}\left\{z_1 \in [m_{k-1}, m_k]\right\}}_\infty\right]}_{(B)}.
\end{align*}
Term $(A)$ is bounded by $\epsilon_1/3$ since $\hat{\times}$ is an $\frac{\epsilon}{3K}$-approximation of multiplication. Term $(B)$ can be estimated as follows.
\begin{align*}
(B) & = \sum_{k=1}^K \EE_{z \sim \rho} \left[\norm{\hat{g}_k(z_{\backslash 1}) T_k(z_1) - g_k(z_{\backslash 1}) \mathds{1}\left\{z_1 \in [m_{k-1}, m_k]\right\}}_\infty \mathds{1}\{T_k(z_1) = 1\}\right] \\
& \quad + \sum_{k=1}^K \EE_{z \sim \rho} \left[\norm{\hat{g}_k(z_{\backslash 1}) T_k(z_1) - g_k(z_{\backslash 1}) \mathds{1}\left\{z_1 \in [m_{k-1}, m_k]\right\}}_\infty \mathds{1}\{T_k(z_1) < 1\}\right] \\
& \overset{(i)}{\leq} \sum_{k=1}^K \EE_{z \sim \rho} \left[\norm{\hat{g}_k(z_{\backslash 1}) - g_k(z_{\backslash 1})}_\infty \right] + 2 \sum_{k=1}^K \EE_{z \sim \rho} \left[\mathds{1}\{T_k(z_1) < 1\}\right] \\
& \leq \epsilon_1/3 + \epsilon_1/3.
\end{align*}
Adding $(A)$ and $(B)$, we obtain $$d_{\cH^1_\infty} \left((g_\theta)_\sharp \rho, \mu\right) \leq \epsilon_1.$$

The remaining step is to determine the network size for implementing $g_\theta$. Note that $g_\theta$ can be exactly represented by a network. In particular, $g_\theta$ consists of $K$ parallel sub-networks, each computing $\hat{\times}(T_k(z_1), \hat{g}_k(z_{\backslash 1}))$. By Corollary 1 in \cite{chen2019efficient}, $\hat{\times}$ can be implemented by a $O(\log 1/\epsilon_1)$-depth constant-width network. $T_k$ is a piecewise linear function with at most $4$ break points. Thus can be realized by a single-layer constant-width network. We remark that using a single layer network for implementing $T_k$ requires the weight parameter to be as large as $12/\epsilon_1$. In order to have constant bounded weight parameters, we utilize the same trick in the proof of Lemma \ref{lemma:approx_lip} (Appendix \ref{pf:approx_lip}). As a consequence, we need a $O(\log 1/\epsilon_1)$-depth constant-width for realizing $T_k$. Lastly, the network size for $\hat{g}_k$ is given in \eqref{eq:gk_NN_size}. Putting together, we have $g_\theta$ implementable by $\cG_{\rm NN}^{\rm mix}$ with
\begin{align}\label{eq:mix_generator_size}
R = 1, \ \kappa = \max\{C_\alpha, 1\}, \ L = O\left(\log \frac{1}{\epsilon_1}\right), \ p = O\left(K d \epsilon_1^{-q/\alpha}\right), \ J = O\left(K d \epsilon_1^{-q/\alpha} \log \frac{1}{\epsilon_1}\right).
\end{align}
Proposition \ref{prop:mixture_approximation} holds by observing $\norm{x}_2 \leq \sqrt{d} \norm{x}_\infty$ for any $x \in \RR^d$.
\end{proof}

\noindent $\bullet$ {\bf Bounding statistical error}. Given $\epsilon_2 > 0$, Lemma \ref{lemma:approx_lip} suggests that we can choose $\cF_{\rm NN}^{\rm mix}$ with
\begin{equation}\label{eq:mix_discriminator_size}
\begin{split}
& \bar{R} = \sqrt{d},\quad \bar{\kappa} = O(1), \quad \bar{\gamma} = 10d, \quad \bar{L} = O\left(\log \frac{1}{\epsilon_2} + d \right), \\
& ~\quad~ \bar{p} = O\left(\epsilon_2^{-d}\right),\quad \bar{J} = O\left(\epsilon_2^{-d}\left(\log 1/\epsilon_2 + d\right) \right)
\end{split}
\end{equation}
such that any $1$-Lipschitz function $f$ on $[0, 1]^d$ can be approximated by $\hat{f} \in \cF_{\rm NN}^{\rm mix}$ up to error $\epsilon_2$. Since functions in $\cF_{\rm NN}^{\rm mix}$ is $10d$-Lipschitz continuous, we have
\begin{align*}
d_{\cF_{\rm NN}^{\rm mix}} (\hat{\mu}_n, \mu) \leq 10d \dw(\hat{\mu}_n, \mu).
\end{align*}
We invoke the following fast convergence of empirical data distribution to its population counterpart in terms of Wasserstein-1 distance.
\begin{lemma}[Theorem 1 in \cite{weed2019sharp}]\label{lemma:fast_w1}
For any $\delta > 0$, we have
\begin{align*}
\EE[\dw(\hat{\mu}_n, \mu)] \leq C'_\delta n^{-\frac{1}{d^*(\mu) + \delta}},
\end{align*}
where $d^*(\mu)$ is the upper Wasserstein dimension of distribution $\mu$ (Definition 4 in \cite{weed2019sharp}) and $C'_\delta$ is independent of $n$.
\end{lemma}
To apply Lemma \ref{lemma:fast_w1}, we need to find $d^*(\mu)$ under Assumption \ref{assumption:mixture}. It suffices to upper bound $d^*(\mu)$ using
\begin{align}\label{eq:wasserstein_dimension}
d^*(\mu) \leq \limsup_{\epsilon \to 0} \frac{\log \cN(\epsilon, \cX, \norm{\cdot}_2)}{-\log \epsilon},
\end{align}
where $\cN$ is the covering number of $\cX$. We construct a covering of $\cX$ utilizing $g_k$'s to bound the covering number. For $\epsilon < 1$, let $\{u_i\}_{i=1}^{\cN(\epsilon/C_\alpha, [0, 1]^q, \norm{\cdot}_2)}$ be an $(\epsilon/C_\alpha)$-covering of $[0, 1]^q$. We claim that $\{x_{k, i} = g_k(u_i)\}$ forms an $\epsilon$-covering of $\cX$. To see this, let $x \in \cX$ be arbitrary. There exists at least one $g_k$, such that $g_k(u) = x$ for some $u \in [0, 1]^q$. We can find $u_i$ satisfying $\norm{u_i - u}_2 \leq \epsilon$. Then we evaluate
\begin{align*}
\norm{g_k(u_i) - x}_2 = \norm{g_k(u_i) - g_k(u)}_2 \overset{(i)}{=} \norm{\nabla g_k(v) (u_i - u)}_2 \overset{(ii)}{\leq} C_\alpha \norm{u_i - u}_2 \leq \epsilon.
\end{align*}
Equality $(i)$ follows from first-order Taylor expansion. Inequality $(ii)$ follows from $\norm{g_k}_{\cH^{\alpha_k}} \leq C_\alpha$. The covering number of a unit cube $\cN(\epsilon/C_\alpha, [0, 1]^q, \norm{\cdot}_2)$ can be obtained by a volume ratio argument \cite[Lemma 5.2]{wainwright2019high},
\begin{align*}
\cN(\epsilon/C_\alpha, [0, 1]^q, \norm{\cdot}_2) \leq c \left(1 + \frac{2C_\alpha}{\epsilon}\right)^q \quad \text{for~some~absolute~constant}~c.
\end{align*}
Using the last display, the cardinality of $\{x_{k, i} = g_k(u_i)\}$ is bounded by $c K \left(1 + \frac{2C_\alpha}{\epsilon}\right)^q$. Substituting into \eqref{eq:wasserstein_dimension}, we have
\begin{align*}
d^*(\mu) \leq \limsup_{\epsilon \to 0} \frac{q \log cK(1+2C_\alpha/\epsilon)}{-\log \epsilon} = q.
\end{align*}
Consequently, the statistical error is bounded by
\begin{align*}
\EE[\dw(\hat{\mu}_n, \mu) + d_{\cF_{\rm NN}^{\rm mix}}(\hat{\mu}_n, \mu)] \leq (10d+1) \EE[\dw(\hat{\mu}_n, \mu)] \leq (10d+1) C'_\delta n^{-\frac{1}{q + \delta}} \quad \text{for~any~constant}~\delta > 0.
\end{align*}

\noindent $\bullet$ {\bf Balancing error terms}. Summing up generator/discriminator approximation error and statistical error in the oracle inequality, we derive
\begin{align*}
\EE\left[\dw\left((g_\theta^*)_\sharp \rho, \mu \right)\right] \leq \epsilon_1 + 4\epsilon_2 + (10d+1) C'_\delta n^{-\frac{1}{q+\delta}}.
\end{align*}
It suffices to choose $\epsilon_1 = \epsilon_2 = n^{-\frac{1}{q}}$, which gives rise to
\begin{align*}
\EE\left[\dw\left((g_\theta^*)_\sharp \rho, \mu \right)\right] \leq C_\delta d n^{-\frac{1}{q + \delta}} \quad \text{for~any~constant}~\delta > 0,
\end{align*}
where $C_\delta$ is independent of $n$. Substituting $\epsilon_1, \epsilon_2$ into the generator and discriminator sizes in \eqref{eq:mix_generator_size} and \eqref{eq:mix_discriminator_size}, respectively, we complete the proof.

\section{Detailed proofs in Section \ref{sec:lowdproof}}
\subsection{Proof of Lemma \ref{lemma:approx_lip}}\label{pf:approx_lip}
\begin{proof}[Proof of Lemma \ref{lemma:approx_lip}]
The proof consists of two steps: 1) construction of a piecewise linear function for approximating $1$-Lipschitz functions, which can be implemented by a ReLU neural network; 2) establishing the global Lipschitz continuity of the neural network, in addition to the $L_\infty$ approximation error guarantee.

\noindent {\bf Step 1)}. Given a positive integer $N > 0$, we evenly choose $(N+1)^q$ points in the hypercube $[0, 1]^q$, denoted as $m/N$ with $m = [m_1, \dots, m_q]^\top \in \{0, \dots, N\}^q$. We define a univariate trapezoid function (see graphical illustration in Figure \ref{fig:trapezoid})
\begin{align*}
\phi(a) = \begin{cases}
1, & \vert a\vert < 1 \\
2 - \vert a\vert, & \vert a\vert \in [1, 2] \\
0, & \vert a\vert > 2 \\
\end{cases}.
\end{align*}
Then for any $x \in [0, 1]^q$, we define a partition of unity based on a product of trapezoid functions indexed by $m$,
\begin{align*}
\xi_m(x) = \prod_{k=1}^q \phi\left(3N\left(x_k - \frac{m_k}{N}\right)\right).
\end{align*}
\begin{figure}[!htb]
\centering
\includegraphics[width = 0.675\textwidth]{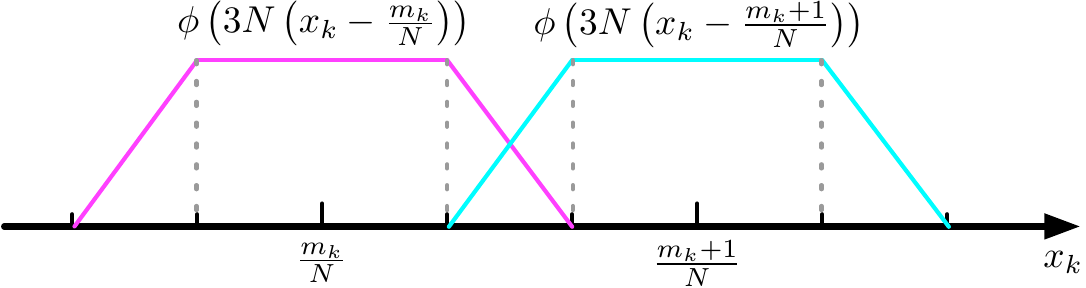}
\caption{Trapezoid function in one dimension.}
\label{fig:trapezoid}
\end{figure}

For any target $1$-Lipschitz function $f$, it is more convenient to write its Lipschitz continuity with respect to the $\ell_\infty$ norm, i.e.,
\begin{align}\label{eq:lip_linfty}
\vert f(x) - f(y)\vert \leq \norm{x - y}_2 \leq \sqrt{q} \norm{x - y}_\infty.
\end{align}
We now define a collection of piecewise constant functions 
\begin{align*}
P_m(x) = f(m) \quad \text{for}\quad m \in \{0, \dots, N\}^q.
\end{align*}
We claim that $\tilde{f}(x) = \sum_{m} \xi_m(x) P_m(x)$ is an approximation of $f$, with an approximation error evaluated as
\begin{align*}
\sup_{x \in [0, 1]^q} \left\vert \tilde{f}(x) - f(x)\right\vert & = \sup_{x \in [0, 1]^q} \left\vert \sum_{m} \xi_m(x) \left(P_m(x) - f(x)\right)\right\vert \\
& \leq \sup_{x \in [0, 1]^q} \sum_{m: \vert x_k - m_k/N\vert \leq \frac{2}{3N}} \left\vert P_m(x) - f(x) \right\vert \\
& = \sup_{x \in [0, 1]^q} \sum_{m: \vert x_k - m_k/N\vert \leq \frac{2}{3N}} \left\vert f(m) - f(x) \right\vert \\
& \leq \sqrt{q} 2^{q+1} \frac{1}{3N},
\end{align*}
where the last inequality follows from the Lipschitz continuity in \eqref{eq:lip_linfty} and the fact that there are at most $2^q$ terms in the summation.

We use a ReLU network to implement $\tilde{f}$. It turns out that we only need to implement the multiplication operation in $\xi_m$. For scalars $a, b \in [0, 1]$, we rewrite $ab$ as $\left(\frac{a+b}{2}\right)^2 - \left(\frac{\vert a-b\vert}{2}\right)^2$. We know neural networks can approximate a univariate quadratic function on $[0, 1]$ as
\begin{align}\label{eq:h_hat}
a^2 \approx \hat{h}_K(a) = a - \sum_{k=1}^K \frac{1}{2^{2k}} g_k(a), \quad \textrm{with} \quad g_k = \underbrace{g \circ \cdots \circ g}_{k ~\textrm{compositions}},
\end{align}
where $g(a) = 2 {\rm ReLU}(a) - 4{\rm ReLU} (a - 0.5) + 2 {\rm ReLU}(a - 1)$. The $L_\infty$ approximation error of $\hat{h}_K$ is $2^{-(2K + 2)}$ (A proof can be found in \cite[Proposition 2]{yarotsky2017error} or \cite[Lemma 1]{chen2019efficient}).
We approximate $\xi_m$ recursively using univariate quadratic functions. Specifically, we construct
\begin{align}\label{eq:xi_hat}
\xi_m(x) \approx \hat{\xi}_m(x) = \hat{\times}\left(\phi(3N(x_q - m_q/N)), \hat{\times}\left(\phi(3N(x_{q-1} - m_{q-1}/N)), \dots \right)\right),
\end{align}
where $\hat{\times}(a, b) = \hat{h}_K((a+b)/2) - \hat{h}_K(\vert a-b\vert/2)$ for $a, b \in [0, 1]$. Then the network for approximating $f$ is obtained as
\begin{align}\label{eq:f_hat}
f(x) \approx \hat{f}(x) = \sum_m \hat{\xi}_m(x) f(m).
\end{align}
We bound $L_\infty$ approximation error of $\hat{f}$ as
\begin{align*}
\norm{\hat{f} - f}_\infty & \leq \norm{\hat{f} - \tilde{f}}_\infty + \norm{\tilde{f} - f}_\infty \\
& \leq \sup_{x \in [0, 1]^q} \left\vert \sum_{m} \left(\hat{\xi}_m(x) - \xi_{m}(x) \right) P_m(x)\right\vert + \sqrt{q} 2^{q+1} \frac{1}{3N} \\
& \leq \norm{f}_\infty \sup_{x \in [0, 1]^q} \left\vert \sum_{m} \left(\hat{\xi}_m(x) - \xi_{m}(x) \right) \right\vert + \sqrt{q} 2^{q+1} \frac{1}{3N} \\
& \leq 2^q \norm{f}_\infty \norm{\hat{\xi}_m - \xi_m}_\infty + \sqrt{q} 2^{q+1} \frac{1}{3N} \\
& \leq q 2^q \norm{f}_\infty 2^{-2K - 1} + \sqrt{q} 2^{q+1} \frac{1}{3N},
\end{align*}
where the last inequality follows from recursively decomposing $\norm{\hat{\xi}_m - \xi_m}_\infty$ into $q$ terms as
\begin{align*}
\norm{\hat{\xi}_m - \xi_m}_\infty & \leq \Big\| \hat{\times}\left(\phi(3N(x_q - m_q/N)), \hat{\times}\left(\phi(3N(x_{q-1} - m_{q-1}/N)), \dots \right)\right) \\
& \quad - \phi(3N(x_q - m_q/N)) \cdot \hat{\times}\left(\phi(3N(x_{q-1} - m_{q-1}/N)), \dots \right) \Big\|_\infty \\
& \quad + \dots \\
& \quad + \phi(3N(x_q - m_q/N)) \cdots \phi(3N(x_3 - m_3/N)) \\
& \quad \cdot \Big\| \hat{\times} \left(\phi(3N(x_2 - m_2/N)), \phi(3N(x_1 - m_1/N))\right) \\
& \qquad - \phi(3N(x_2 - m_2/N)) \phi(3N(x_1 - m_1/N) \Big\|_\infty
\end{align*}
and observing
\begin{align*}
\vert \hat{\times}(a, b) - ab\vert & \leq \left\vert \hat{h}_K((a+b)/2) - (a+b)^2/4\right\vert + \left\vert \hat{h}_K(\vert a-b\vert/2) - (a-b)^2/4\right\vert \\
& \leq 2 \cdot 2^{-2K - 2} = 2^{-2K - 1}
\end{align*}
for any $a, b \in [0, 1]$.

\noindent {\bf Step 2)}. The following lemma establishes the Lipschitz continuity of $\hat{f}$ with respect to the $\ell_\infty$ norm.
\begin{lemma}\label{lemma:hatf_lip}
Let $\hat{f}$ be defined in \eqref{eq:f_hat}. Then for any $x, y \in [0, 1]^q$, it holds
\begin{align}
\left\vert \hat{f}(x) - \hat{f}(y)\right\vert \leq 3q \left(3 + 2(N \norm{f}_\infty + 1) \cdot q 2^{-K+q-1} \frac{1 - \left(q2^{-K}\right)^q}{1 - q2^{-K}} \right) \norm{x - y}_\infty. \nonumber
\end{align}
\end{lemma}
The proof is deferred to Appendix \ref{pf:hatf_lip}. Given Lemma \ref{lemma:hatf_lip}, we choose $N = \left\lceil \frac{\sqrt{q} 2^{q+1}}{\epsilon_2} \right\rceil$ and $K$ satisfying
\begin{align*}
2(N \norm{f}_\infty + 1) \cdot q 2^{-K+q-1} \frac{1 - \left(q2^{-K}\right)^q}{1 - q2^{-K}} \leq \frac{1}{3},
\end{align*}
which implies $K = \left\lceil \log \frac{12q^{3/2} (\norm{f}_\infty + 1)}{\epsilon_2} + 2q \right\rceil$. As a result, we check the $L_\infty$ approximation error of $\hat{f}$ as
\begin{align*}
\norm{\hat{f} - f}_\infty & \leq q 2^q \norm{f}_\infty 2^{-2K - 1} + \sqrt{q} 2^{q+1} \frac{1}{3N} \\
& \leq \frac{1}{9 q^2 2^{3q+5} (\norm{f}_\infty + 1)}\epsilon_2^2 + \frac{1}{3}\epsilon_2 \\
& \leq \epsilon_2.
\end{align*}
Meanwhile, with the choice of $K$ and $N$, Lemma \ref{lemma:hatf_lip} implies that for any $x, y \in [0, 1]^q$, it holds
\begin{align*}
\left\vert \hat{f}(x) - \hat{f}(y) \right\vert & \leq 3q \left(3 + 2(N \norm{f}_\infty + 1) \cdot q 2^{-K+q-1} \frac{1 - \left(q2^{-K}\right)^q}{1 - q2^{-K}} \right) \norm{x - y}_\infty \\
& \leq 10q \norm{x - y}_\infty.
\end{align*}
The remaining step is to characterize the size of the ReLU network for implementing $\hat{f}$. Construction \eqref{eq:f_hat} suggests that the network consists of $(N+1)^q$ parallel subnetworks. In each subnetwork, we need to implement $\hat{\xi}_m$ defined in \eqref{eq:xi_hat}, where the subnetwork architecture consists of $K$ layers and the width is bounded by a constant (since $\hat{h}_K$ is realizable by a width-$3$ network). Putting together all the parallel subnetworks, we conclude that the whole network architecture consists of $K$ layers and the width is bounded by $O((N+1)^q)$. Substituting our choice of $N$ and $K$ into the network size, we obtain $L = O\left(\log \frac{1}{\epsilon_2} + q \right)$ and $p = O(\epsilon_2^{-q})$. The total number of neurons and nonzero weight parameters in the network is $J = O(Lp)$.

The last step is to ensure that each weight parameter in $\hat{f}$ is bounded by a constant. The only caveat stems from the trapezoid function in $\xi_m$, which is rescaled by $3N$ (see equation \eqref{eq:xi_hat}). We use a deep network to implement $\phi(3N(x_k - \frac{m_k}{N}))$. Consider a basic step function $s(x) = 2{\rm ReLU}(x) - 2{\rm ReLU} (x - 1)$, whose $j$-th order composition is
\begin{align*}
s_j = s \circ \cdots \circ s = 
\begin{cases}
0, & x < 0 \\
2^j x, & x \in [0, 1/2^{j-1}] \\
2, & x > 1/2^j
\end{cases}.
\end{align*}
Setting $j = \left\lceil \log (3N) \right\rceil + 1$, we observe that $s_j$ has a slope of at least $6N$. We use $s_j / 2$ to realize the left linear segment in $\phi(3N(x_k - \frac{m_k}{N}))$. For the right linear segment, we can use $1 - s_j/2$ instead. In this way, we increment the network architecture for implementing $\hat{f}$ by a depth of $\left\lceil \log (3N) \right\rceil + 1 = O\left(\log 1/\epsilon_2 + q \right)$ and a width of $4$, while each weight parameter in the network is bounded by a constant. To summarize the network architecture, we have
\begin{align*}
& L = O\left(\log 1/\epsilon_2 + q \right), \quad p = O\left(\epsilon_2^{-q}\right), \quad J = O\left(\epsilon_2^{-q} (\log 1/\epsilon_2 + q)\right), \\
& \hspace{1.2in} \kappa = O(1), \quad R = \sqrt{q}.
\end{align*}
The bound on $R$ is obtained by combining Lipschitz continuity \eqref{eq:lip_linfty} with $f(0) = 0$.

\end{proof}

\subsection{Proof of Lemma \ref{lemma:subspace_match}}\label{pf:subspace_match}
\begin{proof}[Proof of Lemma \ref{lemma:subspace_match}]
Given the choice of generator and discriminator network classes, we show that at a global optimizer $(U^*, g_\theta^*)$, it holds
\begin{align}\label{eq:dw_upperbound}
\dw\left((U^* \circ g_\theta^*)_\sharp \rho, \mu \right) & \leq \left(1 + 4\sqrt{q} \left(\min_i \EE_{z \sim \rho}\left[T^{\rm ld}_i(z)\right]\right)^{-1} \EE_{z \sim \rho} \left[\norm{g_\theta^*(z)}_2\right]\right) \nonumber \\
& \qquad \cdot (\bar{\gamma} \epsilon_1 + 3\epsilon_2).
\end{align}
Suppose for the purpose of contradiction, we have
\begin{align}\label{eq:contradict_assumption}
\dw\left((U^* \circ g_\theta^*)_\sharp \rho, \mu \right) & > \left(1 + 4\sqrt{q} \left(\min_i \EE_{z \sim \rho}\left[T^{\rm ld}_i(z)\right]\right)^{-1} \EE_{z \sim \rho} \left[\norm{g_\theta^*(z)}_2\right]\right) \nonumber \\
& \qquad \cdot (\bar{\gamma} \epsilon_1 + 3\epsilon_2).
\end{align}
We will prove that there exists $(V, f_\omega)$ such that
\begin{align}\label{eq:equilibrium_lowerbound}
d_{\cF_{\rm NN}^{\rm ld}}\left((U^* \circ g_\theta^*)_\sharp \rho, \mu\right) & \geq \EE_{z \sim \rho} \left[f_\omega (V^\top U^* g_\theta^*(z)) \right] - \EE_{x \sim \mu} [f_\omega (V^\top x)] \nonumber \\
& > \bar{\gamma} \epsilon_1.
\end{align}

On the other hand, by choosing $U^* = A$ and $g_\theta$ with $\norm{T^{\rm ld} - g_\theta}_\infty \leq \epsilon_1/q$, we have
\begin{align}\label{eq:equilibrium_upperbound}
d_{\cF_{\rm NN}^{\rm ld}}\left((U^* \circ g_\theta)_\sharp \rho, \mu\right) \leq \bar{\gamma} \epsilon_1,
\end{align}
since discriminator is $\bar{\gamma}$-Lipschitz with respect to the $L_\infty$ norm. Putting \eqref{eq:equilibrium_lowerbound} and \eqref{eq:equilibrium_upperbound} together, we conclude that $(U^*, g_\theta^*)$ cannot be a global optimizer. Therefore, \eqref{eq:dw_upperbound} holds true. It remains to establish \eqref{eq:equilibrium_lowerbound}. Since the discriminator network can approximate any $1$-Lipschitz function by Lemma \ref{lemma:approx_lip}, it is convenient to show the following sufficient condition for \eqref{eq:equilibrium_lowerbound},
\begin{align}\label{eq:equilibrium_contradict}
\sup_V~ \dw\left((V^\top U^* \circ g_\theta^*)_\sharp \rho, V^\top_\sharp \mu \right) > \bar{\gamma}\epsilon_1 + 3\epsilon_2. 
\end{align}
In fact, \eqref{eq:equilibrium_contradict} implies that for any $\delta \in (0, \epsilon_2)$, there exists a discriminative function $f_0$ and matrix $V_0$ such that $\EE_{z \sim \rho} [f_0(V_0^\top U^* g_\theta^*(z))] - \EE_{x \sim \mu} [f_0(V_0^\top x)] > \bar{\gamma} \epsilon_1 + 3\epsilon_2 + 2d_{\cF_{\rm NN}^{\rm ld}}(\hat{\mu}_n, \mu) - \delta$. By choosing $f_\omega$ as an $\epsilon_2$-approximation of $f_0$ and $V = V_0$, we obtain
\begin{align*}
& \quad \EE_{z \sim \rho} \left[f_\omega (V_0^\top U^* g_\theta^*(z)) \right] - \EE_{x \sim \mu} [f_\omega (V_0^\top x)] \\
& = \EE_{z \sim \rho} \left[f_\omega (V_0^\top U^* g_\theta^*(z)) \right] \nonumber \\
& \quad - \EE_{x \sim \mu} [f_\omega (V_0^\top x)] - \EE_{z \sim \rho} [f_0(V_0^\top U^* g_\theta^*(z))] - \EE_{x \sim \mu} [f_0(V_0^\top x)] \\
& \quad + \EE_{z \sim \rho} [f_0(V_0^\top U^* g_\theta^*(z))] - \EE_{x \sim \mu} [f_0(V_0^\top x)] \\
& > \bar{\gamma}\epsilon_1 + 3\epsilon_2 - \delta - 2 \norm{f_\omega - f_0}_\infty \\
& > \bar{\gamma}\epsilon_1,
\end{align*}
which establishes \eqref{eq:equilibrium_lowerbound}.

To ease the presentation, we recall that $\epsilon = \bar{\gamma} \epsilon_1 + 3 \epsilon_2$. We now consider two complementary cases for establishing \eqref{eq:equilibrium_contradict},
\begin{itemize}
\item {\bf (Case 1)} $\frac{1}{q} \left\vert \tr\left(A^\top U^* \right) \right\vert < 1 - 2 \left(\min_i \EE_{z \sim \rho}\left[T^{\rm ld}_i(z)\right]\right)^{-2} \epsilon^2$; 
\item {\bf (Case 2)} $\frac{1}{q} \left\vert \tr\left(A^\top U^* \right) \right\vert \geq 1 - 2 \left(\min_i \EE_{z \sim \rho}\left[T^{\rm ld}_i(z)\right]\right)^{-2} \epsilon^2$,
\end{itemize}
where $T_i^{\rm ld}$ denotes the $i$-th coordinate mapping. Note that {\bf (Case 2)} says that the column spaces of $A, U^{\star}$ are nearly identical. We tackle the two cases separately. To further ease the analysis, we assume without loss of generality that $a_{i}^\top u^\star_{i} \geq 0$ for $i = 1, \dots, q$, where $a_{i}$ and $u^\star_{i}$ are column vectors of $A$ and $U^{\star}$, respectively. Otherwise we can replace $a_i$ with $-a_i$ and $T^{\rm ld}_i$ with $-T^{\rm ld}_i$ simultaneously. As a result, we may remove the absolute values in {\bf (Case 1)} and {\bf (Case 2)} for simplicity.

\noindent $\bullet$ {\bf (Case 1)} We show that there exists an index $I$ such that the corresponding column vectors $a_{I}$ and $u_{I}^*$ are sufficiently mis-aligned in direction. Specifically, given $\frac{1}{q} \tr\left(A^\top U^{\star} \right) < 1 - 2 \EE^{-2}_{z \sim \rho}\left[\min_i T^{\rm ld}_i(z)\right] \epsilon^2$, we expand the expression as
\begin{align*}
\frac{1}{q} \tr\left(A^\top U^{\star} \right) = \frac{1}{q} \sum_{i=1}^q a_{i}^\top u^{\star}_i < 1 - 2 \left(\min_i \EE_{z \sim \rho}\left[T^{\rm ld}_i(z)\right]\right)^{-2} \epsilon^2.
\end{align*}
Since $a_{i}^\top u^{\star}_i \in [0, 1]$ for $i = 1, \dots, q$, by the Pigeonhole principle, we deduce that there exists an index $I$ with
\begin{align}\label{eq:case1_misalign}
a_{I}^\top u^{\star}_I < 1 - 2 \left(\min_i \EE_{z \sim \rho}\left[T^{\rm ld}_i(z)\right]\right)^{-2}\epsilon^2.
\end{align}

Now we prove that the mis-alignment of $a_I$ and $u^*_I$ already results in a sufficient separation between the generated distribution and data distribution, in terms of projected Wasserstein distance. By definition, we have
\begin{align}
& \quad \dw\left(\rbr{V^\top U^* g_\theta^*}_\sharp \rho, V^\top_\sharp \mu \right) \nonumber \\
& = \sup_{f \in {\rm Lip}_1(\RR^q)} \EE_{z \sim \rho} \left[f\left(V^\top U^* g_\theta^*(z)\right)\right] - \EE_{z \sim \rho} \left[f\left(V^\top A T^{\rm ld}(z)\right)\right] \nonumber \\
& = \sup_{f \in {\rm Lip}_1(\RR^q)} \EE_{z \sim \rho} \left[f\left(V^\top \sum_{i=1}^q u^*_{i} (g_\theta^*)_i(z)\right)\right] - \EE_{z \sim \rho} \left[f\left(V^\top \sum_{i=1}^q a_i T^{\rm ld}_i(z)\right)\right]. \label{eq:multiD_case1_step1}
\end{align}
We choose the projection matrix $V$ to be a rank-$1$ matrix with only the $I$-th column nonzero, i.e.,
\begin{align*}
V = \begin{bmatrix}
\boldsymbol{0}_{d \times (I-1)},~~ & 
\frac{a_{I} - u^{\star}_I}{\norm{a_{I} - u^{\star}_I}_2},~~ & 
\boldsymbol{0}_{d \times (q-I)}
\end{bmatrix}.
\end{align*}
We further choose a specific testing function $f$ to derive a lower bound on \eqref{eq:multiD_case1_step1}. Let $f(x) = w^\top x$ be linear with $w_I = 1$ and $w_i = 0$ for $i \neq I$. Substituting our choice of $V$ and $f$ into \eqref{eq:multiD_case1_step1}, we obtain
\begin{align}
& \quad \dw\left(\rbr{V^\top U^* g_\theta^*}_\sharp \rho, V^\top_\sharp \mu \right) \nonumber \\
& \geq \EE_{z \sim \rho} \left[w^\top V^\top \sum_{i=1}^q u_i^* (g^*_\theta)_{i}(z)\right] - \EE_{z \sim \rho} \left[w^\top V^\top \sum_{i=1}^q a_{i} T^{\rm ld}_{i}(z)\right] \nonumber \\
& = \frac{1 - a_{I}^\top u^*_{I}}{\norm{a_{I} - u^*_{I}}_2}~\EE_{z \sim \rho} \left[(g^*_{\theta})_I(z) + T^{\rm ld}_I(z)\right] \nonumber \\
& = \frac{1}{2} \norm{a_{I} - u^*_{I}}_2 \EE_{z \sim \rho} \left[(g^*_{\theta})_I(z) + T^{\rm ld}_I(z)\right]. \label{eq:1Dreduction}
\end{align}
Using \eqref{eq:case1_misalign}, we lower bound 
\begin{align*}
\norm{a_{I} - u^*_{I}}_2 = \sqrt{2 - 2a_{I}^\top u^*_{I}} > 2 \left(\min_i \EE_{z \sim \rho}\left[T^{\rm ld}_i(z)\right]\right)^{-1}\epsilon.
\end{align*}
Substituting into \eqref{eq:1Dreduction}, we conclude
\begin{align*}
\dw\left(\rbr{V^\top U^* g_\theta^*}_\sharp \rho, V^\top_\sharp \mu \right) & > \epsilon \cdot \left(\min_i \EE_{z \sim \rho}\left[T^{\rm ld}_i(z)\right]\right)^{-1} \EE_{z \sim \rho} [(g^*_{\theta})_I(z) + T^{\rm ld}_I(z)] \\
& > \epsilon.
\end{align*}

\noindent $\bullet$ {\bf (Case 2)} The assertion of {\bf (Case 2)} translates to several useful spectral norm bounds. We first observe
\begin{align}\label{eq:case2_AU_spectral}
\norm{A - U^*}_2^2 \leq \norm{A - U^*}_{\rm F}^2 & = \tr \left((A - U^*)^\top (A - U^*)\right) \nonumber \\
& = \tr \left(2I - A^\top U^* - (U^*)^\top A \right) \nonumber \\
& \leq 4q \left(\min_i \EE_{z \sim \rho}\left[T^{\rm ld}_i(z)\right]\right)^{-2} \epsilon^2.
\end{align}
Taking square root on both sides of \eqref{eq:case2_AU_spectral}, we have $\norm{A - U^*}_2 \leq 2\sqrt{q} \left(\min_i \EE_{z \sim \rho}\left[T^{\rm ld}_i(z)\right]\right)^{-1} \epsilon$. In addition, since $A$ has orthonormal columns, we have
\begin{align}\label{eq:case2_IAU_spectral}
\norm{I - A_0^\top A_\star}_2 = \norm{A_0^\top (A_0 - A_\star)}_2 & \leq \norm{A_0}_2 \norm{A_0 - A_\star}_2 \nonumber \\
& \leq 2\sqrt{q} \left(\min_i \EE_{z \sim \rho}\left[T^{\rm ld}_i(z)\right]\right)^{-1} \epsilon.
\end{align}

We use a similar proof strategy as in {\bf (Case 1)} by choosing a specific projection matrix $V = A$, and evaluate the Wasserstein distance
\begin{align}
& \quad \dw\left(\rbr{V^\top U^* g_\theta^*}_\sharp \rho, V^\top_\sharp \mu \right) \nonumber \\ 
& = \sup_{f \in {\rm Lip}_1(\RR^q)} \EE_{z \sim \rho} \left[f(T^{\rm ld}(z))\right] - \EE_{z \sim \rho} \left[f(A^\top U^* g_\theta^*(z))\right] \nonumber \\
& = \sup_{f \in {\rm Lip}_1(\RR^q)} \EE_{z \sim \rho} \left[f(T^{\rm ld}(z))\right] - \EE_{z \sim \rho} \left[f(g_\theta^*(z))\right] \nonumber \\
& \quad + \EE_{z \sim \rho} \left[f(g_\theta^*(z))\right] - \EE_{z \sim \rho} \left[f(A^\top U^* g_\theta^*(z))\right] \nonumber \\
& \geq \sup_{f \in {\rm Lip}_1(\RR^q)} \EE_{z \sim \rho} \left[f(T^{\rm ld}(z))\right] - \EE_{z \sim \rho} \left[f(g_\theta^*(z))\right] - \EE_{z \sim \rho} \left[\norm{(I - A^\top U^*) g_\theta^*(z)}_2\right] \nonumber \\
& \geq \sup_{f \in {\rm Lip}_1(\RR^q)} \EE_{z \sim \rho} \left[f(T^{\rm ld}(z))\right] - \EE_{z \sim \rho} \left[f(g_\theta^*(z))\right] - \norm{I - A^\top U^*}_2 \EE_{z \sim \rho} \left[\norm{g_\theta^*(z)}_2\right] \nonumber \\
& = \underbrace{\dw(T^{\rm ld}_\sharp \rho, (g_\theta^*)_\sharp \rho)}_{(\spadesuit)} - \underbrace{\norm{I - A^\top U^*}_2 \EE_{z \sim \rho} \left[\norm{g_\theta^*(z)}_2\right]}_{(\clubsuit)}. \label{eq:case2_bound}
\end{align}
Invoking inequality \eqref{eq:case2_IAU_spectral}, $(\clubsuit)$ assumes the upper bound
\begin{align}
(\clubsuit) \leq 2 \sqrt{q}\left(\min_i \EE_{z \sim \rho}\left[T^{\rm ld}_i(z)\right]\right)^{-1} \EE\left[\norm{g_\theta^*(z)}_2\right] \epsilon. \nonumber
\end{align}
To lower bound $(\spadesuit)$, we prove a lower bound on $\dw\left((A T^{\rm ld})_\sharp \rho, (U^* g_{\theta}^*)_\sharp \rho \right)$. The triangle inequality implies
\begin{align*}
\dw\left((A T^{\rm ld})_\sharp \rho, (U^* g_{\theta}^*)_\sharp \rho \right) \leq \dw\left((A T^{\rm ld})_\sharp \rho, (A g_\theta^*)_\sharp \rho \right) + \dw\left((A g_\theta^*)_\sharp \rho, (U^* g_{\theta}^*)_\sharp \rho \right).
\end{align*}
We bound the second term in the right-hand side above as
\begin{align}
\dw\left((A g_\theta^*)_\sharp \rho, (U^* g_\theta^*)_\sharp \rho \right) & = \sup_{f \in {\rm Lip}_1(\RR^d)} \EE_{z \sim \rho} \left[f\left(A g_\theta^*(z)\right)\right] - \EE_{z \sim \rho} [f(U^* g_\theta^*(z))] \nonumber \\
& \overset{(i)}{\leq} \EE_{z \sim \rho} \left[\norm{A g_\theta^*(z) - U^* g_\theta^*(z)}_2\right] \nonumber \\
& \leq \norm{A - U^*}_2 \EE_{z \sim \rho} \left[\norm{g_\theta^*(z)}_2\right] \nonumber \\
& \overset{(ii)}{\leq} 2\sqrt{q} \left(\min_i \EE_{z \sim \rho}\left[T^{\rm ld}_i(z)\right]\right)^{-1} \EE_{z \sim \rho} \left[\norm{g_\theta^*(z)}_2\right] \epsilon, \nonumber
\end{align}
where inequality $(i)$ invokes the Lipschitz continuity of $f$ and inequality $(ii)$ invokes \eqref{eq:case2_AU_spectral}. Recall that in \eqref{eq:contradict_assumption}, we assume
$$\dw\left((A T^{\rm ld})_\sharp \rho, (U^* g_{\theta}^*)_\sharp \rho \right) > \left(1 + 4\sqrt{q} \left(\min_i \EE_{z \sim \rho}\left[T^{\rm ld}_i(z)\right]\right)^{-1} \EE_{z \sim \rho} \left[\norm{g_\theta^*(z)}_2\right]\right) \epsilon.$$
Thus, we have
\begin{align}
\dw\left((U^* T^{\rm ld})_\sharp \rho, (U^* g_{\theta}^*)_\sharp \rho \right) & > \left(1 + 2\sqrt{q} \left(\min_i \EE_{z \sim \rho}\left[T^{\rm ld}_i(z)\right]\right)^{-1} \EE_{z \sim \rho} \left[\norm{g_\theta^*(z)}_2\right]\right) \epsilon, \nonumber
\end{align}
which implies
\begin{align}
(\spadesuit) = \dw(T^{\rm ld}_\sharp \rho, (g_\theta^*)_\sharp \rho) & > \left(1 + 2\sqrt{q} \left(\min_i \EE_{z \sim \rho}\left[T^{\rm ld}_i(z)\right]\right)^{-1} \EE_{z \sim \rho} \left[\norm{g_\theta^*(z)}_2\right]\right) \epsilon. \nonumber
\end{align}
Combining the bounds of $(\spadesuit)$ and $(\clubsuit)$ and substituting into \eqref{eq:case2_bound}, we obtain
\begin{align*}
\dw\left(\rbr{V^\top U^* g_\theta^*}_\sharp \rho, V^\top_\sharp \mu \right) \geq (\spadesuit) - (\clubsuit) > \epsilon,
\end{align*}
which checks {\bf (Case 2)}. Putting {\bf (Case 1)} and {\bf (Case 2)} together, we establish inequality \eqref{eq:equilibrium_contradict}. Consequently, \eqref{eq:equilibrium_lowerbound} holds true and therefore, \eqref{eq:dw_upperbound} is valid for a global optimizer $(U^*, g_\theta^*)$.

Next, we show given \eqref{eq:dw_upperbound}, the column space of $A$ and $U^*$ are approximately equal. In particular, we show the following bound
\begin{align*}
\frac{1}{2q} \norm{A- U^*}^2_{\rm F} & \leq 2 \cdot \left(1 + 4\sqrt{q} \left(\min_i \EE_{z \sim \rho}\left[T^{\rm ld}_i(z)\right]\right)^{-1} \EE_{z \sim \rho} \left[\norm{g_\theta^*(z)}_2\right]\right)^2 \\
& \qquad \cdot \left(\min_i \EE_{z \sim \rho} \left[T^{\rm ld}_i(z)\right]\right)^{-2} \epsilon^2.
\end{align*}
Suppose not. We expand the squared Frobenius norm $\norm{A - U^*}_{\rm F}^2$ as
\begin{align*}
\frac{1}{2q} \norm{A - U^*}_{\rm F}^2 & = \frac{1}{2q} \tr\left((A - U^*)^\top (A - U^*)\right) \\
& = \frac{1}{2q} \tr\left(2I - A^\top U^* - (U^*)^\top A\right) \\
& = 1 - \frac{1}{2q} \tr\left(A^\top U^* + (U^*)^\top A\right) \\
& = 1 - \frac{1}{q} \tr\left(A^\top U^*\right).
\end{align*}
From the last display above, we deduce 
\begin{align*}
\frac{1}{q} \tr\left(A^\top U^*\right) & < 1 - 2 \cdot \left(1 + 4\sqrt{q} \left(\min_i \EE_{z \sim \rho}\left[T^{\rm ld}_i(z)\right]\right)^{-1} \EE_{z \sim \rho} \left[\norm{g_\theta^*(z)}_2\right]\right)^2 \\
& \qquad \cdot \left(\min_i \EE_{z \sim \rho} \left[T^{\rm ld}_i(z)\right]\right)^{-2} \epsilon^2.
\end{align*}
We consider distinguish $(U^* \circ g_\theta^*)_\sharp \rho$ and $\mu$ by a linear testing function $f(x) = \frac{(a_{I} - u^*_{I})^\top}{\norm{a_{I} - u^*_{I}}_2}x$, where the index $I$ verifies
\begin{align*}
a_{I}^\top u^*_I & < 1 - 2 \cdot \left(1 + 4\sqrt{q} \left(\min_i \EE_{z \sim \rho}\left[T^{\rm ld}_i(z)\right]\right)^{-1} \EE_{z \sim \rho} \left[\norm{g_\theta^*(z)}_2\right]\right)^2 \\
& \qquad \cdot \left(\min_i \EE_{z \sim \rho} \left[T^{\rm ld}_i(z)\right]\right)^{-2} \epsilon^2.
\end{align*}
Repeating the same argument in {\bf (Case 1)}, we deduce
\begin{align*}
\dw\left((U^* \circ g^*_\theta)_\sharp \rho, \mu\right) > \left(1 + 4\sqrt{q} \left(\min_i \EE_{z \sim \rho}\left[T^{\rm ld}_i(z)\right]\right)^{-1} \EE_{z \sim \rho} \left[\norm{g_\theta^*(z)}_2\right]\right) \epsilon,
\end{align*}
which contradicts \eqref{eq:equilibrium_upperbound}. The proof is complete. 
\end{proof}

\subsection{Proof of Lemma \ref{lemma:staterror_lowd}}\label{pf:staterror_lowd}
\begin{proof}[Proof of Lemma \ref{lemma:staterror_lowd}]
We bound $\dw(\hat{\mu}_n, \mu)$ first. Denote $\nu = A_\sharp^\top \mu$ and $\hat{\nu}_n = A_\sharp^\top \hat{\mu}_n$. By Assumption \ref{assumption:lowdmanifold}, we write $\dw(\hat{\mu}_n, \mu)$ as
\begin{align}
\dw(\hat{\mu}_n, \mu) & = \dw(A_\sharp \hat{\nu}_n, A_\sharp \nu) \nonumber \\
& = \sup_{\norm{f}_{\rm Lip} \leq 1} \EE_{x \sim A_\sharp \hat{\nu}_n} [f(x)] - \EE_{y \sim A_\sharp \nu} [f(y)] \nonumber \\
& = \sup_{\norm{f}_{\rm Lip} \leq 1} \EE_{x \sim \hat{\nu}_n} [f(Ax)] - \EE_{y \sim \nu} [f(Ay)] \nonumber \\
& \overset{(i)}{=} \sup_{g = f \circ A} \EE_{x \sim \hat{\nu}_n} [g(x)] - \EE_{y \sim \nu} [g(y)] \nonumber \\
& \leq \dw(\hat{\nu}_n, \nu).
\end{align}
where in $(i)$, the composite function $g = f \circ A: \RR^q \mapsto \RR$ is Lipschitz continuous, whose Lipschitz constant is bounded by $1$. Applying Lemma \ref{lemma:stat}, with the function class being $1$-Lipschitz functions on $[0, 1]^q$, we derive
\begin{align}
\dw(\hat{\nu}_n, \nu) & \leq 4 \inf_{\delta \in (0, \sqrt{q})} \left(\delta + \frac{6}{\sqrt{n}} \int_\delta^{\sqrt{q}} \sqrt{\log \cN(\tau, \cH^1([0, 1]^q), \norm{\cdot}_\infty)} d\tau \right) \nonumber \\
& \overset{(i)}{\leq} 4 \inf_{\delta} \left(\delta + \frac{6}{\sqrt{n}} \int_{\delta}^{\sqrt{q}} \tau^{-q/2} d\tau \right)\nonumber \\
& \overset{(ii)}{\leq} O\left(n^{-1/q}\log n\right),
\end{align}
where in $(i)$, we substitute a covering number bound $\log \cN(\tau, \cH^1([0, 1]^q), \norm{\cdot}_\infty) = O\left((1/\tau)^{q}\right)$, and in $(ii)$, we take $\delta = n^{-1/q}$ and distinguish two cases depending on $q$:
\begin{itemize}
\item ($q = 2$). Inequality $(i)$ can be simplified as
\begin{align*}
\dw(\hat{\nu}_n, \nu) & \leq \frac{4}{\sqrt{n}} + \frac{24}{\sqrt{n}} \log (\sqrt{qn}) \\
& = O \left(n^{-1/q} \log n\right).
\end{align*}
\item ($q > 2$). Inequality $(i)$ can be computed as
\begin{align*}
\dw(\hat{\nu}_n, \nu) & \leq 4n^{-1/q} + \frac{24}{\sqrt{n}} \frac{1}{1-q/2} \left((\sqrt{q})^{-q/2 + 1} - \left(n^{-1/q}\right)^{-q/2 + 1} \right) \\
& = O\left(n^{-1/q} + n^{-1/2} \right).  
\end{align*}
\end{itemize}
Applying Lemma \ref{lemma:stat} again, with the function class being $\cF_{\rm NN}^{\rm ld}$, we further have
\begin{align}
d_{\cF_{\rm NN}^{\rm ld}}(\hat{\mu}_n, \mu) & \leq 4 \inf_{\delta \in (0, \sqrt{q})} \left(\delta + \frac{6}{\sqrt{n}} \int_\delta^{\sqrt{q}} \sqrt{\log \cN(\tau, \cF_{\rm NN}^{\rm ld}, \norm{\cdot}_\infty)} d\tau \right) \nonumber \\
& \overset{(i)}{\leq} 4 \inf_{\delta} \left(\delta + \frac{6}{\sqrt{n}} \int_{\delta}^{\sqrt{q}} \sqrt{\bar{K} \log \frac{2\bar{L}^2 (\bar{p} + 2) (\bar{\kappa} \bar{p})^{\bar{L}+1}}{\tau}} d\tau \right)\nonumber \\
& \overset{(ii)}{=} O\left(\frac{1}{n} + \frac{1}{\sqrt{n}} \sqrt{\bar{K}\bar{L} \log (\bar{L} \bar{p} n)} \right),
\end{align}
where in $(i)$, we invoke Lemma \ref{lemma:NNcovering} instantiated to $\cF_{\rm NN}^{\rm ld}$, and in $(ii)$, we set $\delta = \frac{1}{n}$.

\end{proof}

\section{Proof of Lemma \ref{lemma:hatf_lip}}\label{pf:hatf_lip}
\begin{proof}[Proof of Lemma \ref{lemma:hatf_lip}]
We begin by considering two points $x, y \in [0, 1]^q$ differing in only one coordinate. Without loss of generality, we assume $x_1 - y_1 \geq 0$, while $x_j - y_j = 0$ for $j = 2, \dots, q$. We have two base cases:
\begin{itemize}
\item {\bf (Base case 1)} there exists $m_1^* \in \{0, \dots, N\}$ such that $x_1, y_1 \in \left[\frac{3m^*_1 -1}{3N}, \frac{3m^*_1 + 1}{3N} \right]$;
\smallskip
\item {\bf (Base case 2)} there exists $m^*_1 \in \{0, \dots, N\}$ such that $x_1, y_1 \in \left[\frac{3m^*_1 - 2}{3N}, \frac{3m^*_1 - 1}{3N} \right]$.
\end{itemize}
In both base cases, $x_1$ and $y_1$ are close enough within distance $2/3N$. Later, we will reduce general positions of $x_1, y_1 \in [0, 1]$ to a collection of base bases. A graphical illustration of base cases are given in Figure \ref{fig:basecase}.
\begin{figure}[!htb]
\centering
\includegraphics[width = 0.8\textwidth]{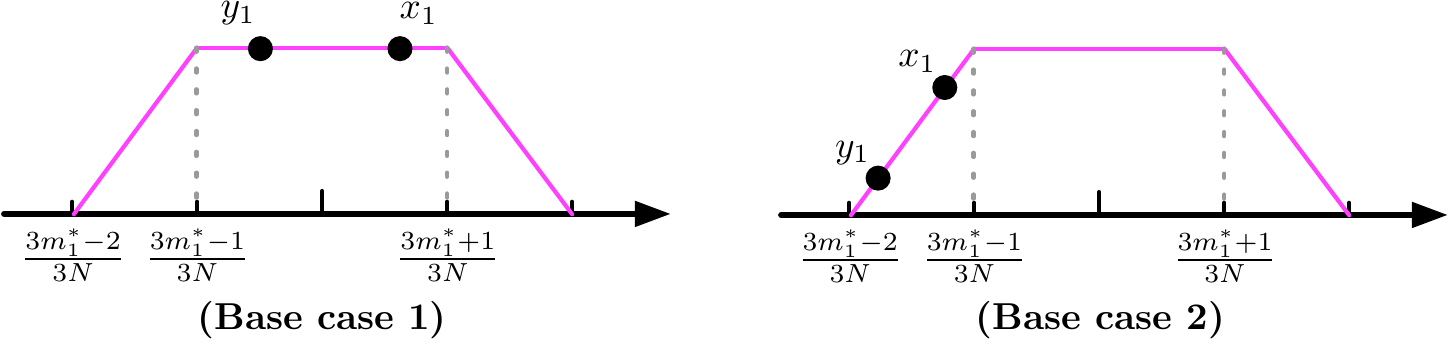}
\caption{Illustration of {\bf (Base case 1)} and {\bf (Base case 2)}.}
\label{fig:basecase}
\end{figure}

In {\bf (Base case 1)}, we have $\phi(3N(x_1 - m^*_1/N)) = \phi(3N(y_1 - m^*_1/N)) = 1$ and $\phi(3N(x_1 - m_1/N)) = \phi(3N(y_1 - m_1/N)) = 0$ for any $m_1 \neq m_1^*$. Therefore, the equality $\hat{\xi}_m(x) = \hat{\xi}_m(y)$ holds true for any $m \in \{0, \dots, N\}^q$. Consequently, we deduce $\hat{f}(x) - \hat{f}(y) = 0$.

In {\bf (Base case 2)}, the analysis is more complicated. We first observe that $\phi(3N(x_1 - m^*_1/N)) = 3Nx_1 - 3m^*_1 + 2$ and $\phi(3N(x_1 - (m^*_1-1)/N)) = -3Nx_1 + 3m^*_1 - 1$ are both nonzero, while $\phi(3N(x_1 - m_1/N)) = 0$ for any $m_1 \not\in \{m_1^*-1, m_1^*\}$ (the same holds for $y_1$). Denote $m_{\backslash 1} = [m_2, \dots, m_q]^\top$ as all the entries in $m$ except the first entry $m_1$. We rewrite $\hat{f}(x)$ as
\begin{align*}
\hat{f}(x) & = \sum_m \hat{\xi}_m(x) P_m(x) \\
& = \sum_{m = [m^*_1, m_{\backslash 1}^\top]^\top} \hat{\xi}_{m}(x) P_m(x) + \sum_{m = [m^*_1-1, m_{\backslash 1}^\top]^\top} \hat{\xi}_{m}(x) P_m(x).
\end{align*}
The second equality above holds, since $\hat{\xi}_m(x) = 0$ whenever $m_1 \not\in \{m_1^*-1, m_1^*\}$.
Furthermore, we have
\begin{align}\label{eq:xy_diff}
\left\vert \hat{f}(x) - \hat{f}(y)\right\vert & = \bigg\vert \sum_{m_{\backslash 1} : m = [m^*_1, m_{\backslash 1}^\top]^\top} \left(\hat{\xi}_{m}(x) - \hat{\xi}_m(y) \right) f(m) \nonumber \\
& \quad + \sum_{m_{\backslash 1} : m = [m^*_1-1, m_{\backslash 1}^\top]^\top} \left(\hat{\xi}_{m}(x) - \hat{\xi}_m(y) \right) f(m) \bigg\vert.
\end{align}
In order to bound the right-hand side of \eqref{eq:xy_diff}, we establish several regularity properties of $\hat{\xi}_m$ based on Lemma \ref{lemma:slope_hK}. The first result proves the monotonicity of $\hat{\xi}_m$.
\begin{lemma}\label{lemma:monotone_xi}
Let $\hat{\xi}_m$ be defined in \eqref{eq:xi_hat}. Consider two points $x = [x_1, \dots, x_i, \dots, x_q]^\top$ and $x' = [x_1, \dots, x_i', \dots, x_q]^\top$ only differing in the $i$-th coordinate. Denote $m_i^*$ satisfying $x_i, x'_i \in \left[\frac{3m^*_i -2}{3N}, \frac{3m^*_i - 1}{3N}\right]$. Then it holds
\begin{align*}
\left(\hat{\xi}_m(x) - \hat{\xi}_m(x')\right) (x_i - x_i') & \geq 0 \quad \textrm{for} \quad m = [m_1, \dots, m^*_i, \dots, m_q]^\top \quad \textrm{and} \\
\left(\hat{\xi}_m(x) - \hat{\xi}_m(x')\right) (x_i - x_i') & \leq 0 \quad \textrm{for} \quad m = [m_1, \dots, m^*_i-1, \dots, m_q]^\top.
\end{align*}
\end{lemma}
The proof is deferred to Appendix \ref{pf:monotone_xi}. Next, we show first-order continuity of $\hat{\xi}_m$.
\begin{lemma}\label{lemma:lip_xi}
Let $\hat{\xi}_m$ be defined in \eqref{eq:xi_hat}. Consider two points $x = [x_1, \dots, x_i, \dots, x_q]^\top$ and $x' = [x_1, \dots, x_i', \dots, x_q]^\top$ only differing in the $i$-th coordinate. Then for any $m$, it holds
\begin{align*}
& 3N \prod_{j\neq i} \max \left\{\phi(3N(x_j - m_j/N)) - \frac{1}{2^K}, 0\right\} \vert x_i - x_i'\vert \leq \left\vert \hat{\xi}_m(x) - \hat{\xi}_m(x')\right\vert \\
& \hspace{1.7in} \leq 3N \prod_{j \neq i} \left(\phi(3N(x_j - m_j/N)) + \frac{1}{2^{K}}\right) \vert x_i - x_i'\vert.
\end{align*}
\end{lemma}
The proof is deferred to Appendix \ref{pf:lip_xi}. Using Lemma \ref{lemma:monotone_xi} and \ref{lemma:lip_xi}, we are able to bound the right-hand side of \eqref{eq:xy_diff}. We partition all the values of $m_{\backslash 1}$ into two complementary disjoint sets:
\begin{align*}
\cA_{\leq 0} & = \left\{m_{\backslash 1} : f\left(\frac{[m^*_1, m_{\backslash 1}^\top]^\top}{N}\right) f\left(\frac{[m^*_1 -1, m_{\backslash 1}^\top]^\top}{N}\right) \leq 0 \right\} \quad \text{and} \\
\cA_{> 0} & = \left\{m_{\backslash 1} : f\left(\frac{[m^*_1, m_{\backslash 1}^\top]^\top}{N}\right) f\left(\frac{[m^*_1 -1, m_{\backslash 1}^\top]^\top}{N}\right) > 0 \right\}.
\end{align*}
In $\cA_{\leq 0}$, by the Lipschitz continuity of $f$, we have $$\left\vert f\left(\frac{[m^*_1, m_{\backslash 1}^\top]^\top}{N}\right) - f\left(\frac{[m^*_1-1, m_{\backslash 1}^\top]^\top}{N}\right)\right\vert \leq 1/N.$$ If either $\left\vert f\left([m^*_1, m_{\backslash 1}^\top]^\top/N\right)\right\vert > \frac{1}{N}$ or $\left\vert f\left([m^*_1-1, m_{\backslash 1}^\top]^\top/N\right)\right\vert > \frac{1}{N}$, then $f\big([m^*_1, m_{\backslash 1}^\top]^\top/N\big)$ and $f\big([m^*_1, m_{\backslash 1}^\top]^\top/N\big)$ should be both positive or negative. Their product must be positive. As a result, we deduce that in $\cA_{\leq 0}$, 
\begin{align*}
\left\vert f\left(\frac{[m^*_1, m_{\backslash 1}^\top]^\top}{N}\right)\right\vert \leq \frac{1}{N} \quad \text{and} \quad \left\vert f\left(\frac{[m^*_1-1, m_{\backslash 1}^\top]^\top}{N}\right)\right\vert \leq \frac{1}{N}
\end{align*}
hold simultaneously. 

In $\cA_{> 0}$, $f\left([m^*_1, m_{\backslash 1}^\top]^\top/N\right)$ and $f\left([m^*_1, m_{\backslash 1}^\top]^\top/N\right)$ are both positive or negative. We rewrite \eqref{eq:xy_diff} according to the partition of $\cA_{\leq 0}$ and $\cA_{> 0}$ on $m_{\backslash 1}$:
\begin{align}\label{eq:xy_diff_AB}
\left\vert \hat{f}(x) - \hat{f}(y)\right\vert = (\spadesuit) + (\clubsuit),
\end{align}
where
\begin{align}
& (\spadesuit) = \bigg\vert \sum_{m_{\backslash 1} \in \cA_{\leq 0}: m = [m^*_1, m_{\backslash 1}^\top]^\top} \left(\hat{\xi}_m(x) - \hat{\xi}_m(y)\right)f(m/N) \nonumber \\
& \hspace{0.7in} + \sum_{m_{\backslash 1} \in \cA_{\leq 0}: m = [m^*_1-1, m_{\backslash 1}^\top]^\top} \left(\hat{\xi}_m(x) - \hat{\xi}_m(y)\right)f(m/N) \bigg\vert, \nonumber \\
& (\clubsuit) = \bigg\vert \sum_{m_{\backslash 1} \in \cA_{> 0}: m = [m^*_1, m_{\backslash 1}^\top]^\top} \left(\hat{\xi}_m(x) - \hat{\xi}_m(y)\right)f(m/N) \nonumber \\
& \hspace{0.7in} + \sum_{m_{\backslash 1} \in \cA_{> 0}: m = [m^*_1-1, m_{\backslash 1}^\top]^\top} \left(\hat{\xi}_m(x) - \hat{\xi}_m(y)\right)f(m/N) \bigg\vert. \nonumber
\end{align}
For term $(\spadesuit)$, we bound it by
\begin{align}\label{eq:A_bound}
& \quad (\spadesuit) \nonumber \\
& \leq \left\vert \sum_{m_{\backslash 1} \in \cA_{\leq 0}: m = [m^*_1, m_{\backslash 1}^\top]^\top} \left(\hat{\xi}_m(x) - \hat{\xi}_m(y)\right)\frac{1}{N}\right\vert \nonumber \\
& \quad + \left\vert \sum_{m_{\backslash 1} \in \cA_{\leq 0}: m = [m^*_1-1, m_{\backslash 1}^\top]^\top} \left(\hat{\xi}_m(x) - \hat{\xi}_m(y)\right)\frac{1}{N} \right\vert \nonumber \\
& \overset{(i)}{\leq} \frac{1}{N} \sum_{m_{\backslash 1}: \hat{\xi}_m(x) \neq 0, \hat{\xi}_m(y) \neq 0} 6N \vert x_1 - y_1\vert \prod_{k\geq 2} \min \left\{\phi(3N(x_k - m_k/N)) + \frac{1}{2^{K}}, 1\right\} \nonumber \\
& \leq 6\vert x_1 - y_1\vert \sum_{m_{\backslash 1}: \hat{\xi}_m(x) \neq 0, \hat{\xi}_m(y) \neq 0} \prod_{k \geq 2} \left(\phi(3N(x_k - m_k/N)) + \frac{1}{2^{K}}\right) \nonumber \\
& \overset{(ii)}{\leq} 6\vert x_1 - y_1\vert \sum_{m_{\backslash 1}: \hat{\xi}_m(x) \neq 0, \hat{\xi}_m(y) \neq 0} \left[\prod_{k \geq 2} \phi(3N(x_k - m_k/N)) + \sum_{j=1}^q 2^{-jK} \binom{q}{j} \right] \nonumber \\
& \overset{(iii)}{\leq} 6 \left(1 + q2^{-K+q-1} \frac{1 - \left(q2^{-K}\right)^q}{1 - q2^{-K}}\right) \vert x_1 - y_1\vert,
\end{align}
where inequality $(i)$ invokes Lemma \ref{lemma:lip_xi} and neglects terms involving $\hat{\xi}_m(x) = \hat{\xi}_m(y) = 0$ and inequality $(ii)$ expands the product $\prod_{k \geq 2} \left(\phi(3N(x_k - m_k/N)) + \frac{1}{2^{K}}\right)$ by noting $\phi(3N(x_k - m_k/N)) \leq 1$. To see inequality $(iii)$, we first observe that there are at most $2^{q-1}$ terms in the summation, due to the definition of $\phi$. Then we bound $\sum_{m_{\backslash 1}:\hat{\xi}_m(x)\neq 0, \hat{\xi}_m(y) \neq 0} \sum_{j=1}^q 2^{-jK} \binom{q}{j}$ as
\begin{align*}
\sum_{m_{\backslash 1}:\hat{\xi}_m(x)\neq 0, \hat{\xi}_m(y) \neq 0} \sum_{j=1}^q 2^{-jK} \binom{q}{j} & \leq \sum_{m_{\backslash 1}:\hat{\xi}_m(x)\neq 0, \hat{\xi}_m(y) \neq 0} \sum_{j=1}^q 2^{-jK} q^j \\
& \leq 2^{q-1} q 2^{-K} \frac{1 - \left(q2^{-K}\right)^q}{1-q2^{-K}} \\
& = q 2^{-K + q - 1} \frac{1 - \left(q2^{-K}\right)^q}{1-q2^{-K}}.
\end{align*}
Meanwhile, $\prod_{k \geq 2} \phi(3N(x_k - m_k/N))$ is indeed a partition of unity on a $(d-1)$-dimensional unit cude. Therefore, we have
\begin{align*}
\sum_{m_{\backslash 1}: \hat{\xi}_m(x) \neq 0, \hat{\xi}_m(y) \neq 0} \prod_{k \geq 2} \phi(3N(x_k - m_k/N)) = 1.
\end{align*}

For term $(\clubsuit)$, we leverage the cancellation in the two summations. We assume without loss of generality, $f\rbr{[m_1^*, m_{\backslash 1}^\top]^\top / N} > 0$ and $f\rbr{[m_1^*-1, m_{\backslash 1}^\top]^\top / N} > 0$ for $m_{\backslash 1} \in \cA_{>0}$. Otherwise, replacing $f$ by $-f$ won't change term $(\clubsuit)$. Therefore, we derive
\begin{align}\label{eq:B_bound}
& \quad (\clubsuit) \nonumber \\
& \overset{(i)}{\leq} \Bigg\vert \sum_{m_{\backslash 1} \in \cA_{> 0}: m = [m^*_1, m_{\backslash 1}^\top]^\top} \left(\hat{\xi}_m(x) - \hat{\xi}_m(y) \right) f(m/N) \nonumber \\
& \quad - \sum_{m_{\backslash 1} \in \cA_{> 0}: m = [m^*_1-1, m_{\backslash 1}^\top]^\top} \left\vert \hat{\xi}_m(x) - \hat{\xi}_m(y) \right\vert f(m/N) \Bigg\vert \nonumber \\
& \overset{(ii)}{\leq} 3N \vert x_1 - y_1 \vert \nonumber \\
& \quad \cdot \sum_{m_{\backslash 1}: \hat{\xi}_m(x) \neq 0, \hat{\xi}_m(y) \neq 0} \Bigg\vert \prod_{k\geq 2}\left(\phi(3N(x_k - m_k/N)) + \frac{1}{2^{K}} \right) f\left([m^*_1, m_{\backslash 1}^\top]^\top/N\right) \nonumber \\
& \hspace{0.5in} - \prod_{k\geq 2}\max \left\{\phi(3N(x_k - m_k/N)) - \frac{1}{2^{K}}, 0 \right\} f\left([m^*_1-1, m_{\backslash 1}^\top]^\top/N\right) \Bigg\vert \nonumber \\
& \overset{(iii)}{\leq} 3N \vert x_1 - y_1\vert \sum_{m_{\backslash 1}: \hat{\xi}_m(x) \neq 0, \hat{\xi}_m(y) \neq 0} \Bigg(\prod_{k\geq 2} \phi(3N(x_k - m_k/N)) \nonumber \\
& \hspace{1in} \cdot \left\vert f([m^*_1, m_{\backslash 1}^\top]^\top/N) - f([m^*_1-1, m_{\backslash 1}^\top]^\top/N) \right\vert \Bigg) \nonumber \\
& \quad + 6N \vert x_1 - y_1\vert \norm{f}_\infty \cdot q 2^{-K+q-1} \frac{1 - \left(q2^{-K}\right)^q}{1 - q2^{-K}} \nonumber \\
& \leq 3\left(1 + 2N \norm{f}_\infty\cdot q 2^{-K+q-1} \frac{1 - \left(q2^{-K}\right)^q}{1 - q2^{-K}}\right) \vert x_1 - y_1\vert, 
\end{align}
where inequality $(i)$ uses the monotonicity of $\hat{\xi}_m$ in Lemma \ref{lemma:monotone_xi}, inequality $(ii)$ invokes Lemma \ref{lemma:lip_xi}, and inequality $(iii)$ follows from the same argument of $(iii)$ in \eqref{eq:A_bound}. Combining \eqref{eq:A_bound}, \eqref{eq:B_bound} and substituting into \eqref{eq:xy_diff_AB}, we obtain 
\begin{align}\label{eq:lip_base}
\left\vert \hat{f}(x) - \hat{f}(y)\right\vert \leq 3\left(3 + 2(N \norm{f}_\infty + 1) \cdot q 2^{-K+q-1} \frac{1 - \left(q2^{-K}\right)^q}{1 - q2^{-K}} \right) \vert x_1 - y_1\vert.
\end{align}

Given two base cases, we proceed to show Lipschitz continuity of $\hat{f}$. We first partition $[0, 1]$ into two types of sub-intervals,
\begin{center}
{\it (Type 1)} $\left[\frac{3k-1}{3N}, \frac{3k+1}{3N}\right] \bigcap \left[0, 1\right]$ $\quad$ and $\quad$ {\it (Type 2)} $\left[\frac{3k+1}{3N}, \frac{3k+2}{3N}\right] \bigcap \left[0, 1\right]$,
\end{center}
where $k \leq N$ is an integer. We observe that on a {\it (Type 1)} sub-interval, {\bf (Base case 1)} applies; while on a {\it (Type 2)} sub-interval, {\bf (Base case 2)} applies. Depending on the location of $x_1$ and $y_1$, we discuss four situations.

\noindent {\bf (Situation 1)}: $x_1$ belongs to a {\it (Type 1)} sub-interval and $y_1$ belongs to a {\it (Type 1)} sub-interval. If the two sub-intervals coincide, we obtain {\bf (Base case 1)}. There is nothing to show. Otherwise, we denote integer $k_x$ such that $x_1 \in \left[\frac{3k_x-1}{3N}, \frac{3k_x+1}{3N}\right]\bigcap \left[0, 1\right]$ and integer $k_y < k_x$ such that $y_1 \in \left[\frac{3k_y-1}{3N}, \frac{3k_y+1}{3N}\right]\bigcap \left[0, 1\right]$. See Figure \ref{fig:situation1} for an illustration.
\begin{figure}[!htb]
\centering
\includegraphics[width = 0.9\textwidth]{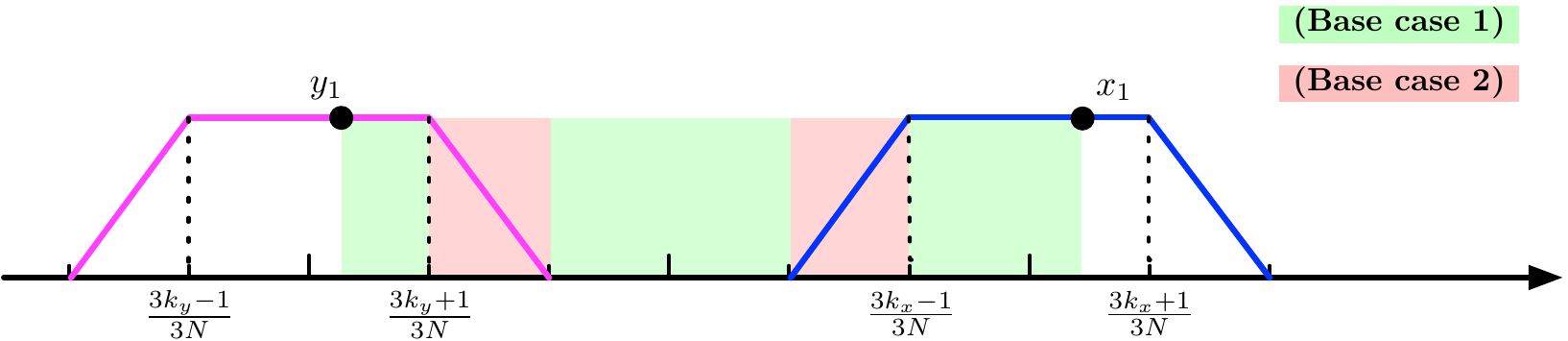}
\caption{Demonstration of {\bf (Situation 1)} with $k_x = k_y + 2$. We can decompose such an situation into a serial of alternating {\bf (Base case 1)} (green) and {\bf (Base case 2)} (red). The function value difference can be obtained by aggregating differences in each base case.}
\label{fig:situation1}
\end{figure}

We can derive
\begin{align*}
\left\vert \hat{f}(x) - \hat{f}(y)\right\vert & \leq \left\vert \hat{f}(x) - \hat{f}\rbr{\left[\frac{3k_x-1}{3N}, x_{\backslash 1}^\top \right]^\top}\right\vert \\
& \quad + \left\vert \hat{f}\rbr{\left[\frac{3k_x-1}{3N}, x_{\backslash 1}^\top \right]^\top} - \hat{f}\rbr{\left[\frac{3k_y+1}{3N}, y_{\backslash 1}^\top \right]^\top}\right\vert \\
& \quad + \left\vert \hat{f}\rbr{\left[\frac{3k_y+1}{3N}, y_{\backslash 1}^\top \right]^\top} - \hat{f}(y) \right\vert \\
& \overset{(i)}{=} \left\vert \hat{f}\rbr{\left[\frac{3k_x-1}{3N}, x_{\backslash 1}^\top \right]^\top} - \hat{f}\rbr{\left[\frac{3k_y+1}{3N}, y_{\backslash 1}^\top \right]^\top}\right\vert,
\end{align*}
where inequality $(i)$ follows from {\bf (Base case 1)}. If $k_x = k_y + 1$, then we can apply {\bf (Base case 2)} to show
\begin{align*}
\left\vert \hat{f}(x) - \hat{f}(y)\right\vert & \leq \left\vert \hat{f}\rbr{\left[\frac{3k_x-1}{3N}, x_{\backslash 1}^\top \right]^\top} - \hat{f}\rbr{\left[\frac{3k_y+1}{3N}, y_{\backslash 1}^\top \right]^\top}\right\vert \\
& \leq 3\left(3 + 2(N \norm{f}_\infty + 1) \cdot q 2^{-K+q-1} \frac{1 - \left(q2^{-K}\right)^q}{1 - q2^{-K}} \right) \frac{1}{3N} \\
& \leq 3\left(3 + 2(N \norm{f}_\infty + 1) \cdot q 2^{-K+q-1} \frac{1 - \left(q2^{-K}\right)^q}{1 - q2^{-K}} \right) \vert x_1 - y_1\vert.
\end{align*}
Otherwise, we have
\begin{align*}
& \quad \left\vert \hat{f}(x) - \hat{f}(y)\right\vert \\
& \leq \left\vert \hat{f}\rbr{\left[\frac{3k_x-1}{3N}, x_{\backslash 1}^\top \right]^\top} - \hat{f}\rbr{\left[\frac{3k_y+1}{3N}, y_{\backslash 1}^\top \right]^\top}\right\vert \\
& \leq \left\vert \hat{f}\rbr{\left[\frac{3k_x-1}{3N}, x_{\backslash 1}^\top \right]^\top} - \hat{f}\rbr{\left[\frac{3(k_x-1)+1}{3N}, x_{\backslash 1}^\top \right]^\top}\right\vert \\
& \quad + \left\vert \hat{f}\rbr{\left[\frac{3(k_x-1)+1}{3N}, x_{\backslash 1}^\top \right]^\top} - \hat{f}\rbr{\left[\frac{3(k_y+1)-1}{3N}, y_{\backslash 1}^\top \right]^\top}\right\vert \\
& \quad + \left\vert \hat{f}\rbr{\left[\frac{3(k_y+1)-1}{3N}, y_{\backslash 1}^\top \right]^\top} - \hat{f}\rbr{\left[\frac{3k_y+1}{3N}, y_{\backslash 1}^\top \right]^\top}\right\vert \\
& \overset{(i)}{\leq} 3\left(3 + 2(N \norm{f}_\infty + 1) \cdot q 2^{-K+q-1} \frac{1 - \left(q2^{-K}\right)^q}{1 - q2^{-K}} \right) \frac{2}{3N} \\
& \quad + \left\vert \hat{f}\rbr{\left[\frac{3(k_x-1)+1}{3N}, x_{\backslash 1}^\top \right]^\top} - \hat{f}\rbr{\left[\frac{3(k_y+1)-1}{3N}, y_{\backslash 1}^\top \right]^\top}\right\vert,
\end{align*}
where inequality $(i)$ is obtained by applying {\bf (Base case 2)} twice. To complete the argument, we can replace $k_x = k_x - 1$ and $k_y = k_y + 1$ and repeat the derivation to accumulate all the differences yielded on a {\it (Type 2)} sub-interval, until $k_x - i = k_y + i + 1$ or $k_x - i = k_y + i$ for some integer $i$. Consequently, noting that the total length of {\it (Type 2)} interval between $x_1$ and $y_1$ is always smaller than $\vert x_1 - y_1\vert$, we deduce
\begin{align*}
\left\vert \hat{f}(x) - \hat{f}(y)\right\vert \leq 3\left(3 + 2(N \norm{f}_\infty + 1) \cdot q 2^{-K+q-1} \frac{1 - \left(q2^{-K}\right)^q}{1 - q2^{-K}} \right) \vert x_1 - y_1\vert.
\end{align*}

\noindent {\bf (Situation 2)}: $x_1$ belongs to a {\it (Type 1)} sub-interval and $y_1$ belongs to a {\it (Type 2)} sub-interval. We aim to reduce this situation to {\bf (Situation 1)}. Following the same notation, we denote $x_1 \in \left[\frac{3k_x-1}{3N}, \frac{3k_x+1}{3N}\right]\bigcap \left[0, 1\right]$ for some integer $k_x$ and $y_1 \in \left[\frac{3k_y+1}{3N}, \frac{3k_y+2}{3N}\right]\bigcap \left[0, 1\right]$ for $k_y < k_x$. Triangle inequality yields
\begin{align*}
& \quad \left\vert \hat{f}(x) - \hat{f}(y)\right\vert \\
& \leq \left\vert \hat{f}(x) - \hat{f}\rbr{\left[\frac{3(k_y+1)-1}{3N}, y_{\backslash 1}^\top \right]^\top}\right\vert + \left\vert \hat{f}\rbr{\left[\frac{3(k_y+1)-1}{3N}, x_{\backslash 1}^\top \right]^\top} - \hat{f}(y)\right\vert \\
& \leq \left\vert \hat{f}(x) - \hat{f}\rbr{\left[\frac{3(k_y+1)-1}{3N}, y_{\backslash 1}^\top \right]^\top}\right\vert \\
& \quad + 3\left(3 + 2(N \norm{f}_\infty + 1) \cdot q 2^{-K+q-1} \frac{1 - \left(q2^{-K}\right)^q}{1 - q2^{-K}} \right) \left\vert \frac{3(k_y+1)-1}{3N} - y_1\right\vert.
\end{align*}
We now observe that term $\left\vert \hat{f}(x) - \hat{f}\rbr{\left[\frac{3(k_y+1)-1}{3N}, y_{\backslash 1}^\top \right]^\top}\right\vert$ falls into {\bf (Situation 1)}. A straightforward adaptation of the argument in {\bf (Situation 1)} gives rise to
\begin{align*}
\left\vert \hat{f}(x) - \hat{f}(y)\right\vert \leq 3\left(3 + 2(N \norm{f}_\infty + 1) \cdot q 2^{-K+q-1} \frac{1 - \left(q2^{-K}\right)^q}{1 - q2^{-K}} \right) \vert x_1 - y_1\vert.
\end{align*}

\noindent {\bf (Situation 3)}: $x_1$ belongs to a {\it (Type 2)} sub-interval and $y_1$ belongs to a {\it (Type 1)} sub-interval. The analysis is analogous to {\bf (Situation 2)} by switching $x_1$ and $y_1$. Denoting $x_1 \in \left[\frac{3k_x+1}{3N}, \frac{3k_x+2}{3N}\right]\bigcap \left[0, 1\right]$ for some integer $k_x$ and $y_1 \in \left[\frac{3k_y-1}{3N}, \frac{3k_y+1}{3N}\right]\bigcap \left[0, 1\right]$ for $k_y \leq k_x$, we derive
\begin{align*}
& \quad \left\vert \hat{f}(x) - \hat{f}(y)\right\vert \\
& \leq \left\vert \hat{f}(x) - \hat{f}\rbr{\left[\frac{3k_x+1}{3N}, x_{\backslash 1}^\top \right]^\top}\right\vert + \left\vert \hat{f}\rbr{\left[\frac{3k_x+1}{3N}, x_{\backslash 1}^\top \right]^\top} - \hat{f}(y)\right\vert \\
& \leq \left\vert \hat{f}\rbr{\left[\frac{3k_x+1}{3N}, x_{\backslash 1}^\top \right]^\top} - \hat{f}(y)\right\vert \\
& \quad + 3\left(3 + 2(N \norm{f}_\infty + 1) \cdot q 2^{-K+q-1} \frac{1 - \left(q2^{-K}\right)^q}{1 - q2^{-K}} \right) \left\vert x_1 - \frac{3k_x+1}{3N}\right\vert.
\end{align*}
Note that $\left\vert \hat{f}\rbr{\left[\frac{3k_x+1}{3N}, x_{\backslash 1}^\top \right]^\top} - \hat{f}(y)\right\vert$ falls into {\bf (Situation 1)}. Therefore, the desired Lipschitz continuity \eqref{eq:lip_base} holds.

\noindent {\bf (Situation 4)}: $x_1$ belongs to a {\it (Type 2)} sub-interval and $y_1$ belongs to a {\it (Type 2)} sub-interval. If the two sub-intervals coincide, this recovers {\bf (Base case 2)}, and there is nothing to show. Otherwise, applying the analysis in {\bf (Situation 2)} and {\bf (Situation 3)} consecutively to move $x_1$ first into a {\it (Type 1)} sub-interval and then $y_1$, we reduce this situation to {\bf (Situation 1)} again. Therefore, Lipschitz continuity in \eqref{eq:lip_base} still holds true. 

Combining all four situations, for any $x, y$ only differing in the first coordinate, it holds
\begin{align*}
\left\vert \hat{f}(x) - \hat{f}(y)\right\vert \leq 3\left(3 + 2(N \norm{f}_\infty + 1) \cdot q 2^{-K+q-1} \frac{1 - \left(q2^{-K}\right)^q}{1 - q2^{-K}} \right) \vert x_1 - y_1\vert.
\end{align*}
The proof is complete for general $x, y$ by aggregating coordinate-wise differences and the fact $\sum_{i=1}^q \vert x_i - y_i\vert = \norm{x - y}_1 \leq q\norm{x - y}_\infty$.
\end{proof}


\subsection{Proofs of supporting results for Lemma \ref{lemma:hatf_lip}}\label{pf:hatf_lip_omitted}
Before we present omitted proofs in Lemma \ref{lemma:hatf_lip}, we study the regularity of the approximated square function $\hat{h}_K$, which will be frequently used in proving Lemma \ref{lemma:monotone_xi} and Lemma \ref{lemma:lip_xi}.

\subsubsection{Regularity of $\hat{h}_K$}
Given a function $g: [0, 1] \mapsto \RR$, for any $x \in (0, 1)$, we define upper and lower slopes at $x$, denoted by $\upperslope_g$ and $\lowerslope_g$, respectively, as
\begin{align*}
\upperslope_{g}(x) & = \upperlim_{\Delta \rightarrow 0}~ \frac{g(x + \Delta) - g(x)}{\Delta}, \\
\lowerslope_{g}(x) & = \lowerlim_{\Delta \rightarrow 0}~ \frac{g(x + \Delta) - g(x)}{\Delta}.
\end{align*}
The definition above coincides with upper and lower derivatives of a univariate function. We use ``slope'' instead of derivatives as we will instantiate the definition to the piecewise linear function $\hat{h}_K$. We show several useful properties. 
\begin{lemma}\label{lemma:slope_hK}
For a given positive integer $K$, let $\hat{h}_K$ be defined on $[0, 1]$ as in \eqref{eq:h_hat}. Then the following identities hold.
\begin{enumerate}
\item For any $x \in (0, 1)$, we have
\begin{align}\label{eq:hK_upperlower_slope}
\upperslope_{\hat{h}_K}(x) & = \lim_{\Delta \rightarrow 0^+} \frac{\hat{h}_K(x + \Delta) - \hat{h}_K(x)}{\Delta} \quad \textrm{and} \\
\lowerslope_{\hat{h}_K}(x) & = \lim_{\Delta \rightarrow 0^+} \frac{\hat{h}_K(x) - \hat{h}_K(x-\Delta)}{\Delta}.
\end{align}
\item Given an integer $i$, we denote $\cB_K(i) = [b_1, \dots, b_K]^\top \in \{0, 1\}^K$ as the $K$-bit binary encoding of $i$, that is, $i = \sum_{k=1}^K b_k 2^{k-1}$. Then for any $x \in (0, 1)$, we have
\begin{align}
& \overline{\slope}_{\hat{h}_K}(x) = 1 + \sum_{k=1}^K \left(2\left[\cB_K\left(\left\lfloor 2^K\cdot x \right\rfloor\right)\right]_{K-k+1}-1\right) 2^{-k} \quad \textrm{and} \label{eq:upper_slope_value} \\
& \underline{\slope}_{\hat{h}_K}(x)  = 1 + \sum_{k=1}^K \left(2\left[\cB_K\left(\left\lceil 2^K\cdot x \right\rceil - 1\right)\right]_{K-k+1}-1\right) 2^{-k}. \label{eq:lower_slope_value}
\end{align}
\end{enumerate}
\end{lemma}
\begin{proof}[Proof of Lemma \ref{lemma:slope_hK}]
By construction, $g_k$ is a piecewise linear function. Each of its linear segment is supported on a sub-interval $[i/2^k, (i+1)/2^k]$ for $i = 1, \dots, 2^k-1$. Therefore, it can be checked that $\hat{h}_K$ is also a piecewise linear function, since it is a linear combination of $g_k$'s. Furthermore, the $i$-th linear segment in $g_k$ has a slope $(-1)^{i} 2^{k}$, i.e., $g_k' = (-1)^{i} 2^{k}$ on open interval $\left(i/2^k, (i+1)/2^k\right)$. As a result, $\hat{h}_K$ is differentiable on $\left(i/2^K, (i+1)/2^K\right)$ and its derivative satisfies
\begin{align}\label{eq:hK_derivative}
\hat{h}_K'(x) & = 1 - \sum_{k=1}^K \frac{1}{2^{2k}} g_k'(x) \nonumber \\
& = 1 - \sum_{k=1}^K (-1)^{\left\lfloor i/2^{K - k}\right\rfloor} \frac{1}{2^{k}} \quad \textrm{for any}~ x \in \left(i/2^K, (i+1)/2^K\right).
\end{align}
We observe $i = \lfloor x \cdot 2^K \rfloor$ for $x \in \left(i/2^K, (i+1)/2^K\right)$, which implies $x \cdot 2^K - 1 < i \leq x \cdot 2^K$. For any $k = 1, \dots, K$, we have
\begin{align*}
\frac{x \cdot 2^K - 1}{2^{K-k}} < i/2^{K - k} \leq \frac{x \cdot 2^K}{2^{K-k}} \quad \Longrightarrow \quad x \cdot 2^k - 1 < i/2^{K - k} \leq x \cdot 2^k.
\end{align*}
Thus, we deduce $\left\lfloor i/2^{K - k}\right\rfloor = \left\lfloor x \cdot 2^k \right\rfloor$ and \eqref{eq:hK_derivative} can be simplified as
\begin{align}\label{eq:hK_derivative_simp}
\hat{h}_K'(x) = 1 - \sum_{k=1}^K (-1)^{\left\lfloor x \cdot 2^k \right\rfloor} \frac{1}{2^{k}}.
\end{align} 
We claim 
\begin{align}\label{eq:claim_binary}
(-1)^{\left\lfloor x \cdot 2^k \right\rfloor} = -2b_{K-k+1} + 1 \quad \textrm{for} \quad k = 1, \dots, K,
\end{align}
where $b_j$ is the $j$-th entry of the $K$-bit binary encoding $\cB_K\left(\left\lfloor x \cdot 2^K \right\rfloor \right)$. In other words, the parity of $\left\lfloor x \cdot 2^{k} \right\rfloor$ is encoded by $b_{K-k+1}$. In particular, if $\left\lfloor x \cdot 2^{k} \right\rfloor$ is odd (resp. even), $b_{K-k+1} = 1$ (resp. $b_{K-k+1} = 0$).

To show the claim in \eqref{eq:claim_binary}, we first prove 
\begin{align}
\left\lfloor x \cdot 2^{k} \right\rfloor = \sum_{j=1}^k 2^{k-j} b_{K-j+1} \nonumber
\end{align}
for any $k = 1, \dots, K$. Indeed, we can show the following sandwich inequality
\begin{align}\label{eq:binary_sandwich}
\left\lfloor\frac{\left\lfloor x \cdot 2^K \right\rfloor}{2^{K-k}}\right\rfloor \overset{(i)}{\leq} \left\lfloor x \cdot 2^{k} \right\rfloor \overset{(ii)}{\leq} \frac{\left\lfloor x \cdot 2^K \right\rfloor}{2^{K-k}}.
\end{align}
Inequality $(i)$ holds, since $\left\lfloor\frac{\left\lfloor x \cdot 2^K \right\rfloor}{2^{K-k}}\right\rfloor \leq \left\lfloor\frac{x \cdot 2^K}{2^{K-k}}\right\rfloor = \left\lfloor x \cdot 2^{k} \right\rfloor$; inequality $(ii)$ holds, since $2^{K - k} \left\lfloor x \cdot 2^{k} \right\rfloor = \left\lfloor 2^{K - k} \left\lfloor x \cdot 2^{k} \right\rfloor \right\rfloor \leq \left\lfloor x \cdot 2^{K} \right\rfloor$. Substituting $\left \lfloor x \cdot 2^K \right\rfloor = \sum_{j=1}^K b_j 2^{j-1}$ into \eqref{eq:binary_sandwich}, we derive
\begin{align*}
& \left\lfloor \sum_{j > K - k} b_j 2^{j - 1 - K + k} + \sum_{j \leq K-k} b_j 2^{j - 1 - K + k} \right\rfloor \leq \left\lfloor x \cdot 2^{k} \right\rfloor \\
& \hspace{2in} \leq \sum_{j > K - k} b_j 2^{j - 1 - K + k} + \underbrace{\sum_{j \leq K-k} b_j 2^{j - 1 - K + k}}_{(\spadesuit) < 1}.
\end{align*}
Due to $(\spadesuit) < 1$, we conclude $\left\lfloor x \cdot 2^{k} \right\rfloor = \sum_{j=1}^k 2^{k-j} b_{K-j+1}$. Consequently, we deduce $\left\lfloor x \cdot 2^{k} \right\rfloor \equiv b_{K - k+1} \pmod{2}$, which verifies the claim by noting $(-1)^{\left\lfloor x \cdot 2^k \right\rfloor} = (-1)^{b_{K - k + 1}} = -2b_{K-k+1} + 1$.

Substituting \eqref{eq:claim_binary} into \eqref{eq:hK_derivative_simp}, for $x \in \left(i/2^K, (i+1)/2^K\right)$, we obtain
\begin{align}\label{eq:hK_derivative_segment}
\hat{h}_K'(x) = 1 + \sum_{k=1}^K (2b_{K - k + 1} - 1) 2^{-k} \quad \textrm{with} \quad \cB_K\left(\left\lfloor x \cdot 2^K \right\rfloor\right) = [b_1, \dots, b_K]^\top.
\end{align}

To establish the first assertion in Lemma \ref{lemma:slope_hK}, we only need to consider end points of each linear segment of $\hat{h}_K$. Otherwise, when $x \in \left(i/2^K, (i+1)/2^K\right)$ for some $i = 0, \dots, 2^K -1$, \eqref{eq:hK_derivative_segment} shows $\hat{h}_K$ is differentiable at $x$, and therefore, \eqref{eq:hK_upperlower_slope} holds true. Consider an end point $x = i/2^K$ for some $i = 1, \dots, 2^K - 1$. We evaluate left and right derivatives of $\hat{h}_K$ at $x$. We denote left and right derivatives as $\partial^- \hat{h}_K$ and $\partial^+ \hat{h}_K$, respectively. Using \eqref{eq:hK_derivative_segment} again, we derive 
\begin{align*}
\partial^- \hat{h}_K(x) & = \lim_{\Delta \rightarrow 0^+} \frac{\hat{h}_K(x-\Delta) - \hat{h}_K(x)}{\Delta} \\
& = \lim_{y \rightarrow x^-} \hat{h}_K'(y) \\
& = 1 + \sum_{k=1}^K \left(2 [\cB_K(i-1)]_{K-k+1} - 1\right)2^{-k}, \\
\partial^+ \hat{h}_K(x) & = \lim_{\Delta \rightarrow 0^+} \frac{\hat{h}_K(x+\Delta) - \hat{h}_K(x)}{\Delta} \\
& = \lim_{y \rightarrow x^+} \hat{h}_K'(y) \\
& = 1 + \sum_{k=1}^K \left(2 [\cB_K(i)]_{K-k+1} - 1\right)2^{-k}.
\end{align*} 
We note $\partial^+ \hat{h}_K(x) \geq \partial^- \hat{h}_K(x)$, and therefore, for $x = i/2^K$, we obtain
\begin{align}\label{eq:endpoint_slope}
\begin{split}
\upperslope_{\hat{h}_K}(x) & = \upperlim_{\Delta \rightarrow 0} \frac{\hat{h}_K(x+\Delta) - \hat{h}_K(x)}{\Delta} = \partial^+ \hat{h}_K(x), \\
\lowerslope_{\hat{h}_K}(x) & = \lowerlim_{\Delta \rightarrow 0} \frac{\hat{h}_K(x+\Delta) - \hat{h}_K(x)}{\Delta} = \partial^- \hat{h}_K(x). \\
\end{split}
\end{align}
This establishes the first assertion in Lemma \ref{lemma:slope_hK}.

To show the second assertion, we also tackle separately when $x$ is an end point of a linear segment or inside a linear segment of $\hat{h}_K$. Suppose $x \in \left(i/2^K, (i+1)/2^K\right)$ for some $i = 0, \dots, 2^K - 1$. We check that $\left\lfloor x \cdot 2^K \right\rfloor = \left\lceil x \cdot 2^K \right\rceil - 1 = i$. It implies \eqref{eq:upper_slope_value} and \eqref{eq:lower_slope_value} are both equal to \eqref{eq:hK_derivative_segment}. On the other hand, suppose $x = i/2^K$ for some $i = 1, \dots, 2^K - 1$, we check $\left\lfloor x \cdot 2^K \right\rfloor = \left\lceil x \cdot 2^K \right\rceil = i$. Therefore, \eqref{eq:upper_slope_value} and \eqref{eq:lower_slope_value} coincide with $\partial^+ \hat{h}_K(x)$ and $\partial^- \hat{h}_K(x)$, respectively. In combination with \eqref{eq:endpoint_slope}, we verify that \eqref{eq:upper_slope_value} and \eqref{eq:lower_slope_value} hold for any $x \in (0, 1)$. The proof is complete. 
\end{proof}
For later convenience, we define slopes at end points $x = 0$ and $x = 1$ as
\begin{align*}
\upperslope_{\hat{h}_K}(1) & = 2, && \lowerslope_{\hat{h}_K}(1) = \lim_{x \rightarrow 1^-} \hat{h}_K'(x) = 2 - 2^{-K}, \\
\lowerslope_{\hat{h}_K}(0) & = 0, && \upperslope_{\hat{h}_K}(0) = \lim_{x \rightarrow 0^+} \hat{h}_K'(x) = 2^{-K}. \\
\end{align*}

\subsubsection{Proof of Lemma \ref{lemma:monotone_xi}}\label{pf:monotone_xi}
\begin{proof}[Proof of Lemma \ref{lemma:monotone_xi}]
We first show $\hat{\times}(x, a)$ is monotone in $x$ for any fixed $a$. Let $x_1 \leq x_2 \in [0, 1]$. (We slightly abuse the notation here. Note that $x_1, x_2$ are scalars.) By the construction of $\hat{\times}$, we have
\begin{align*}
\hat{\times}(x_2, a) - \hat{\times}(x_1, a) & = \underbrace{\hat{h}_K\left(\frac{x_2 + a}{2}\right) - \hat{h}_K\left(\frac{x_1 + a}{2}\right)}_{(A)} \\
& \quad - \underbrace{\left(\hat{h}_K\left(\frac{\vert x_2 - a\vert}{2}\right) - \hat{h}_K\left(\frac{\vert x_1 - a\vert}{2}\right)\right)}_{(B)}.
\end{align*}
By the triangle inequality, we observe
\begin{align*}
\left\vert \frac{\vert x_1-a\vert}{2} - \frac{\vert x_2-a\vert}{2} \right\vert \leq \left\vert \frac{\vert x_1 - a - x_2 + a\vert}{2}\right\vert = \left\vert \frac{x_1+a}{2} - \frac{x_2+a}{2}\right\vert,
\end{align*} 
and $\frac{x_2 + a}{2} \geq \max\left\{\frac{\vert x_1 - a\vert}{2}, \frac{\vert x_2 - a\vert}{2}, \frac{x_1 + a}{2}\right\}$.

We need to compare the differences in term $(A)$ and $(B)$ in the following two cases. 

\noindent $\bullet$ If $\frac{x_1+a}{2} \geq \max\left\{\frac{\vert x_2-a\vert}{2}, \frac{\vert x_1-a\vert}{2} \right\}$, we have
\begin{align*}
(A) - (B) & \geq \overline\slope_{\hat{h}_K} \left(\frac{x_1+a}{2}\right) \left\vert \frac{x_2+a}{2} - \frac{x_1+a}{2}\right\vert \\
& \quad - \underline\slope_{\hat{h}_K}\left(\max\left\{\frac{\vert x_2-a\vert}{2}, \frac{\vert x_1-a\vert}{2}\right\}\right) \left\vert \frac{\vert x_2-a\vert}{2} - \frac{\vert x_1-a\vert}{2}\right\vert.
\end{align*}
By the triangle inequality, we observe
\begin{align*}
\left\vert \frac{\vert x_1-a\vert}{2} - \frac{\vert x_2-a\vert}{2} \right\vert \leq \left\vert \frac{\vert x_1 - a - x_2 + a\vert}{2}\right\vert = \left\vert \frac{x_1+a}{2} - \frac{x_2+a}{2}\right\vert.
\end{align*}
Meanwhile, by Lemma \ref{lemma:slope_hK}, $\overline\slope_{\hat{h}_K}(z_1) \geq \underline\slope_{\hat{h}_K}(z_2)$ whenever $z_1 \geq z_2$. Therefore, we verify $(A) - (B) \geq 0$.

\noindent $\bullet$ If on the contrary, $\frac{x_1+a}{2} < \max\left\{\frac{\vert x_2-a\vert}{2}, \frac{\vert x_1-a\vert}{2} \right\}$, by removing overlapping pieces, we have
\begin{align*}
& \quad (A) - (B) \\
& = \hat{h}_K\left(\frac{x_2 + a}{2}\right) - \hat{h}_K\left(\max\left\{\frac{\vert x_2-a\vert}{2}, \frac{\vert x_1-a\vert}{2} \right\}\right) \\
& \quad - \left(\hat{h}_K\left(\frac{x_1+a}{2}\right) - \hat{h}_K\left(\min\left\{\frac{\vert x_2-a\vert}{2}, \frac{\vert x_1-a\vert}{2} \right\}\right)\right) \\
& \geq \overline\slope_{\hat{h}_K} \left(\max\left\{\frac{\vert x_2-a\vert}{2}, \frac{\vert x_1-a\vert}{2} \right\}\right) \left\vert \frac{x_2+a}{2} - \max\left\{\frac{\vert x_2-a\vert}{2}, \frac{\vert x_1-a\vert}{2} \right\}\right\vert \\
& \quad - \underline\slope_{\hat{h}_K}\left(\frac{x_1+a}{2}\right) \left\vert \frac{x_1+a}{2} - \min\left\{\frac{\vert x_2-a\vert}{2}, \frac{\vert x_1-a\vert}{2} \right\}\right\vert \\
& \overset{(i)}{\geq} 0,
\end{align*}
where inequality $(i)$ holds, since
\begin{align*}
& \quad \left\vert \frac{x_2+a}{2} - \max\left\{\frac{\vert x_2-a\vert}{2}, \frac{\vert x_1-a\vert}{2} \right\}\right\vert \\
& = \left\vert \frac{x_1+a}{2} - \frac{x_2+a}{2}\right\vert - \left\vert \max\left\{\frac{\vert x_2-a\vert}{2}, \frac{\vert x_1-a\vert}{2} \right\} - \frac{x_1+a}{2}\right\vert \\
& \geq \left\vert \frac{\vert x_1-a\vert}{2} - \frac{\vert x_2-a\vert}{2} \right\vert - \left\vert \max\left\{\frac{\vert x_2-a\vert}{2}, \frac{\vert x_1-a\vert}{2} \right\} - \frac{x_1+a}{2}\right\vert \\
& = \left\vert \frac{x_1+a}{2} - \min\left\{\frac{\vert x_2-a\vert}{2}, \frac{\vert x_1-a\vert}{2} \right\}\right\vert
\end{align*}
and $\upperslope_{\hat{h}_K}\left(\max\left\{\frac{\vert x_2-a\vert}{2}, \frac{\vert x_1-a\vert}{2} \right\}\right) \geq \lowerslope_{\hat{h}_K}\left(\frac{x_1+a}{2}\right)$.

Combining the two cases above, we deduce $(A) - (B) \geq 0$, and $\hat{\times}(x, a)$ is monotone in $x$ for any fixed $a$. By symmetry, $\hat{\times}(a, x)$ is also monotone. When $m = [m_1, \dots, m_i^*, \dots, m_q]^\top$, $\phi(3N(x_i - m_i^*/N)) = 3Nx_i - 3m^*_i + 2$, which is increasing in $x_i$. By construction of $\hat{\xi}_m$ in \eqref{eq:xi_hat} and the monotonicity of composite functions, we deduce the monotonicity of $\hat\xi_m$. Similarly, when $m = [m_1, \dots, m_i^*-1, \dots, m_q]^\top$, we have $\phi(3N(x_i - (m_i^*-1)/N)) = -3Nx_i + 3m^*_i - 1$ --- decreasing in $x_i$. Therefore, $\hat{\xi}_m$ is decreasing with respect to the $i$-th coordinate in $x$. The proof is complete. 
\end{proof}

\subsubsection{Proof of Lemma \ref{lemma:lip_xi}}\label{pf:lip_xi}
\begin{proof}[Proof of Lemma \ref{lemma:lip_xi}]
We first analyze the Lipschitz continuity of $\hat{\times}$. Let's fix $a \in [0, 1]$ and recall $\hat{\times}(x, a) = \hat{h}_K\left(\frac{x + a}{2}\right) - \hat{h}_K\left(\frac{\vert x - a\vert}{2}\right)$. We observe that $\hat{\times}(x, a)$ is a piecewise linear function in $x$, due to $\hat{h}_K$ being piecewise linear. Therefore, to characterize the Lipschitz continuity of $\hat{\times}$, it suffices to evaluate the steepest and flattest slopes of $\hat{\times}(x, a)$ as $x$ varies in $[0, 1]$. Specifically, we define
\begin{align}
\textsf{SteepSlope}\left(\hat{\times}(\cdot, a)\right) & = \sup_{x \in (0, 1)} \upperlim_{\Delta \rightarrow 0} \frac{\hat{\times}(x+\Delta, a) - \hat{\times}(x, a)}{\Delta}, \label{eq:steepslope} \\ 
\textsf{FlatSlope}\left(\hat{\times}(\cdot, a)\right) & = \inf_{x \in (0, 1)} \lowerlim_{\Delta \rightarrow 0} \frac{\hat{\times}(x+\Delta, a) - \hat{\times}(x, a)}{\Delta}. \label{eq:flatslope}
\end{align}

\noindent $\bullet$ {\bf Steepest slope}. We consider two cases depending on the value of $x$, namely, $0 < x \leq a$ and $a < x < 1$. 

\noindent $\star$ {\it (Case 1)} When $a < x < 1$, we rewrite $\hat{\times}(x, a)$ as $\hat{\times}(x, a) = \hat{h}_K\left(\frac{x+a}{2}\right) - \hat{h}_K\left(\frac{x - a}{2}\right)$. Substituting into \eqref{eq:steepslope}, we obtain 
\begin{align*}
& ~\quad \upperlim_{\Delta \rightarrow 0} \frac{\hat{\times}(x+\Delta, a) - \hat{\times}(x, a)}{\Delta} \\
& = \upperlim_{\Delta \rightarrow 0} \frac{\hat{h}_K\left(\frac{x + \Delta + a}{2}\right) - \hat{h}_K\rbr{\frac{x + \Delta - a}{2}} - \hat{h}_K\rbr{\frac{x+a}{2}} + \hat{h}_K\rbr{\frac{x - a}{2}}}{\Delta} \\
& = \upperlim_{\Delta \rightarrow 0} \frac{\hat{h}_K\left(\frac{x + \Delta + a}{2}\right) - \hat{h}_K\rbr{\frac{x + a}{2}} - \sbr{\hat{h}_K\left(\frac{x + \Delta - a}{2}\right) - \hat{h}_K\rbr{\frac{x - a}{2}}} }{\Delta}.
\end{align*}
Lemma \ref{lemma:slope_hK} implies that $\hat{h}_K$ is strictly monotone increasing. Hence, for any $\Delta$, $\hat{h}_K\left(\frac{x + \Delta + a}{2}\right) - \hat{h}_K\rbr{\frac{x + a}{2}}$ and $\hat{h}_K\left(\frac{x + \Delta - a}{2}\right) - \hat{h}_K\rbr{\frac{x - a}{2}}$ are both positive or negative depending on the sign of $\Delta$. Moreover, Lemma \ref{lemma:slope_hK} shows that $\upperslope_{\hat{h}_K}$ and $\lowerslope_{\hat{h}_K}$ are monotone increasing. As a result, we have
\begin{align*}
& \quad~ \upperlim_{\Delta \rightarrow 0} \frac{\hat{\times}(x+\Delta, a) - \hat{\times}(x, a)}{\Delta} \\
& \leq \upperlim_{\Delta \rightarrow 0} \frac{\hat{h}_K\left(\frac{x + \Delta + a}{2}\right) - \hat{h}_K\rbr{\frac{x + a}{2}}}{\Delta} - \lowerlim_{\Delta \rightarrow 0} \frac{\hat{h}_K\left(\frac{x + \Delta - a}{2}\right) - \hat{h}_K\rbr{\frac{x - a}{2}}}{\Delta} \\
& = \frac{1}{2}\upperslope_{\hat{h}_K}\rbr{\frac{x+a}{2}} - \frac{1}{2}\lowerslope_{\hat{h}_K}\rbr{\frac{x-a}{2}}.
\end{align*}
Using Lemma \ref{lemma:slope_hK}, we can upper bound $\upperslope_{\hat{h}_K}(z)$ and lower bound $\lowerslope_{\hat{h}_K}(z)$ for any $z \in (0, 1)$ as
\begin{align}
\frac{1}{2} \upperslope_{\hat{h}_K}(z) & \leq \min\left\{z + \frac{1}{2^{K+1}}, 1\right\} \label{eq:slope_upper_bound}, \\
\frac{1}{2} \lowerslope_{\hat{h}_K}(z) & \geq \max\left\{z - \frac{1}{2^{K+1}}, 0 \right\} \label{eq:slope_lower_bound}.
\end{align}
The upper bound \eqref{eq:slope_upper_bound} is a consequence of \eqref{eq:upper_slope_value}. Specifically, for any $a \in (0, 1)$, it holds
\begin{align*}
\upperslope_{\hat{h}_K}(z) & = 1 + \sum_{k=1}^K \left(2 \left[\cB_K\left(\left\lfloor z \cdot 2^K \right\rfloor\right)\right]_{K - k + 1} - 1\right) 2^{-k} \\
& = 2^{-K} + 2 \sum_{k=1}^K \frac{\left[\cB_K\left(\left\lfloor z \cdot 2^K \right\rfloor\right)\right]_{K - k + 1} 2^{K-k}}{2^K} \\
& = 2^{-K} + 2\frac{\left\lfloor z \cdot 2^K \right\rfloor}{2^K} \\
& \leq 2z + 2^{-K}. 
\end{align*}
In combination with $2$ being a natural upper bound of $\upperslope_{\hat{h}_K}$ and rescaling by $1/2$, \eqref{eq:slope_upper_bound} holds true. The lower bound \eqref{eq:slope_lower_bound} is a consequence of \eqref{eq:lower_slope_value}. We have
\begin{align*}
\lowerslope_{\hat{h}_K}(z) & = 1 + \sum_{k=1}^K \left(2 \left[\cB_K\left(\left\lceil z \cdot 2^K \right\rceil - 1\right)\right]_{K-k+1} - 1 \right) 2^{-k} \\
& = 2^{-K} + 2 \sum_{k=1}^K \frac{\left[\cB_K\left(\left\lceil z \cdot 2^K \right\rceil - 1\right)\right]_{K-k+1} 2^{K-k}}{2^K} \\
& = 2^{-K} + 2 \frac{\left\lceil z \cdot 2^K \right\rceil - 1}{2^K} \\
& \geq 2z - 2^{-K}.
\end{align*}
Combining with $0$ being a natural lower bound of $\lowerslope_{\hat{h}_K}$, we establish \eqref{eq:slope_lower_bound}. To this end, \eqref{eq:slope_upper_bound} and \eqref{eq:slope_lower_bound} together yield
\begin{align}\label{eq:x_larger_a_slope}
& \quad~ \sup_{a < x < 1} \upperlim_{\Delta \rightarrow 0} \frac{\hat{\times}(x+\Delta, a) - \hat{\times}(x, a)}{\Delta} \nonumber \\
& \leq \sup_{a < x < 1} \frac{1}{2} \upperslope_{\hat{h}_K}\rbr{\frac{x+a}{2}} - \frac{1}{2} \lowerslope_{\hat{h}_K}\rbr{\frac{x-a}{2}} \nonumber \\
& \leq \sup_{a < x < 1} \min\left\{\frac{x+a}{2} + \frac{1}{2^{K+1}}, 1 \right\} - \max\left\{\frac{x-a}{2} - \frac{1}{2^{K+1}}, 0\right\} \nonumber \\
& = \min\left\{a + \frac{1}{2^K}, 1\right\}.
\end{align}

\noindent $\star$ {\it (Case 2)} When $0 < x \leq a$, the analysis is similar. We have $\hat{\times}(x, a) = \hat{h}_K\left(\frac{x+a}{2}\right) - \hat{h}_K\left(\frac{a - x}{2}\right)$, and derive
\begin{align*}
& \quad \upperlim_{\Delta \rightarrow 0} \frac{\hat{\times}(x+\Delta, a) - \hat{\times}(x, a)}{\Delta} \\
& = \upperlim_{\Delta \rightarrow 0} \frac{\hat{h}_K\left(\frac{x + \Delta + a}{2}\right) - \hat{h}_K\rbr{\frac{a - x - \Delta}{2}} - \hat{h}_K\rbr{\frac{x+a}{2}} + \hat{h}_K\rbr{\frac{a - x}{2}}}{\Delta} \\
& = \upperlim_{\Delta \rightarrow 0} \frac{\hat{h}_K\left(\frac{x + \Delta + a}{2}\right) - \hat{h}_K\rbr{\frac{x + a}{2}} + \sbr{\hat{h}_K\rbr{\frac{a - x}{2}} - \hat{h}_K\left(\frac{a - x - \Delta}{2}\right)}}{\Delta} \\
& = \upperlim_{\Delta \rightarrow 0} \frac{\hat{h}_K\left(\frac{x + \Delta + a}{2}\right) - \hat{h}_K\rbr{\frac{x + a}{2}}}{\Delta} + \frac{\hat{h}_K\rbr{\frac{a - x}{2}} - \hat{h}_K\left(\frac{a - x - \Delta}{2}\right)}{\Delta}.
\end{align*}
We also notice that $\hat{h}_K\left(\frac{x + \Delta + a}{2}\right) - \hat{h}_K\rbr{\frac{x + a}{2}}$ and $\hat{h}_K\rbr{\frac{a - x}{2}} - \hat{h}_K\left(\frac{a - x - \Delta}{2}\right)$ have the same sign depending on $\Delta$. In the case of $\Delta > 0$, we have
\begin{align*}
\lim_{\Delta \rightarrow 0^+} \frac{\hat{\times}(x+\Delta, a) - \hat{\times}(x, a)}{\Delta} = \frac{1}{2}\upperslope_{\hat{h}_K}\rbr{\frac{x+a}{2}} + \frac{1}{2}\lowerslope_{\hat{h}_K}\rbr{\frac{a-x}{2}}.
\end{align*}
Using \eqref{eq:upper_slope_value} and \eqref{eq:lower_slope_value}, we derive
\begin{align*}
& \quad \upperslope_{\hat{h}_K}\rbr{\frac{x+a}{2}} + \lowerslope_{\hat{h}_K}\rbr{\frac{a-x}{2}} \\
& = 2^{-K} + 2 \frac{\left\lfloor (x+a) \cdot 2^{K-1} \right\rfloor}{2^K} + 2^{-K} + 2 \frac{\left\lceil (a-x) \cdot 2^{K-1} \right\rceil - 1}{2^K} \\
& = \frac{\left\lfloor (x+a) \cdot 2^{K-1} \right\rfloor + \left\lceil (a-x) \cdot 2^{K-1} \right\rceil}{2^{K-1}} \\
& \leq \frac{(x+a)\cdot 2^{K-1} + (a - x) \cdot 2^{K-1} + 1}{2^{K-1}} \\
& = 2a + 2^{-K+1},
\end{align*}
which implies $\lim_{\Delta \rightarrow 0^+} \frac{\hat{\times}(x+\Delta, a) - \hat{\times}(x, a)}{\Delta} \leq a + 2^{-K}$ for any $0 < x \leq a$. In the case of $\Delta < 0$, we have
\begin{align*}
\lim_{\Delta \rightarrow 0^-} \frac{\hat{\times}(x+\Delta, a) - \hat{\times}(x, a)}{\Delta} & = \frac{1}{2}\lowerslope_{\hat{h}_K}\rbr{\frac{x+a}{2}} + \frac{1}{2}\upperslope_{\hat{h}_K}\rbr{\frac{a-x}{2}} \\
& = \frac{\left\lceil (a+x) \cdot 2^{K-1}\right\rceil + \left\lfloor (a - x) \cdot 2^{K-1}\right\rfloor}{2^{K-1}} \\
& \leq \frac{(a+x) \cdot 2^{K-1} + 1 + (a - x) \cdot 2^{K-1}}{2^{K-1}} \\
& = 2a + 2^{-K+1},
\end{align*}
which implies $\lim_{\Delta \rightarrow 0^-} \frac{\hat{\times}(x+\Delta, a) - \hat{\times}(x, a)}{\Delta} \leq a + 2^{-K}$ for any $0 < x \leq a$. Combining both $\Delta > 0$ and $\Delta < 0$ cases, we conclude
\begin{align}\label{eq:x_smaller_a_slope}
\sup_{0 < x \leq a} \upperlim_{\Delta \rightarrow 0} \frac{\hat{\times}(x+\Delta, a) - \hat{\times}(x, a)}{\Delta} \leq a + 2^{-K}.
\end{align} 
Putting \eqref{eq:x_larger_a_slope} and \eqref{eq:x_smaller_a_slope} together, we deduce
\begin{align*}
\textsf{SteepSlope}\left(\hat{\times}(\cdot, a)\right) \leq a + \frac{1}{2^K}.
\end{align*} 

\noindent $\bullet$ {\bf Flattest slope}. We also discuss two cases, i.e., $0 < x \leq a$, $a < x < 1$. 

\noindent $\star$ {\it (Case 1)} When $a < x < 1$, following the same computation for the steepest slope, we have
\begin{align}\label{eq:x_larger_a_slope_flat}
& \quad~ \lowerlim_{\Delta \rightarrow 0} \frac{\hat{\times}(x+\Delta, a) - \hat{\times}(x, a)}{\Delta} \nonumber \\
& \geq \lowerlim_{\Delta \rightarrow 0} \frac{\hat{h}_K\rbr{\frac{x+\Delta+a}{2}} - \hat{h}_K\rbr{\frac{x+a}{2}}}{\Delta} - \upperlim_{\Delta \rightarrow 0} \frac{\hat{h}_K\rbr{\frac{x+\Delta-a}{2}} - \hat{h}_K\rbr{\frac{x-a}{2}}}{\Delta} \nonumber \\
& = \frac{1}{2} \lowerslope_{\hat{h}_K}\rbr{\frac{x+a}{2}} - \frac{1}{2} \upperslope_{\hat{h}_K} \rbr{\frac{x-a}{2}} \nonumber \\
& \overset{(i)}{\geq} \max\left\{\frac{x+a}{2} - \frac{1}{2^{K+1}}, 0\right\} - \min\left\{\frac{x-a}{2} + \frac{1}{2^{K+1}}, 1\right\} \nonumber \\
& \overset{(ii)}{\geq} \max\left\{a - \frac{1}{2^K}, 0\right\},
\end{align}
where inequality $(i)$ invokes \eqref{eq:slope_upper_bound} and \eqref{eq:slope_lower_bound}, and inequality $(ii)$ uses the natural lower bound $\lowerslope_{\hat{h}_K}\rbr{\frac{x+a}{2}} - \upperslope_{\hat{h}_K} \rbr{\frac{x-a}{2}} \geq 0$ since $x+a > x - a$.

\noindent $\star$ {\it (Case 2)} When $0 < x \leq a$, we derive
\begin{align}
& \quad \lowerlim_{\Delta \rightarrow 0} \frac{\hat{\times}(x+\Delta, a) - \hat{\times}(x, a)}{\Delta} \nonumber \\
& = \lowerlim_{\Delta \rightarrow 0} \frac{\hat{h}_K\left(\frac{x + \Delta + a}{2}\right) - \hat{h}_K\rbr{\frac{x + a}{2}} + \sbr{\hat{h}_K\rbr{\frac{a - x}{2}} - \hat{h}_K\left(\frac{a - x - \Delta}{2}\right)}}{\Delta} \nonumber \\
& = \lowerlim_{\Delta \rightarrow 0} \frac{\hat{h}_K\left(\frac{x + \Delta + a}{2}\right) - \hat{h}_K\rbr{\frac{x + a}{2}}}{\Delta} + \frac{\hat{h}_K\rbr{\frac{a - x}{2}} - \hat{h}_K\left(\frac{a - x - \Delta}{2}\right)}{\Delta}. \nonumber
\end{align}
We distinguish the limit depending on $\Delta$ being positive or negative. If $\Delta > 0$, we have
\begin{align}
& \quad \lim_{\Delta \rightarrow 0^+} \frac{\hat{\times}(x+\Delta, a) - \hat{\times}(x, a)}{\Delta} \nonumber \\
& = \frac{1}{2} \upperslope_{\hat{h}_K} \rbr{\frac{x+a}{2}} + \frac{1}{2} \lowerslope_{\hat{h}_K} \rbr{\frac{a-x}{2}} \nonumber \\
& = \frac{1}{2}\left(2^{-K} + 2\frac{\left\lfloor (x+a) \cdot 2^{K-1} \right\rfloor}{2^{K}} + 2^{-K} + 2 \frac{\left\lceil (a-x) \cdot 2^{K-1} \right\rceil - 1}{2^K}\right) \nonumber \\
& = \frac{\left\lfloor (x+a) \cdot 2^{K-1} \right\rfloor + \left\lceil (a-x) \cdot 2^{K-1} \right\rceil}{2^{K}} \nonumber \\
& \overset{(i)}{\geq} \max\left\{\frac{(a+x) \cdot 2^{K-1} - 1 + (a - x) \cdot 2^{K-1}}{2^{K}}, 0 \right\} \nonumber \\
& = \max\left\{a - \frac{1}{2^K}, 0\right\}, \nonumber
\end{align}
where inequality $(i)$ uses $0$ being a natural lower bound of $\upperslope_{\hat{h}_K} \rbr{\frac{x+a}{2}}$ and $\lowerslope_{\hat{h}_K} \rbr{\frac{a-x}{2}}$.
If $\Delta < 0$, we have
\begin{align*}
& \quad \lim_{\Delta \rightarrow 0^-} \frac{\hat{\times}(x+\Delta, a) - \hat{\times}(x, a)}{\Delta} \nonumber \\
& = \frac{1}{2} \lowerslope_{\hat{h}_K} \rbr{\frac{x+a}{2}} + \frac{1}{2} \upperslope_{\hat{h}_K} \rbr{\frac{a-x}{2}} \nonumber \\
& = \frac{\left\lceil (a+x) \cdot 2^{K-1} \right\rceil + \left\lfloor (a-x) \cdot 2^{K-1} \right\rfloor}{2^K} \nonumber \\
& \geq \max\left\{\frac{(a+x) \cdot 2^{K-1} + (a-x) \cdot 2^{K-1} - 1}{2^K}, 0\right\} \nonumber \\
& = \max\left\{a - 2^{-K}, 0\right\}.
\end{align*}
Combining both $\Delta > 0$ and $\Delta < 0$, we deduce
\begin{align}\label{eq:x_smaller_a_slope_flat}
\sup_{0 < x \leq a} \lowerlim_{\Delta \rightarrow 0} \frac{\hat{\times}(x+\Delta, a) - \hat{\times}(x, a)}{\Delta} \geq \max\{a - 2^{-K}, 0\}. 
\end{align}
Putting \eqref{eq:x_larger_a_slope_flat} and \eqref{eq:x_smaller_a_slope_flat} together, we deduce
\begin{align*}
\textsf{FlatSlope}\left(\hat{\times}(\cdot, a)\right) \geq \max\left\{a - \frac{1}{2^K}, 0\right\}.
\end{align*}

To complete the proof, we observe that $\hat{\xi}_m$ is a composition of $q$ approximate product operations $\hat{\times}$. For $x = [x_1, \dots, x_i, \dots, x_q]^\top$ and $x' = [x_1, \dots, x_i', \dots, x_q]^\top$ only differing in the $i$-th coordinate, recursively applyling \eqref{eq:steepslope} and \eqref{eq:flatslope}, we derive 
\begin{align*}
& 3N \prod_{j\neq i} \max \left\{\phi(3N(x_j - m_j/N)) - \frac{1}{2^K}, 0\right\} \vert x_i - x_i'\vert \leq \left\vert \hat{\xi}_m(x) - \hat{\xi}_m(x')\right\vert \\
& \hspace{1.7in} \leq 3N \prod_{j \neq i} \left(\phi(3N(x_j - m_j/N)) + \frac{1}{2^{K}}\right) \vert x_i - x_i'\vert.
\end{align*}
The proof is complete. 
\end{proof}

\section{Proof of Lemma \ref{lemma:approx}}\label{pf:approx}
\begin{proof}[Proof of Lemma \ref{lemma:approx}]
Lemma \ref{lemma:approx} is a direct result of {\it Theorem 1} in \cite{yarotsky2017error}, which is originally proved for Sobolev functions on $[0,1]^d$. The proof for H\"{o}lder functions can be found in \cite{chen2019efficient}. The high level idea consists of two steps: 1) Approximate the target function $f$ using a weighted sum of local Taylor polynomials; 2) Implement each Taylor polynomial using a ReLU network. The proof holds true even when $\cZ$ is a subset of $[0,1]^d$. In the first step, we let $N = \left\lceil (\frac{s!\delta}{2^{d+1} d^s})^{-1/s} \right\rceil$ and discretize $[0,1]^d$ by a uniform grid with side length $1/N$. A local Taylor polynomial approximation is used in each small cube with side length $1/N$. If a cube intersects with $\partial \cZ$, we can pick any point in the interior of $\cZ$ as the center for Taylor expansion and the proof of {\it Theorem 1} in \cite{yarotsky2017error} remains the same. When the H\"older norm of $f$ is bounded by $C$, all weight parameters in the network constructed in Lemma \ref{lemma:approx} are bounded by $C$. 
\end{proof}

\section{Proof of Lemma \ref{lemma:linearprojection}}\label{pf:linearprojection}
\begin{proof}[Proof of Lemma \ref{lemma:linearprojection}]
The compactness of $\cY$ follows from $\cX$ being compact and $A$ being a continuous transformation. To see $\cY$ is also convex, we consider $y_1, y_2 \in \cY$ and any $\lambda \in (0, 1)$. Since $\cY = A^\top \cX$, we have $x_1, x_2 \in \cX$ such that $A^\top x_1 = y_1$ and $A^\top x_2 = y_2$. Then we have
\begin{align*}
\lambda y_1 + (1 - \lambda) y_2 = A^\top (\lambda x_1 + (1 - \lambda) x_2) \in A^\top \cX = \cY.
\end{align*}
Therefore, $\cY$ is convex.

To check $A \cY = \cX$, we show $A \cY \subset \cX$ and $\cX \subset A \cY$ hold true simultaneously. Let $y \in \cY$, then there exists $x \in \cX$ such that $y = A^\top x$. Due to Assumption \ref{assumption:lowdmanifold}, we write $x \in \cX$ as $x = A z$. As a result, we derive $A y = A A^\top x = AA^\top Az = Az = x \in \cX$. Thus, $A \cY \subset \cX$. On the other hand, any $x \in \cX$ can be written as $x = A z$, which implies $A^\top x = z \in \cY$, since $A$ has orthonormal columns. Therefore, $\cX \subset A \cY$. Combining two arguments together, we deduce $\cX = A\cY$.
\end{proof}

\end{document}